\documentclass[10pt]{article} 
\usepackage[accepted]{tmlr}

\usepackage{amsmath,amsfonts,bm}

\def\eqref#1{equation~\ref{#1}}
\def\Eqref#1{Equation~\ref{#1}}

\def\1{\bm{1}}

\DeclareMathAlphabet{\mathsfit}{\encodingdefault}{\sfdefault}{m}{sl}
\SetMathAlphabet{\mathsfit}{bold}{\encodingdefault}{\sfdefault}{bx}{n}

\def\gA{{\mathcal{A}}}

\def\gP{{\mathcal{P}}}

\def\gS{{\mathcal{S}}}

\def\gX{{\mathcal{X}}}

\usepackage{xcolor}
\usepackage{hyperref}
\usepackage{url}
\usepackage{subcaption}
\usepackage{graphicx}
\usepackage{svg}

\usepackage{amsthm}
\newtheorem{theorem}{Theorem}

\newtheorem{definition}{Definition}
\newtheorem{preposition}{Preposition}[section]

\usepackage{dsfont}
\usepackage{algorithm}
\usepackage{algpseudocode}
\usepackage{comment}
\usepackage{amssymb}
\usepackage{booktabs} 
\usepackage{tikz} 

\usepackage{bbm}
\usepackage{bm}
\usepackage{booktabs, makecell}
\newcommand\headercell[1]{%
   \smash[b]{\begin{tabular}[t]{@{}c@{}} #1 \end{tabular}}}

\title{ETGL-DDPG: A Deep Deterministic Policy Gradient Algorithm for Sparse Reward Continuous Control}

\author{\name Ehsan Futuhi  \\
      \addr Department of Computing Science, University of Alberta\\
      \AND
      \name Shayan Karimi \\
      \addr Department of Computing Science, University of Alberta
      \AND
      \name Chao Gao \\
      \addr Huawei Canada Research Center, Edmonton, Canada \\
      \AND
      \name Martin Müller \\
      \addr Department of Computing Science, University of Alberta
}

\def\month{05}  
\def\year{2025} 
\def\openreview{\url{https://openreview.net/forum?id=6g1WJ55N51}} 

\begin{document}

\maketitle

\begin{abstract}
We consider deep deterministic policy gradient (DDPG) in the context of reinforcement learning with sparse rewards. To enhance exploration, we introduce a search procedure, \emph{${\epsilon}{t}$-greedy}, which generates exploratory options for exploring less-visited states. We prove that search using $\epsilon t$-greedy has polynomial sample complexity under mild MDP assumptions. To more efficiently use the information provided by rewarded transitions, we develop a new dual experience replay buffer framework, \emph{GDRB}, and implement \emph{longest n-step returns}. The resulting algorithm, \emph{ETGL-DDPG}, integrates all three techniques: \bm{$\epsilon t$}-greedy, \textbf{G}DRB, and \textbf{L}ongest $n$-step, into DDPG. We evaluate ETGL-DDPG on standard benchmarks and demonstrate that it outperforms DDPG, as well as other state-of-the-art methods, across all tested sparse-reward continuous environments. Ablation studies further highlight how each strategy individually enhances the performance of DDPG in this setting.
\end{abstract}

\section{Introduction}\label{sec:intro}

Deep deterministic policy gradient (DDPG)~\citep{lillicrap2015continuous} is one of the representative algorithms for reinforcement learning (RL)~\citep{sutton2018reinforcement}, alongside other prominent approaches \citep{haarnoja2018soft, fujimoto2018addressing, andrychowicz2017hindsight}. In recent years, DDPG has served as the backbone algorithm for introducing novel ideas in robotics and RL \cite{DBLP:conf/iclr/Barth-MaronHBDH18, pan2020softmax, liu2023metric,wang2023optimal,tiapkin2024generative}. While DDPG demonstrates strong performance in continuous control tasks with dense reward signals~\citep{duan2016benchmarking, kiran2021deep}, its effectiveness diminishes significantly in sparse-reward settings where rewards are only observed upon reaching the goal~\citep{matheron2019problem, luck2019improved}. 

In sparse-reward environments where success depends on reaching a goal state, DDPG's deficiency can be explained from three perspectives. The first one is its lack of  \emph{directional exploration}. Like other off-policy RL algorithms, DDPG employs a \emph{behavior policy} for exploring the environment. The standard choices are either an $\epsilon$-greedy behavior policy that samples a random action with probability $\epsilon$ (e.g., $0.1$), or the main policy with artificial noise. As argued in~\citep{dabney2020temporally}, these one-step \emph{noise augmented greedy} strategies are ineffective for exploring large sparse-reward state spaces due to the lack of temporal abstraction. To improve $\epsilon$-greedy, \citet{dabney2020temporally} propose a temporally extended $\epsilon z$-greedy policy that expands exploration into multiple steps, controlled by a distribution $z$. $\epsilon z$-greedy represents an advancement from the option framework for reinforcement learning~\citep{sutton1999between}. Theoretically, an option $O$ is defined as a tuple $O= \langle I,\pi,\beta \rangle$, where $I$ is the set of states where an option can begin, $\pi$ is the option policy that determines which actions to take while executing the option, and $\beta$ is the termination condition. In $\epsilon z$-greedy, each option repeats a primitive action for a specific number of time steps which is sampled from a distribution $z$ (e.g., a uniform distribution). The option can begin at any state with probability $\epsilon$ and terminates whenever their length reaches a limit that is decided by $z$. While $\epsilon z$-greedy improves over $\epsilon$-greedy, it is also  \emph{directionless}: for exploratory action, the agent does not use any information from its experience for more informed exploration.  

\begin{figure}[t] \small
    \centering
        \begin{subfigure}[b]{197pt} \label{sd}
        \centering
        \vspace{-5pt}
        \includegraphics[width=7cm,height=4.0cm]{"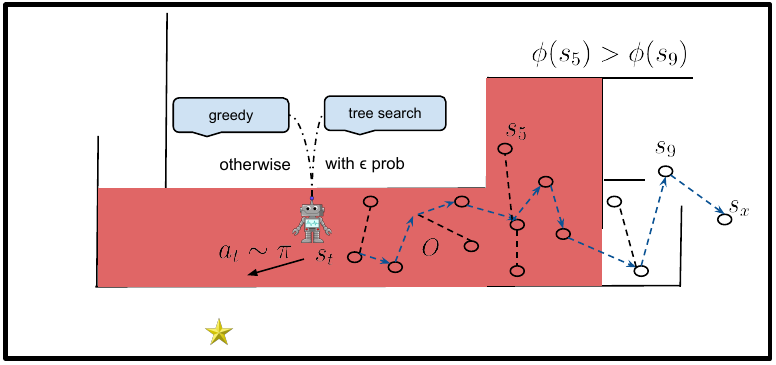"}
        \caption{$\epsilon t$-greedy: greedy or tree search}
        \end{subfigure}
        \begin{subfigure}[b]{197pt}
        \centering
        \includegraphics[width=7cm,height=4.055cm]{"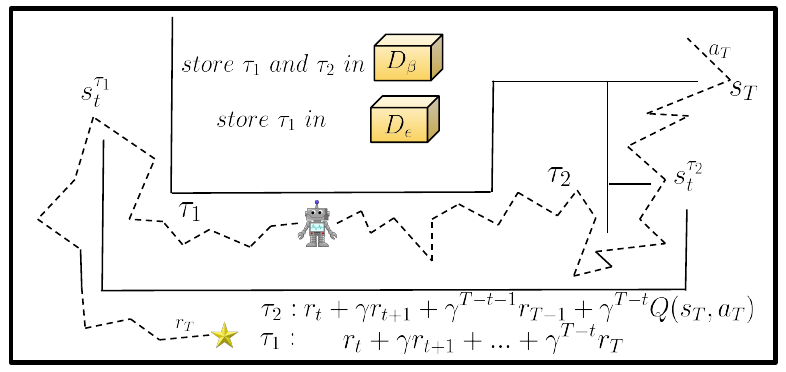"}
        \caption{GDRB and the longest n-step return}
        \end{subfigure}
    \caption{ (a): $\epsilon t$-greedy exploration strategy. The agent creates a tree from the current state $s_{t}$ with $\epsilon$ probability. Otherwise, it uses its policy to determine the next action $a_{t}\sim \pi$. The tree uses a hash function $\phi$ to estimate the visit counts to states. If the newly added node $s_{x}$ to the tree is located in an unvisited area $n(\phi(s_{x}))=0$, the path from the root to that node is returned as option $O$. The tree helps in avoiding obstacles, discovering unexplored areas, and staying away from highly-visited regions (middle red area). (b): GDRB and the longest n-step return for Q-value updates. $\tau_{1}$ reaches the goal (a successful episode), and $\tau_{2}$ is truncated by time limit (an unsuccessful episode). The first buffer $D_{\beta}$ stores both trajectories but $D_{e}$ only stores successful trajectories. The target Q-value for state $s_{t}$ is shown for both trajectories below the figure. In successful episodes, the target Q-value is the episode return. $s_T$ represents the last state in each episode, which is the goal state indicated by a star in $\tau_1$.}
    \label{fig: intro}
\end{figure}

The second drawback of DDPG is its uniform treatment of zero and non-zero rewards in the replay buffer. For most off-policy RL algorithms, a replay buffer is used to store and sample transitions of the agent's interactions with the environment. By default, DDPG uses a uniform sampling strategy that assigns an equal probability of being chosen to all transitions in the buffer. In sparse-reward environments, uniform sampling therefore rarely chooses rewarded transitions. In general, RL algorithms can be improved by prioritizing transitions based on the associated rewards or TD error~\citep{schaul2015prioritized}. For problems with well-defined goals, a replay buffer can be further enhanced to exploit the strong correlation of rewards and goals.  
The third weakness of DDPG is its slow information propagation when updating its learning policy. Since only the last transition in a successful episode (i.e., goal reached) gets rewarded, in standard DDPG, the agent must achieve the goal many times to make sure that the reward is eventually propagated backward to early states.  It is known that one way to achieve this is to provide intermediate rewards with reward shaping methods~\citep{laud2004theory}. However, effective reward shaping is usually problem-specific and does not generalize to a wide range of tasks.

In this paper, we enhance DDPG~\citep{lillicrap2015continuous} to address all three aforementioned problems. We choose DDPG over SAC, as SAC's maximum entropy framework already provides advanced exploration, and combining it with our components would introduce extra complexity, making it harder to isolate and analyze our contributions. Our first contribution is \emph{$\epsilon t$-greedy}, a new temporally version of $\epsilon$-greedy that utilizes a light-weight search procedure, similar to \citet{laud2004theory}, to enable more directional exploration based on the agent's previous experience data. We show that similar to $\epsilon z$-greedy, $\epsilon t$-greedy has polynomial sample complexity in related parameters of the MDP. Our second contribution is a new \emph{goal-conditioned dual replay buffer} (GDRB), that uses two replay buffers along with an adaptive sampling strategy to differentiate goal-reached and goal-not-reached experience data. These two buffers differ in retention policy, size, and the transitions they store. Our third enhancement is to replace the one-step update in DDPG with the \emph{longest $n$-step return} for all transitions in an episode. Figure~\ref{fig: intro} illustrates the innovations of ETGL-DDPG. In Section \ref{sec:experiments}, we evaluate the performance of ETGL-DDPG through extensive experiments on 2D and 3D continuous control benchmarks. As the proposed strategies are orthogonal, we show that each of the three strategies individually improves the performance of DDPG via a thorough ablation study. Furthermore, ETGL-DDPG outperforms current state-of-the-art methods across all tested environments.

\section{Background}
We consider a Markov decision process (MDP) defined by the tuple $(S, A, \mathcal{T}, r, \gamma, \rho)$. $S$ is the set of states, $A$ is the set of actions, $\mathcal{T}(s^{\prime}|s,a)$ is the transition distribution, $r: S\times A\times S \rightarrow \mathbb{R}$ is the reward function, $\gamma \in [0,1]$ is the discount factor, and $\rho(s_{0},s_{g})$ is the distribution from which initial and goal states are sampled for each episode. 
Every episode starts with sampling a new pair of initial and goal states. 
At each time-step $t$, the agent chooses an action using its policy and considering the current state and the goal state $a_{t}=\pi(s_{t},s_{g})$ resulting in reward $r_{t}=(s_{t},a_{t},s_{g})$. The next state is sampled from $\mathcal{T}(.|s_{t},a_{t})$. The episode ends when either the goal state or the maximum number of steps $T$ is reached. The return is the discounted sum of future rewards $R_{t}=\sum^{T}_{i=t} \gamma^{i-t}r_{i}$. The Q-function and value function associated with the agent's policy are defined as $Q^{\pi}(s_{t},a_{t},s_{g})=\mathbb{E}[R_{t}|s_{t},a_{t},s_{g}]$ and $V^{\pi}(s_{t},s_{g})=max_{a} Q^{\pi}(s_{t},a_{t},s_{g})$. The agent's objective is to learn an optimal policy $\pi^{*}$ that maximizes the expected return $\mathbb{E}_{s_{0}}[R_{0}|s_{0},s_{g}]$. 

\subsection{Deep Deterministic Policy Gradient (DDPG)}
To ease presentation, we adopt our notation with explicit reference to the goal state for both the critic and the actor networks in DDPG. DDPG maintains an actor $\mu(s, s_{g})$ and a critic $Q(s, a, s_{g})$. The agent explores the environment through a stochastic policy $a \sim \mu(s,s_{g})+ w$, where $w$ is a noise sampled from a normal distribution or an Ornstein-Uhlenbeck process \citep{uhlenbeck1930theory}. To update both actor and critic, transition tuples are sampled from a replay buffer to perform a mini-batch gradient descent. The critic is updated by a loss $L$, 

\begin{equation}
L= \mathbb{E}[Q(s_{t},a_{t},s_{g})-y_{t}]^{2} 
\end{equation}

where $y_{t}=r_{t}+\gamma Q^{\prime}(s_{t+1},\mu^{\prime}(s_{t+1},s_{g}),s_{g})$. $Q^{\prime}$ and $\mu^{\prime}$ are the target critic and actor, respectively; their weights are soft-updated to the current weights of the main critic and actor, respectively. The actor is updated by the deterministic policy gradient algorithm \citep{silver2014deterministic} to maximize the estimated Q-values of the critic using loss $-\mathbb{E}_{s}[Q(s,\mu(s,s_{g}),s_{g})]$.

\subsection{Locality-Sensitive Hashing}

Our approach discretizes the state space with a hash function $\phi: \mathbb{S} \rightarrow \mathbb{M}$, that maps states to buckets in $\mathbb{M}$. When we encounter a state $s$, we increment the visit count for $\phi(s)$. We use $n(\phi(s))$ as the visit counts of all states that map to the same bucket $\phi(s)$. Clearly, the \emph{granularity} of the discretization significantly impacts our exploration method. The goal for the granularity is that “distant” states are in separate buckets while “similar” states are grouped into one. 

We use Locality-Sensitive Hashing (LSH) as our hashing function, a popular class of hash functions for querying nearest neighbors based on a similarity metric \citep{bloom1970space}. SimHash~\citep{charikar2002similarity} is a computationally efficient LSH method that calculates similarity based on angular distance. SimHash retrieves a binary code of state $s \in S$ as 

\begin{equation}
\phi(s)=sgn(Af(s)) \in \{-1,1\}^{k},
\end{equation}

where $f: S \rightarrow \mathbb{R}^{D}$ is a preprocessing function and $A$ is a $k\times D$ matrix with i.i.d. entries drawn from a standard Gaussian distribution $\mathcal{N}(0,1)$. The parameter $k$ determines the granularity of the hash: larger values result in fewer collisions, thereby enhancing the ability to distinguish between different states.

\section{The ETGL-DDPG Method}\label{sec:methods}

In this section, we describe three strategies in ETGL-DDPG for improving DDPG in sparse-reward tasks. The full pseudocode for ETGL-DDPG is presented in Supplementary Algorithm \ref{alg:ETGL algorithm}.

\subsection{$\epsilon t$-Greedy: Exploration With Search}\label{et-Greedy}
In principle, exploration should be highest at the beginning of training, as discovering rewarded transitions during early steps is essential for escaping local optima~\citep{matheron2019problem}. Motivated by the success of the fast exploration algorithms RRT \citep{lavalle1998rapidly} and $\epsilon z$-greedy~\citep{dabney2020temporally}, we introduce \emph{$\epsilon t$-greedy}, which combines $\epsilon$-greedy with a \emph{tree search} procedure. Like $\epsilon$-greedy, $\epsilon t$-greedy selects a greedy action with probability $1-\epsilon$, and an exploratory action with probability $\epsilon$. However, instead of exploring uniformly at random, the exploratory action in $\epsilon t$-greedy is the first step of an \emph{option} generated via a search with time budget $N$. 

To execute the search process, the agent requires access to the environment's transition function $\mathcal{T}$ of the corresponding MDP. This is used to generate new nodes within the search tree. However, since our exploration strategy is built on DDPG, the model-free algorithm, the transition function $\mathcal{T}$ is not known. Instead, the agent utilizes its replay buffer to advance the search. We briefly discuss the impact of having access to $\mathcal{T}$ on the exploration process in Supplementary Material \ref{sec: exp with perfect model}. We also assume that the agent has a SimHash function $\phi$, which discretizes the large continuous environment. For each state $s$, $n(\phi(s))$ serves as an estimate of the number of visits to a neighbourhood of $s$ throughout the entire learning process.

The replay buffer contains transitions observed during training. It can be used as a transition model for observed transitions and an approximate one for transitions similar to those already seen. For simplicity, we identify each bucket with its hash code $\phi(s)$. We use a buffer $B_{M}$ which stores observed transitions based on the hash of their states $\phi(s)$. If the agent makes a transition $(s_{t}, a_{t},r_{t},s_{t+1})$ in the environment, the transition is stored in bucket $b=\phi(s_{t})$. All transitions are assigned to their buckets upon being added to the replay buffer. As training may take a long time, we limit the number of transitions in each bucket, and randomly replace one of the old transitions in a full bucket with the new transition. 

The function \texttt{next\_state\_from\_replay\_buffer} in Algorithm \ref{alg:exp with replay buffer} shows how new nodes can be added to the search: assuming we are at node $s_{x}$, we randomly select a transition $(s^{\prime},a,r,s{''})$ in bucket $\phi(s_{x})$ and create a new child $s_{x{\prime}}$ for $s_x$ by using following approximation:

\begin{equation}
    \mathcal{T}(s_x,a) \approx \mathcal{T}(s^{\prime},a) 
\end{equation}

Algorithm \ref{alg:exp with replay buffer} explains how the search generates an exploratory option. Initially, at state $s$, we create a list of frontier nodes consisting of only the root node $s$. If bucket of state $s$ in $B_{M}$ is empty: $b_{\phi(s)} = \varnothing$, there is no transition to approximate $\mathcal{T}(s,a)$. In this case, $\epsilon t$-greedy as in $\epsilon$-greedy generates a random action at $s$. Otherwise, when $b_{\phi(s)} \neq \varnothing$, $\epsilon t$-greedy conducts a tree search iteratively, with a maximum of $N$ iterations. At each iteration, a node $s_x$ is sampled uniformly from the frontier nodes, and a \emph{child} for $s_x$, noted as $s_{x'}$, is generated using \texttt{next\_state\_from\_replay\_buffer} function. If $n(\phi(s_{x'})) = 0$, we terminate and return the action sequence from the root to $s_{x'}$; otherwise, we repeat this process until we have added $N$ nodes to the tree. We then return the action sequence from the root to a least-visited node $s_{min}$: 

\begin{equation}
    s_{min}=\min_{s\  \in \ frontier\  nodes} n(\phi(s))    
\end{equation}
 
To justify this exploration method, we adopt the conditions outlined in \citet{DBLP:journals/corr/abs-1805-09045} to validate the sample efficiency of $\epsilon t$-greedy. We begin by introducing the relevant terms and then present the main theorem. Detailed definitions and proofs are provided in Appendix \ref{thm:PAC-RL}. The key idea is to define a measure that captures the concept of visiting all state-action pairs, as outlined in Definition \ref{def-cover-len}.
\begin{definition}[\textbf{Covering Length}] \label{def-cover-len}
    The covering length \citep{EvanEyalMansour} represents the minimum number of steps an agent must take in an MDP, starting from any state-action pair $(s,a) \in \mathcal{S} \times \mathcal{A}$, to visit all state-action pairs at least once with a probability of at least 0.5. We define the covering length only for discrete MDPs; for continuous MDPs, we consider a discretization of the state-action space $\mathcal{S} \times \mathcal{A}$.
\end{definition}

Our objective is to find a near-optimal policy, as defined in Definition \ref{good-pol}.

\begin{definition}[\textbf{$\epsilon$-optimal Policy}]\label{good-pol}
    A policy $\pi$ is called $\epsilon$-optimal if it satisfies $V^{\pi^*}(s) - V^{\pi}(s) \le \epsilon$, for all $s \in S$, where $\epsilon > 0$.
\end{definition}
Next, we define the concept of sample efficiency, which is captured through the notion of polynomial sample complexity in Definition \ref{poly-sample}.
\begin{definition}[\textbf{PAC-MDP Algorithm}] \label{poly-sample}
    Given a state space $\mathcal{S}$, action space $\mathcal{A}$, suboptimality error $\epsilon > 0$ (from Definition \ref{good-pol}) and $0 < \delta < 1$, an algorithm $\mathcal{A}$ is called PAC-MDP \citep{Kakade2003OnTS}, if the number of time steps required to find a $\epsilon$-optimal policy is less than some polynomial in the relevant quantities $(|\mathcal{S}|,|\mathcal{A}|,\frac{1}{\epsilon},\frac{1}{1-\gamma},\frac{1}{\delta})$ with probability at least $1-\delta$.
    \label{PACMDP}
\end{definition}
For simplicity, when we say an algorithm $\mathcal{A}$ has polynomial sample complexity, we imply that $\mathcal{A}$ is PAC-MDP. The work by \citet{DBLP:journals/corr/abs-1805-09045} establishes polynomial sample complexity for a uniformly random exploration by bounding the covering length defined in Definition \ref{def-cover-len}. Using this, and considering a limited tree budget $N$, we show that $\epsilon t$-greedy is PAC-MDP. Let's denote the search tree by $\mathcal{X}$, and the distribution over the generated options in $\mathcal{X}$ as $\mathcal{P}_{\gX}$. The following Theorem provides a lower bound on option sampling in tree $\mathcal{X}$ under certain condition.

\begin{theorem}[\textbf{Worst-Case Sampling}] \label{theorem : worstcase}
    Given a tree $\mathcal{X}$ with $N$ nodes ($s_1$ to $s_N$), for any $\omega \in \Omega_{\mathcal{X}}$, the sampling probability satisfies:
    \begin{align}
        \mathcal{P}_{\mathcal{X}}[\omega] \ge \frac{1}{N!(\max_{i \in [N]}|\phi(s_i)|)^{N-1}} \ge \frac{1}{\Theta(|\mathcal{S}||\mathcal{A}|)}
    \end{align}, \textbf{if} $N \le \frac{\log(|\mathcal{S}||\gA|)}{\log\log(|\mathcal{S}||\gA|)}$. 
Here, $\mathcal{S}$ and $\mathcal{A}$ represent the state space and action space, respectively.
\end{theorem}
To prove Theorem \ref{theorem : worstcase}, we examine the construction of the ``hardest option'', $\hat{\omega} \in \Omega_{\mathcal{X}}$, which has the lowest sampling probability in the tree $\mathcal{X}$. Since $\mathcal{P}_{\mathcal{X}}$ is an unknown distribution, we cannot directly exploit it. Instead, we construct a worst-case scenario to approximate the minimum option sampling probability. Now, we present the following Theorem on the sample complexity of $\epsilon t$-greedy.

\begin{theorem}[\textbf{\bm{$\epsilon t$}-greedy Sample Efficiency}] \label{thm:1}
    Given a state space $\mathcal{S}$, action space $\mathcal{A}$, and a set of options $\Omega_{\mathcal{X}}$ generated by $\epsilon t$-greedy for each tree $\mathcal{X}$, if $\mathcal{P}_{\mathcal{X}}[\omega] \ge \frac{1}{\Theta(|\mathcal{S}||\mathcal{A}|)}$, $\epsilon t$-greedy achieves polynomial sample complexity or i.e. is PAC-MDP.
\end{theorem}
Theorem \ref{theorem : worstcase} asserts that the sampling bound condition from Theorem \ref{thm:1} is satisfied when $N \leq \frac{\log(|\mathcal{S}||\gA|)}{\log\log(|\mathcal{S}||\gA|)}$. Theorem \ref{thm:1} establishes the necessary lower bound on the sampling probability of an option $\omega \in \Omega_{\gX}$ for any given exploration tree $\gX$, ensuring that the $\epsilon t$-greedy strategy is PAC-MDP under this criterion.

\begin{algorithm}[tb!] \small
\caption{Generating exploratory option with tree search}
\label{alg:exp with replay buffer}   
\begin{algorithmic}[1] 
\Function{}{}\textbf{generate\_option}(state s, hash function $\phi$, budget N)
\State frontier\_nodes $\gets \{ \}$
\State Initialize root using $s$: $\mathit{root} \gets \mathit{TreeNode}(s)$
\State frontier\_nodes $\gets$ frontier\_nodes $\cup $ \{root\};
\State $s_{\min}$ $\gets$ root
\State $i \gets 0$
\While{$ i < N$}
    \State $s_{x}$ $\sim$ \textit{UniformRandom}(frontier\_nodes)
    \State $s_{x^{\prime}}$= \textbf{next\_state\_from\_buffer}($s_{x}$)
    \If {$n(\phi(s_{x^{\prime}}))$=1}  \Comment{For a state to be in the replay buffer, it must have been visited at least once.}     
        \State Extract option $o$ by actions $root$ to $s_{x^\prime}$
        \State \Return $o$
    \EndIf
    \If {$n(\phi$($s_{x^{\prime}})$) $<$ $n(\phi$($s_{\min})$)}
        \State $s_{\min}$=$s_{x^{\prime}}$
    \EndIf
    \State $i \gets i + 1$
\EndWhile
\State Extract option $o$ by actions $root$ to $s_{\min}$
\State \Return $o$
\EndFunction
\State
\Function{}{}\textbf{next\_state\_from\_buffer}($s_{x}$, frontier\_nodes)
    \State $(s^{\prime},a,r,s{''}) \sim$ \textit{UnifromRandom}($\phi(s_{x})$)
    \State $s_{x^{\prime}} \leftarrow s{''}$ 
    \State $s_{x}$.add\_child($s_{x^{\prime}}$)
    \State frontier\_nodes $\gets$ frontier\_nodes $\cup$ \{$s_{x^{\prime}}$\}
    \State \Return $s_{x^{\prime}}$
\EndFunction
\end{algorithmic}
\end{algorithm}

\subsection{GDRB: Goal-conditioned Dual Replay Buffer}\label{grs-drb}

The experience replay buffer is an indispensable part of deep off-policy RL algorithms. It is common to use only one buffer to store all transitions and use FIFO as the retention policy, with the most recent data replacing the oldest data~\citep{mnih2013playing}. As an alternative, in the reservoir sampling \citep{vitter1985random} retention policy, each transition in the buffer has an equal chance of being overwritten. This maintains coverage of some older data over training. \emph{RS-DRB} \citep{zhang2019framework} uses two replay buffers, one for exploitation and the other for exploration. The transitions made by the agent's policy are stored in the exploitation buffer, and the random exploratory transitions are stored in the exploration buffer. For the retention policy, the exploration buffer uses reservoir sampling, while the exploitation buffer uses FIFO.

Inspired by this dual replay buffer framework, we propose a \emph{Goal-conditioned Double Replay Buffer (GDRB)}. The first buffer $D_{\beta}$ stores all transitions during training, and the second buffer $D_{e}$ stores the transitions that belong to successful episodes (i.e., goal reached). $D_{\beta}$ uses reservoir sampling, and $D_{e}$ uses FIFO. Since $D_{\beta}$ needs to cover transitions from the entire training process, it is larger than $D_{e}$. We balance the number of samples taken from the two buffers with an adaptive sampling ratio. Specifically, in a training process of $E$ episodes, at current episode $i$, the sampling ratios $\tau_{e}$ and $\tau_{\beta}$ for $D_{e}$ and $D_{\beta}$ are set as follows: $\tau_{e}= \frac{i}{E},   \tau_{\beta}= 1-\tau_{e}.$
To select $C$ mini-batches, $\max (\left \lfloor \tau_{\beta}*C \right \rfloor,1)$ mini-batches are chosen from $D_{\beta}$ and the rest from $D_{e}$. Later stages of training still sample from $D_{\beta}$ to not forget previously acquired knowledge, as we assume the policy is more likely to reach the goal as the training progresses. In case that $D_{e}$ is empty, since there are no successful episodes yet, we draw all mini-batches from $D_{\beta}$.

\subsection{Using Longest $n$-step Return}\label{n-step}

In standard DDPG, $Q$-values are updated using one-step TD. In goal-reaching tasks with sparse rewards, only one rewarded transition per successful episode is added to the replay buffer. The agent needs rewards provided by these transitions to update its policy toward reaching the goal. With few rewarded transitions, the agent should exploit a successful path to the goal many times so the reward is propagated backward quickly. Multi-step updates can accelerate this process by looking ahead several steps, resulting in more rewarded transitions in the replay buffer~\citep{meng2021effect, hessel2018rainbow}. For example, \citet{meng2021effect} utilize $n$-step updates in DDPG with $n$ ranging from $1$ to $8$. In our design, to share the reward from the last step of a successful episode for all transitions in the episode, we use \emph{longest $n$-step return} \citep{mnih2016asynchronous}, shown in Equation \ref{longest-step}.

\begin{equation} \label{longest-step} \small
Q(s_{t},a_{t})=
\begin{cases}
\sum^{T-t}_{k=0} \gamma^{k}r_{t+k}, \ \ \  \ \  \  \ \ \ \  \ \  \ \ \ \ \ \ \ \ \ \ \ \ \ \ \ \ \ \  \ \ \ \ \  \text{$s_{T}$ is a goal state}  \\
\\
\sum^{T-t-1}_{k=0} \gamma^{k}r_{t+k} + \gamma^{T-t}Q(s_{T},a_{T}), \ \ \ \ \ \text{otherwise}    
\end{cases}
\end{equation}

Here, $s_{T}$ is the last state in the episode. Using the longest n-step return for each transition from a successful episode, the reward is immediately propagated to all $Q$-value updates. In unsuccessful episodes, using the longest $n$-step return reduces the overestimation bias in Q-values~\citep{thrun1993issues}. \citet{meng2021effect} empirically show that using multi-step updates can improve the performance of DDPG on robotic tasks mostly by reducing overestimation bias --- they demonstrate that the larger the number of steps, the lower the estimated target Q-value and overestimation bias. 

\section{Experiments}\label{sec:experiments}

In this section, we show the details of how ETGL-DDPG improves DDPG for sparse-reward tasks using its three strategies. We use experiments to answer the following questions: 1) Can ETGL-DDPG outperform state-of-the-art methods in goal-reaching tasks with sparse rewards? 2) How does each of these three innovations impact the performance of DDPG? 3) Can $\epsilon t$-greedy explore more efficiently than $\epsilon z$-greedy and other common exploration strategies?

\begin{figure*}[t] \small
    \centering
        \begin{subfigure}[b]{60pt}
        \centering
        \includegraphics[width=\textwidth,height=2.3cm]{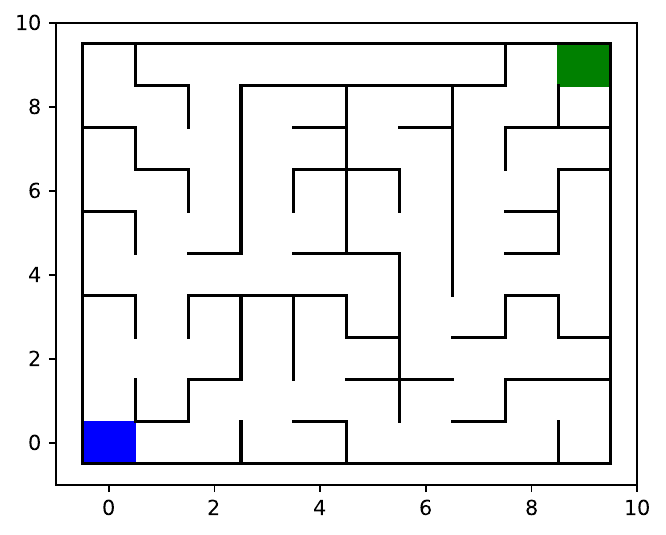}
        \caption{\scriptsize Wall-maze}
        \end{subfigure}
        \begin{subfigure}[b]{60pt}
        \centering
        \includegraphics[width=\textwidth,height=2.3cm]{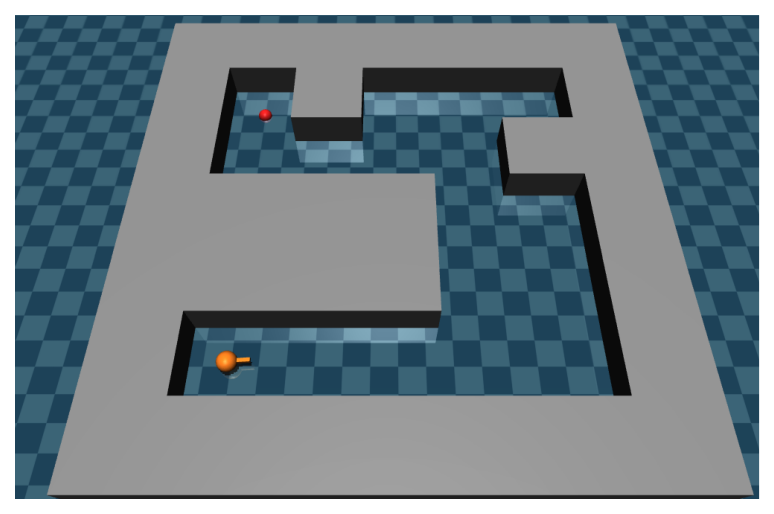}
        \caption{\scriptsize U-maze}
        \end{subfigure}
        \begin{subfigure}[b]{60pt}
        \centering
        \includegraphics[width=\textwidth,height=2.3cm]{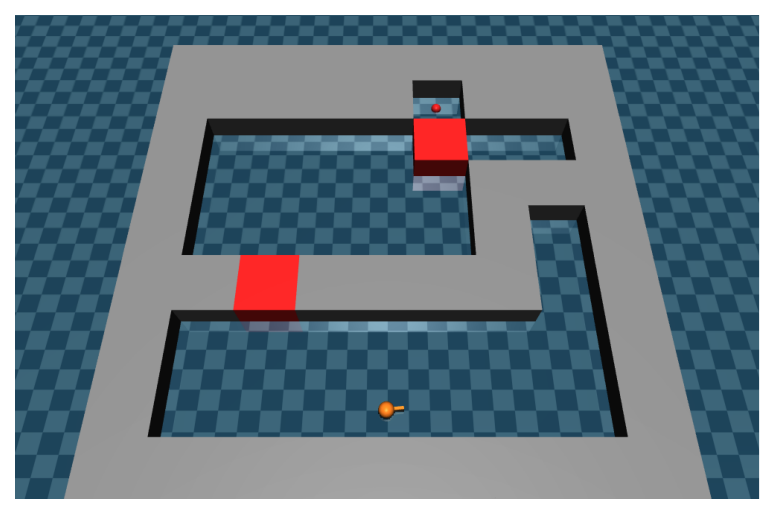}
        \caption{\scriptsize Point-push}
        \end{subfigure}
        \begin{subfigure}[b]{60pt}
        \centering
        \includegraphics[width=\textwidth,height=2.3cm]{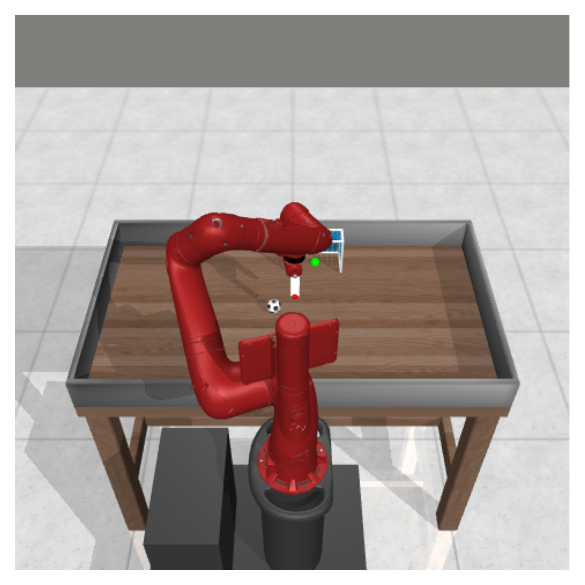}
        \caption{\scriptsize soccer}
        \end{subfigure}
        \begin{subfigure}[b]{60pt}
        \centering
        \includegraphics[width=\textwidth,height=2.3cm]{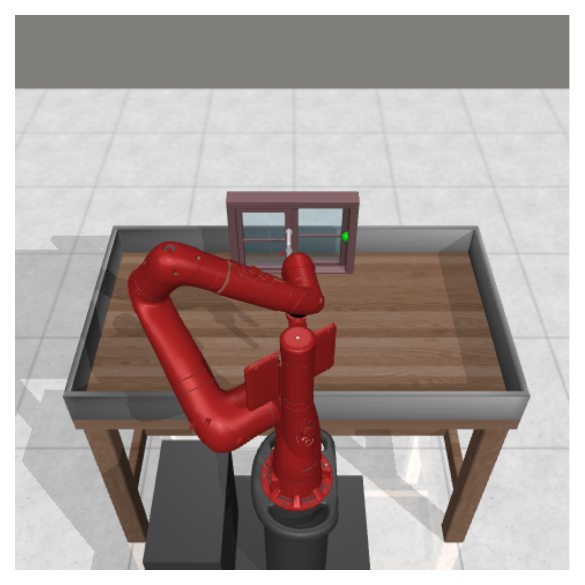}
        \caption{\scriptsize window-open}
        \end{subfigure}
        \begin{subfigure}[b]{60pt}
        \centering
        \includegraphics[width=\textwidth,height=2.3cm]{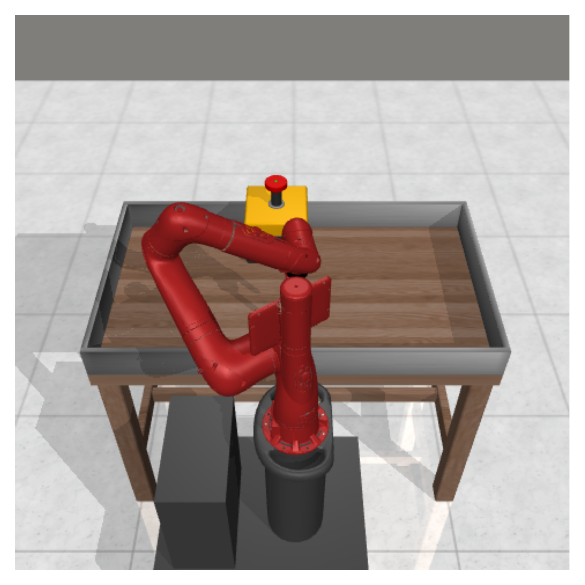}
        \caption{\scriptsize press-button}
        \end{subfigure}
    \caption{ The environments used in our experiments.}
    \label{fig: envs}
\end{figure*}

\begin{figure}[t]
    \centering
        \begin{subfigure}[b]{150pt}
        \centering
        \includegraphics[width=\textwidth]{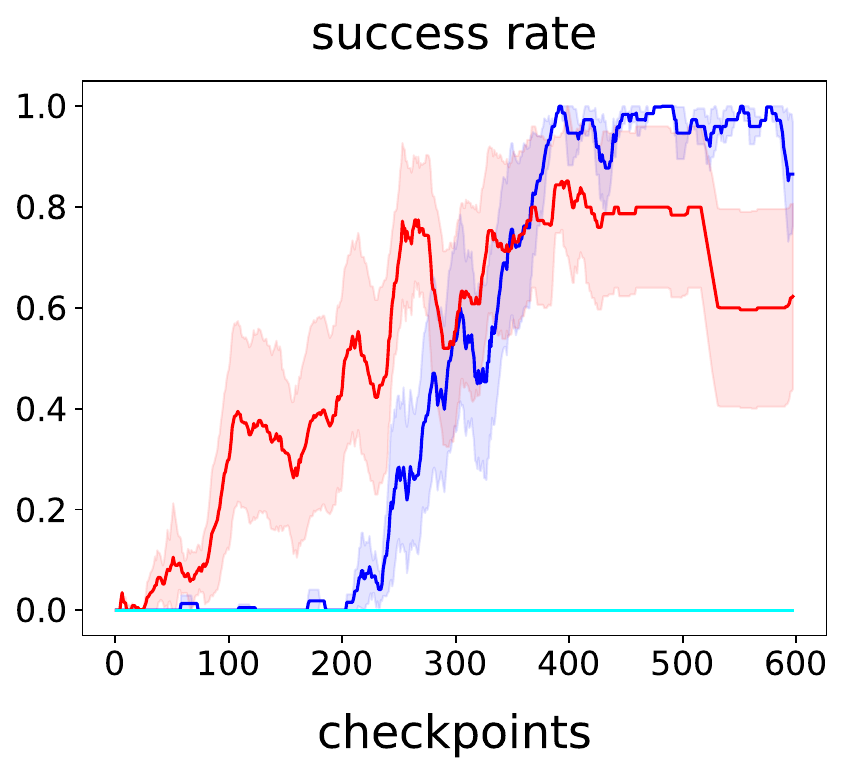}
        \caption{Wall-maze}
        \vspace{10pt}
        \end{subfigure}
        \begin{subfigure}[b]{150pt}
        \centering
        \includegraphics[width=\textwidth]{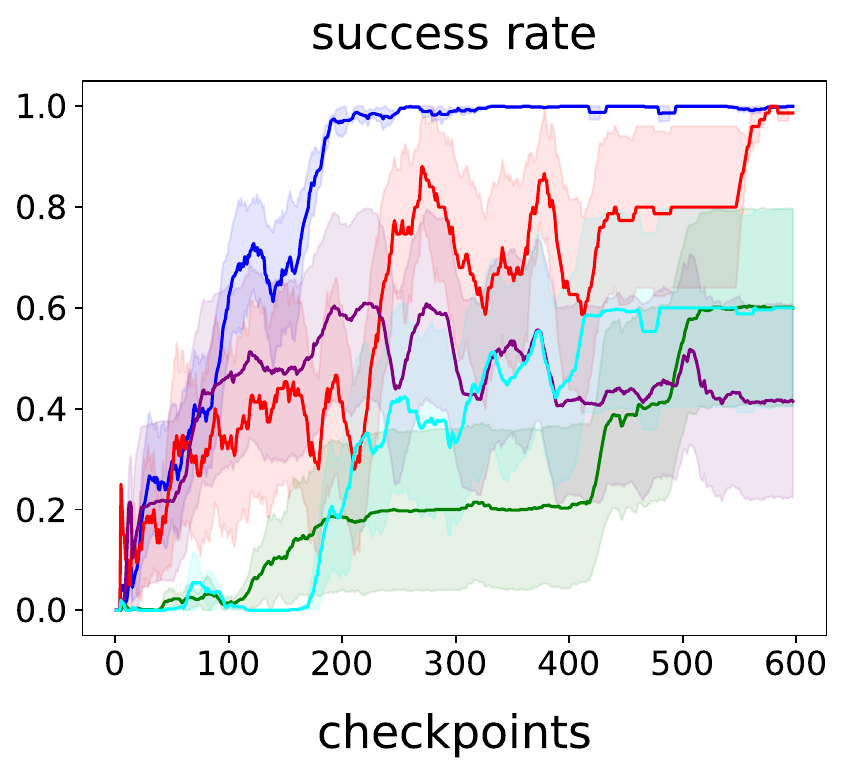}
        \caption{U-maze}
        \vspace{10pt}
        \end{subfigure}
        \begin{subfigure}[b]{150pt}
        \centering
        \includegraphics[width=\textwidth]{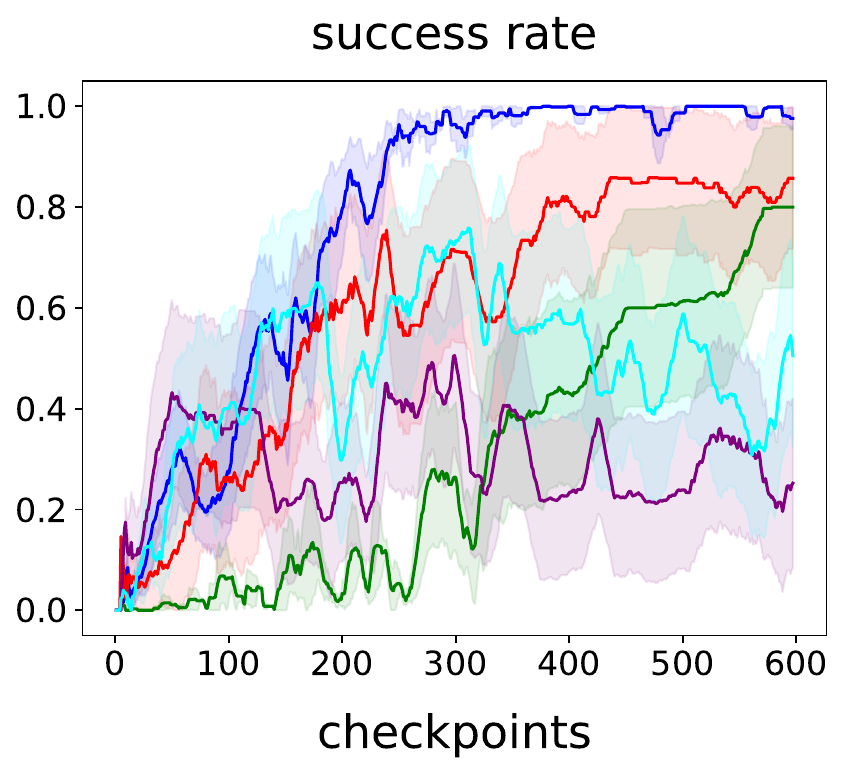}
        \caption{Point-push}
        \vspace{10pt}
        \end{subfigure}
        \begin{subfigure}[b]{150pt}
        \centering
        \includegraphics[width=\textwidth]{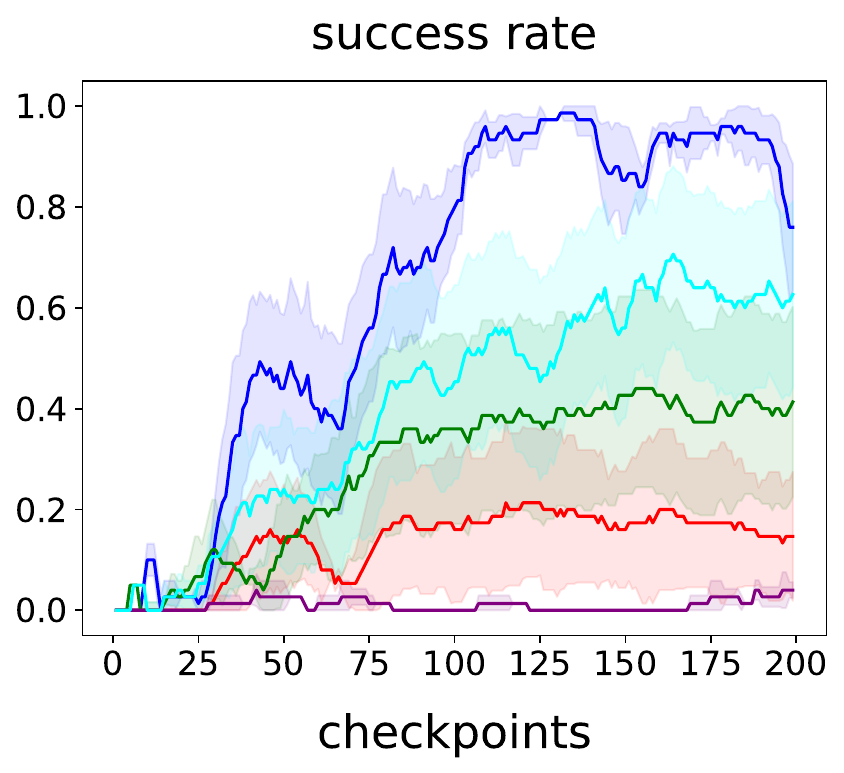}
        \caption{soccer}
        \end{subfigure}
        \begin{subfigure}[b]{150pt}
        \centering
        \includegraphics[width=\textwidth]{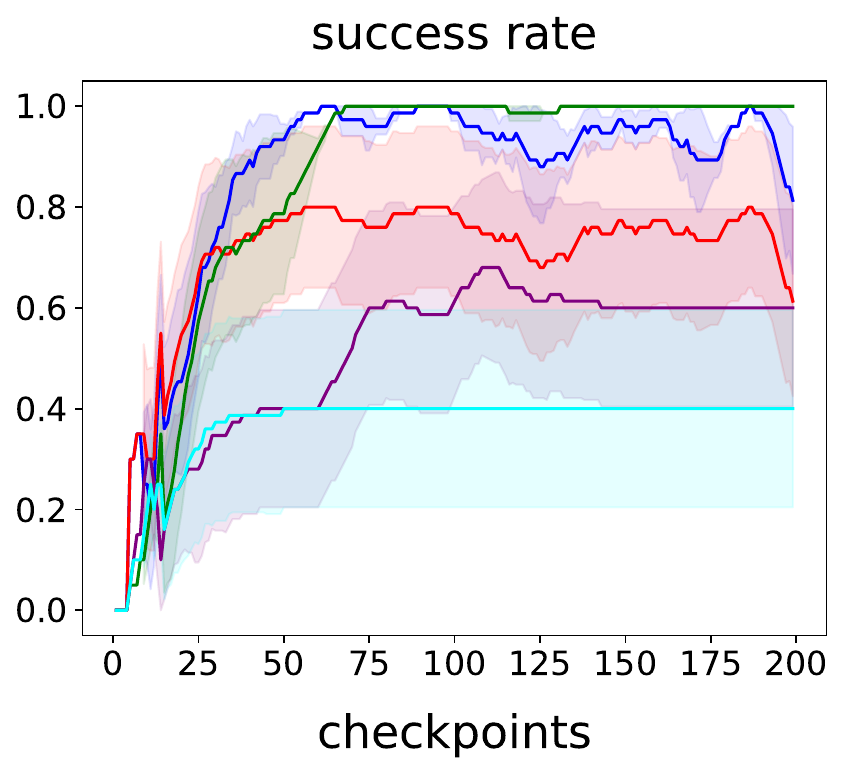}
        \caption{window-open}
        \end{subfigure}
        \begin{subfigure}[b]{150pt}
        \centering
        \includegraphics[width=\textwidth]{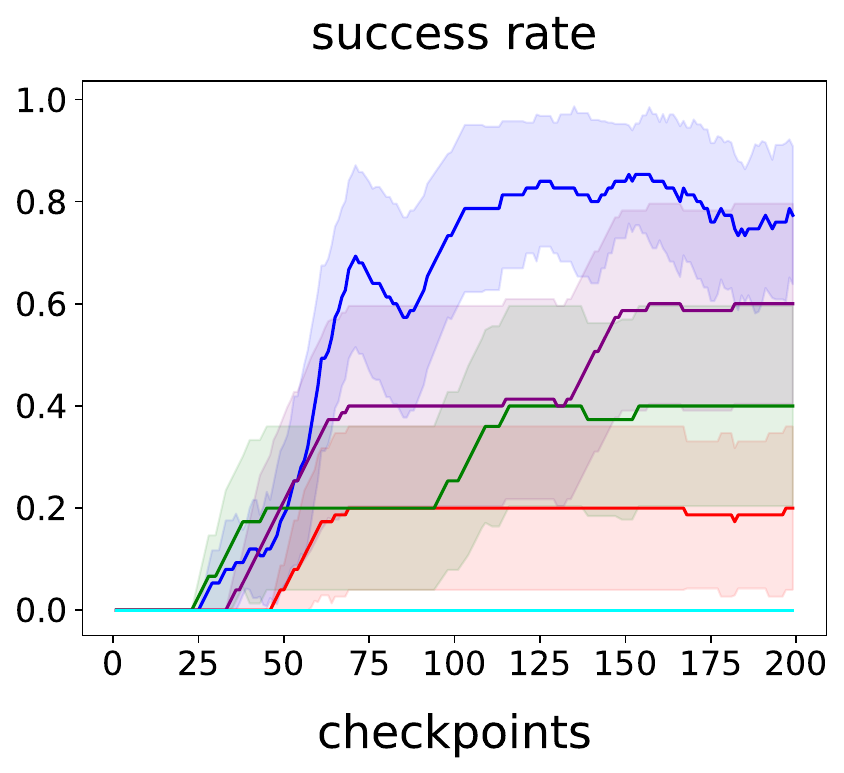}
        \caption{press-button}
        \end{subfigure}
        \hspace{-20pt} \includegraphics[width=0.6\textwidth]{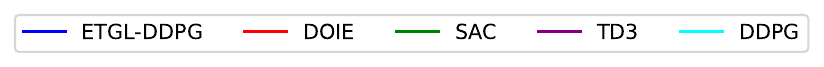}
    \caption{ The success rates across all environments, averaged over 5 runs with different random seeds. Shaded areas represent one standard deviation. We trained all methods for 6 million frames in the navigation environments and 2 million frames in the manipulation environments, with success rates reported at every $10^{5}$-step checkpoint. A moving average with a window size of 10 is applied to all methods for better readability.}
    \label{fig: overall performance}
\end{figure}


We consider two types of tasks: \emph{navigation} and \emph{manipulation}. We use three sparse-reward continuous environments for navigation. The first environment is a 2D maze called \emph{Wall-maze} \citep{trott2019keeping}, where a reward of -1 is given at each step, and a reward of 10 is given if the goal is reached. The start and goal states for each episode are randomly selected from the blue and green regions, respectively, as shown in Figure \ref{fig: envs}a. The agent’s action (dx,dy) determines the amount of movement along both axes. The environment contains a gradient cliff feature \citep{lehman2018more}, where the fastest way to reach the goal results in a deadlock close to the goal. Our second and third 3D environments are \emph{U-maze} (Figure \ref{fig: envs}b) and \emph{Point-push} (Figure \ref{fig: envs}c) \citep{github}, designed using the MuJoCo physics engine \citep{todorov2012mujoco}. In both environments, a robot (orange ball) seeks to reach the goal (red region). In Point-push, the robot must additionally push aside the two movable red blocks that obstruct the path to the goal. A small negative reward of -0.001 is given at each step unless the goal is reached, where the reward is 1. In each episode, the robot starts near the same position with slight random variations, but the goal region remains fixed.

We also employ three manipulation tasks: \emph{window-open}, \emph{soccer}, and \emph{button-press} (Figures \ref{fig: envs}d, e, and f) \citep{yu2020meta}. In window-open, the goal is to push the window open; in soccer, the goal is to kick the ball into the goal; and in button-press, the aim is to press the top-down button. Each episode begins with the robot's gripper in a randomized starting position, while the positions of other objects remain constant. The original versions of these tasks employ a uniquely shaped reward function for each task. However, these versions offer limited challenges for exploration, as standard baselines, such as SAC, demonstrate strong performance~\citep{yu2020meta}. We modified the original reward function to be sparse, transforming these tasks into challenging exploration problems.

The maximum number of steps per episode is set to 100 for Wall-maze and 500 for all other environments. Across all methods, the neural network architecture consists of 3 hidden layers with 128 units each, using ReLU activation functions. For standard baselines, we utilize the implementations from OpenAI Gym \citep{baselines}, and for other baselines, we rely on their publicly available implementations. After testing various configurations, we found that $\epsilon t$-greedy and $\epsilon z$-greedy perform best with budgets of $N=40$ and $N=15$, respectively, across these environments. Additional details about the environments and experimental setup are provided in Appendix \ref{sec: exp setup}.

\subsection{Overall Performance of ETGL-DDPG}\label{Performance Analysis}

We evaluate the performance of ETGL-DDPG compared to state-of-the-art methods. We compare with SAC \citep{haarnoja2018soft}, TD3~\citep{fujimoto2018addressing}, DDPG, and DOIE~\citep{lobel2022optimistic}. DOIE demonstrates state-of-the-art performance in challenging sparse-reward continuous control problems by drastically improving the exploration. While both DOIE and $\epsilon t$-greedy use a similarity measure between new and observed states, DOIE applies this to compute an optimistic value function rather than solely guiding the agent to unexplored areas. The results are shown in Figure \ref{fig: overall performance}. In the navigation environments, ETGL-DDPG and DOIE demonstrate superior performance compared to other methods, with ETGL-DDPG achieving a success rate of 1 faster than DOIE. Notably, Wall-maze presents a more challenging task among navigation environments, where only ETGL-DDPG and DOIE are able to achieve a success rate above zero. In manipulation tasks, the press-button poses the hardest challenge as none of the methods achieve a success rate of 1. ETGL-DDPG still outperforms all other approaches, while DOIE underperforms compared to SAC, indicating its limitations in adapting to high dimensional environments. Regarding training time, we observed that ETGL-DDPG requires approximately 1.5 times the training time of DDPG. This extended training duration can be attributed to two main components: $\epsilon t$-greedy and the longest $n$-step return. Among these, $\epsilon t$-greedy has a greater effect, as each step necessitates a local search.

\begin{figure*}[t]
    \centering
        \begin{subfigure}[b]{140pt}
        \centering
        \includegraphics[width=\textwidth]{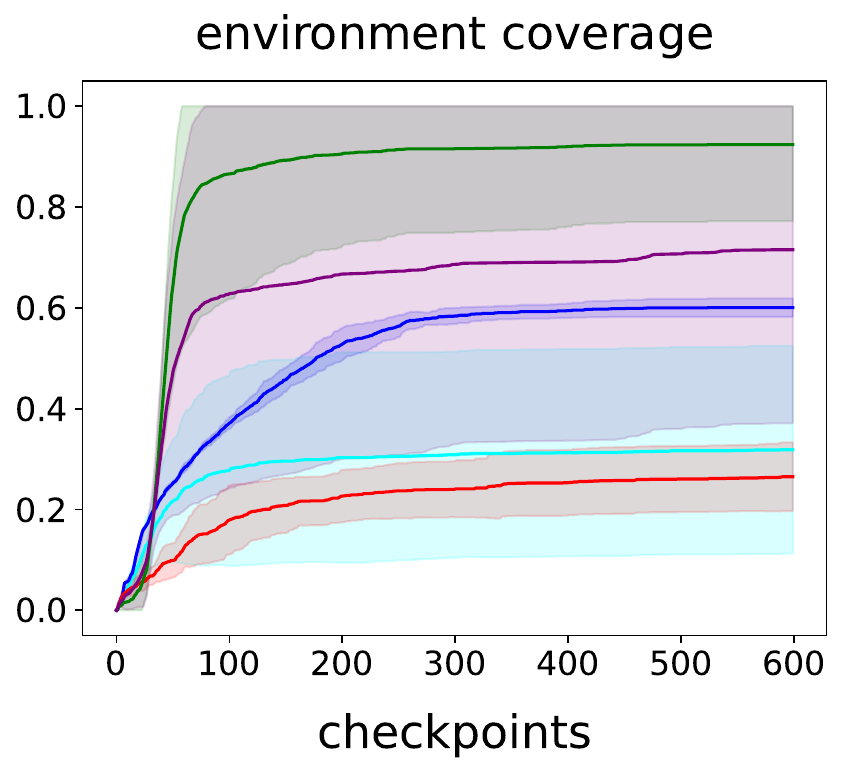}
        \caption{Wall-maze}
        \end{subfigure}
        \begin{subfigure}[b]{140pt}
        \centering
        \includegraphics[width=\textwidth]{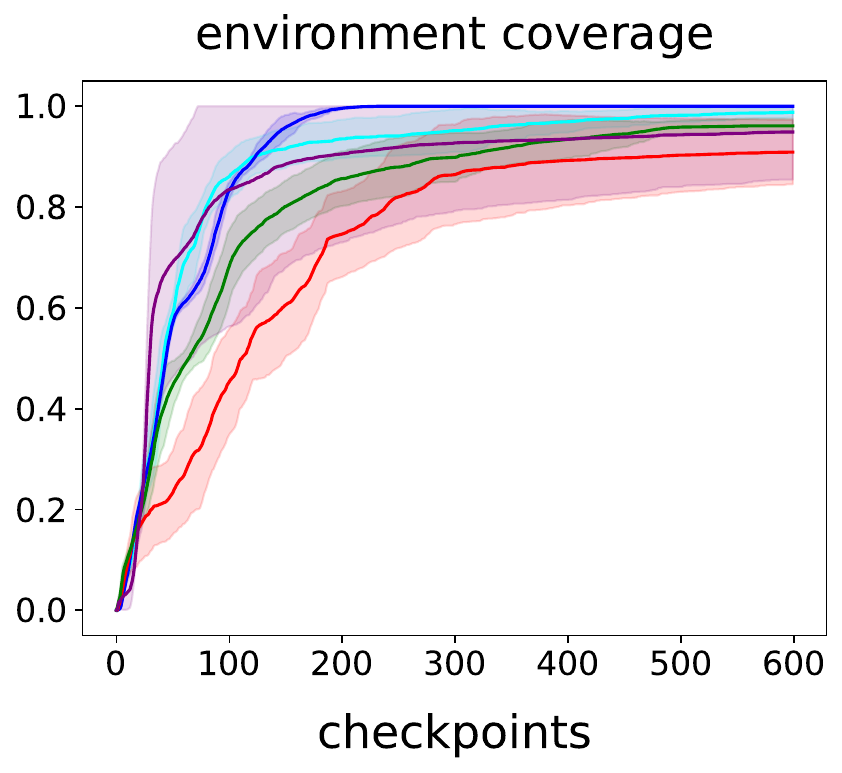}
        \caption{U-maze}
        \end{subfigure}
        \begin{subfigure}[b]{140pt}
        \centering
        \includegraphics[width=\textwidth]{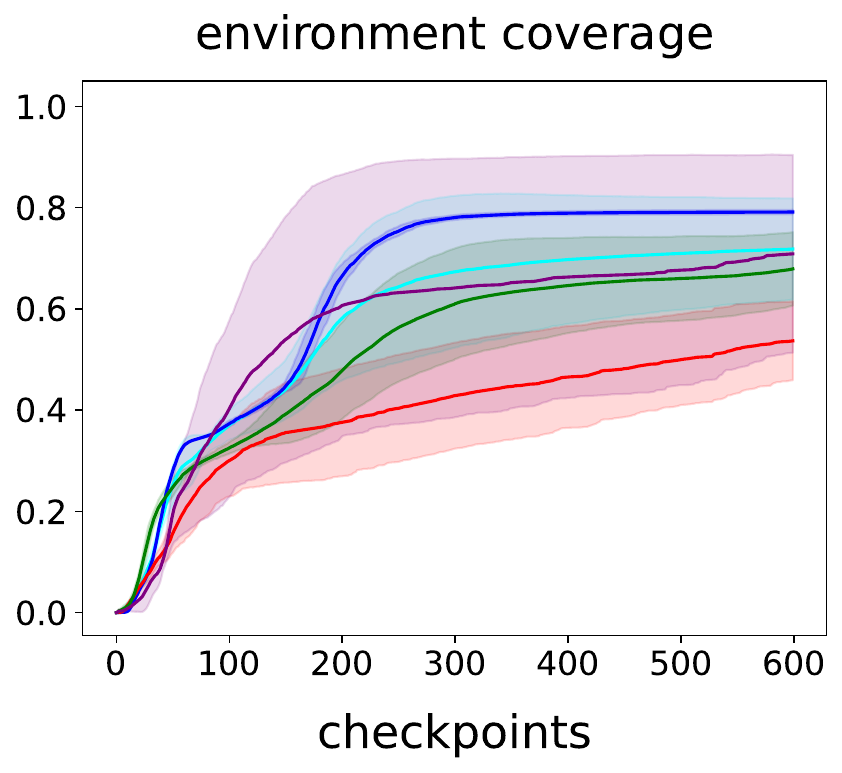}
        \caption{Point-push}
        \end{subfigure}
        \hspace{-20pt} \includegraphics[width=0.80\textwidth]{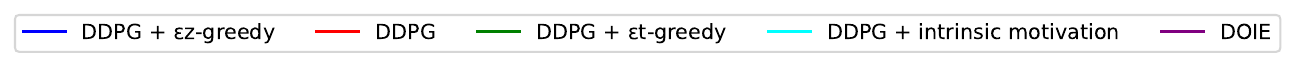}
    \caption{ The environment coverage for exploration strategies in navigation environments. On the graph, the y-axis indicates the portion of the environment that has been covered, and the checkpoints occur every $10^{4}$ steps shown on the x-axis. The results are given for the average of 10 runs with random seeds. The shaded region represents one standard deviation.}
    \label{fig: environment coverage}
\end{figure*}

\subsection{Environment Coverage through Exploration}\label{explore}

We now examine how effective $\epsilon t$-greedy is in covering the environment. To do so, we discretize the navigation environments into small cells. A cell is considered visited if the agent encounters a sufficient number of distinct states within it, and the overall environment coverage is quantified as the fraction of visited cells. Figure \ref{fig: environment coverage} presents a comparison of environment coverage across different exploration strategies. All strategies except DOIE, which uses Radial Basis Function Deep Q-Network (RBFDQN)~\citep{asadi2021deep}, use DDPG as their underlying algorithm. RBFDQN is an enhanced DQN variant that incorporates Radial Basis Functions (RBF) to achieve more accurate Q-value approximations in continuous environments. In Wall-maze, $\epsilon t$-greedy achieves 80\% coverage, whereas DOIE reaches 60\%. $\epsilon z$-greedy covers approximately half of the environment, whereas the remaining methods manage to explore only around 30\%. In U-maze, all strategies are successful, covering 80\% or more of the environment. Even so, $\epsilon z$-greedy reaches full coverage faster than other methods. In Point-push, none of the methods can fully cover the environment and ultimately achieve nearly the same coverage. The tree budget $N$ serves as an upper bound on the option length in $\epsilon t$-greedy, analogous to its role in $\epsilon z$-greedy when using a uniform distribution $z(n) = \mathbbm{1}_{n \leq N}/N$. We evaluate environment coverage under varying values of $N$ in Section~\ref{sec: impact of N}, where results indicate that $\epsilon t$-greedy achieves improved coverage as $N$ increases. In contrast, $\epsilon z$-greedy does not exhibit a consistent improvement in coverage and, in some cases, experiences a decline. This highlights the advantages of directed exploration over undirected approaches. Finally, we present the distribution of final states reached in the episodes to illustrate the order in which the agent visits different regions of the environment (see Appendix~\ref{sec: terminal dist}). 

In $\epsilon z$-greedy, various distributions can be used to select options based on their length. Although a uniform distribution is a straightforward choice, alternative distributions such as the zeta distribution $z(n) \propto n^{-\mu}$ (where $\mu$ controls the decay rate) can also be employed. Empirical results by \citep{dabney2020temporally} show that the zeta distribution slightly outperforms the uniform distribution, and moreover, produces a pattern akin to \emph{Lévy flights} observed in certain ecological foraging models \citep{baronchelli2013levy}. Figure~\ref{fig: option distribution} shows the distribution of option lengths generated by $\epsilon t$-greedy during training across all environments, revealing two main observations: (1) moderate-length options have the highest probability of being selected, and (2) probabilities decay as option length increases. While this decay pattern resembles that of the zeta distribution (where length~1 is most likely), $\epsilon t$-greedy instead favors moderate lengths. Note that Although options longer than 12 do occur, their frequency is extremely low and thus not visible in Figure~\ref{fig: option distribution}. 

\begin{figure*}[t]
    \centering
        \begin{subfigure}[b]{150pt}
        \centering
        \includegraphics[width=\textwidth]{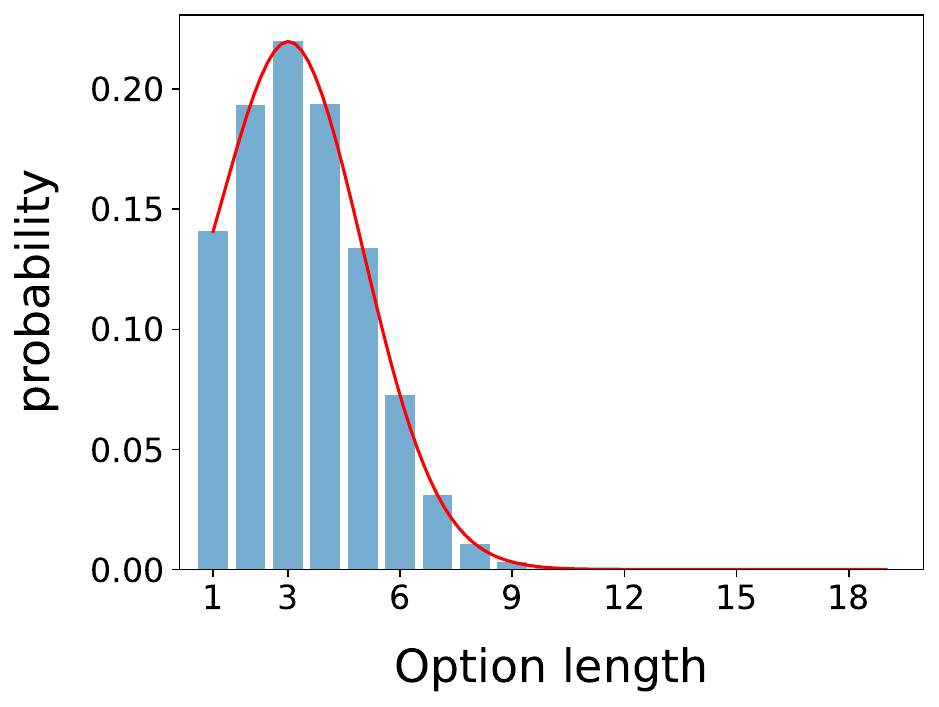}
        \caption{Wall-maze}
        \end{subfigure}
        \begin{subfigure}[b]{150pt}
        \centering
        \includegraphics[width=\textwidth]{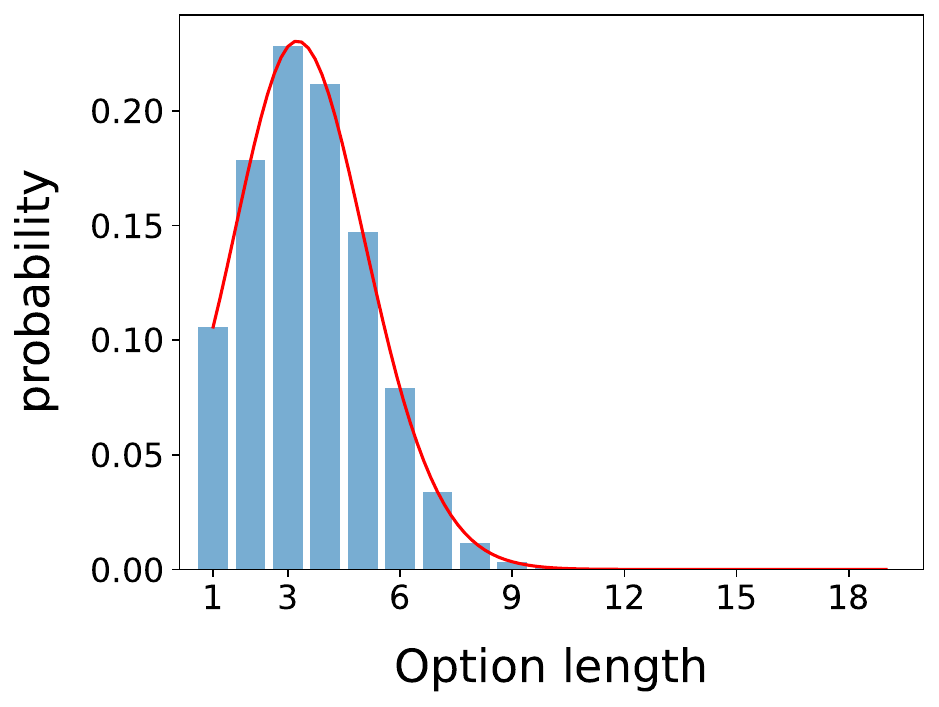}
        \caption{U-maze}
        \end{subfigure}
        \begin{subfigure}[b]{150pt}
        \centering
        \includegraphics[width=\textwidth]{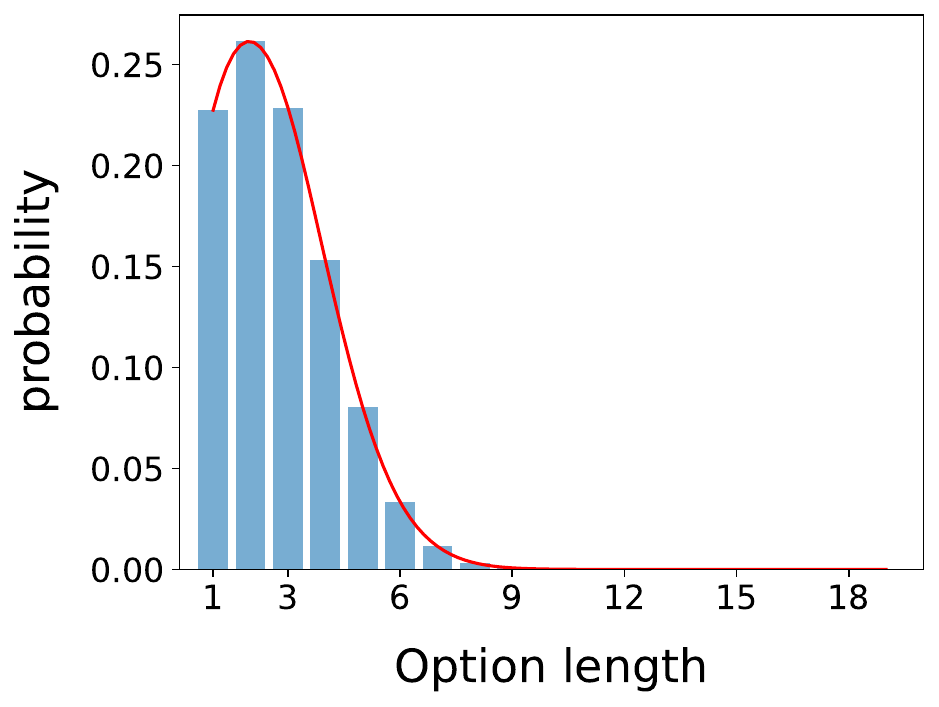}
        \caption{Point-push}
        \end{subfigure}
        \begin{subfigure}[b]{150pt}
        \centering
        \vspace{5pt}
        \includegraphics[width=\textwidth]{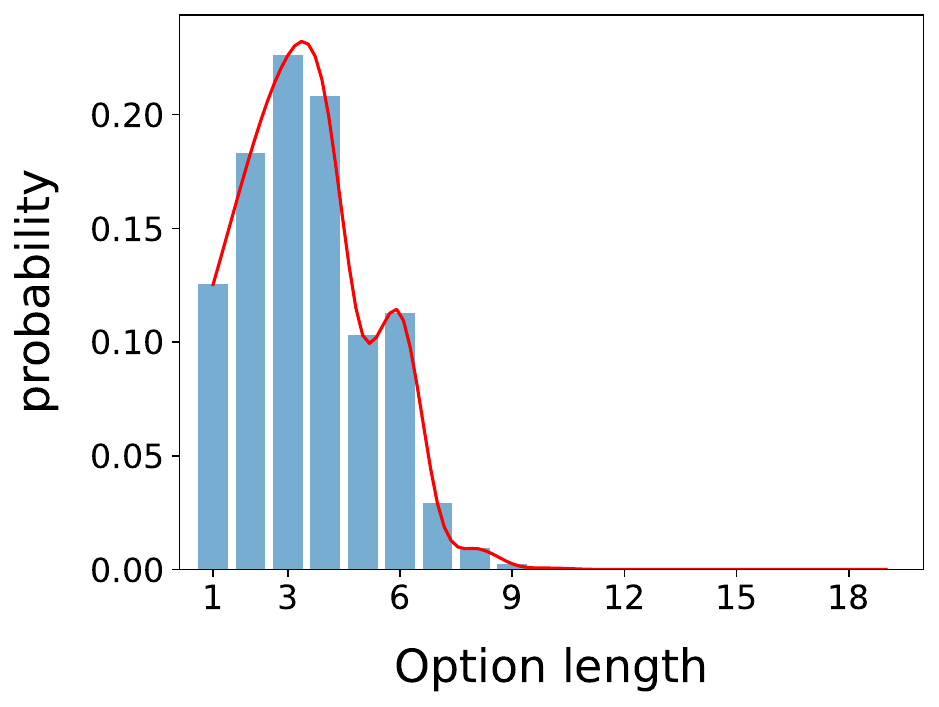}
        \caption{soccer}
        \end{subfigure}
        \begin{subfigure}[b]{150pt}
        \centering
        \includegraphics[width=\textwidth]{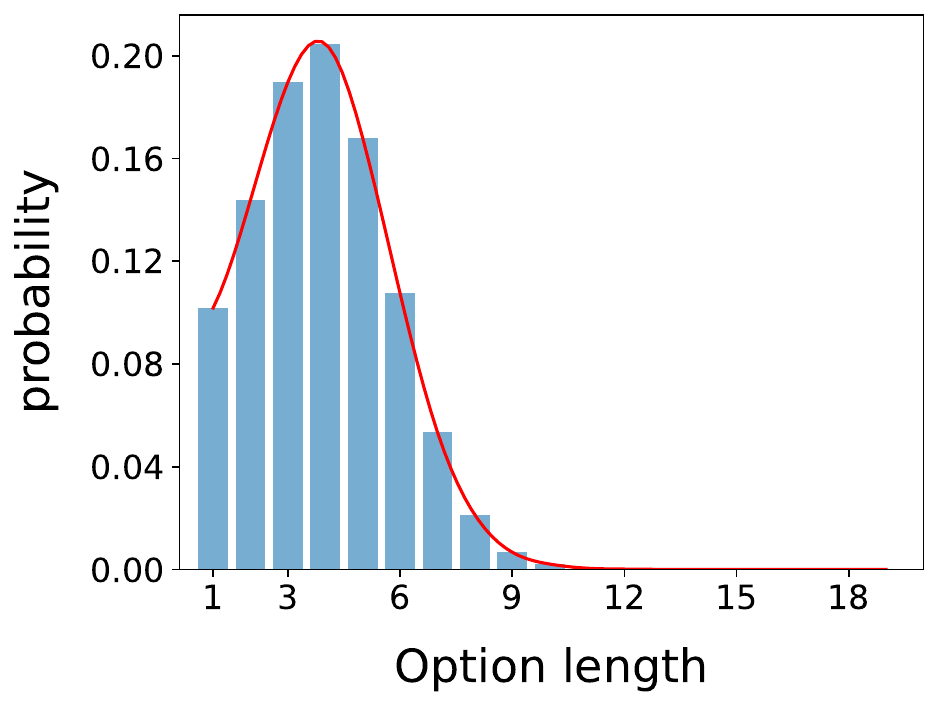}
        \caption{window-open}
        \end{subfigure}
        \begin{subfigure}[b]{150pt}
        \centering
        \includegraphics[width=\textwidth]{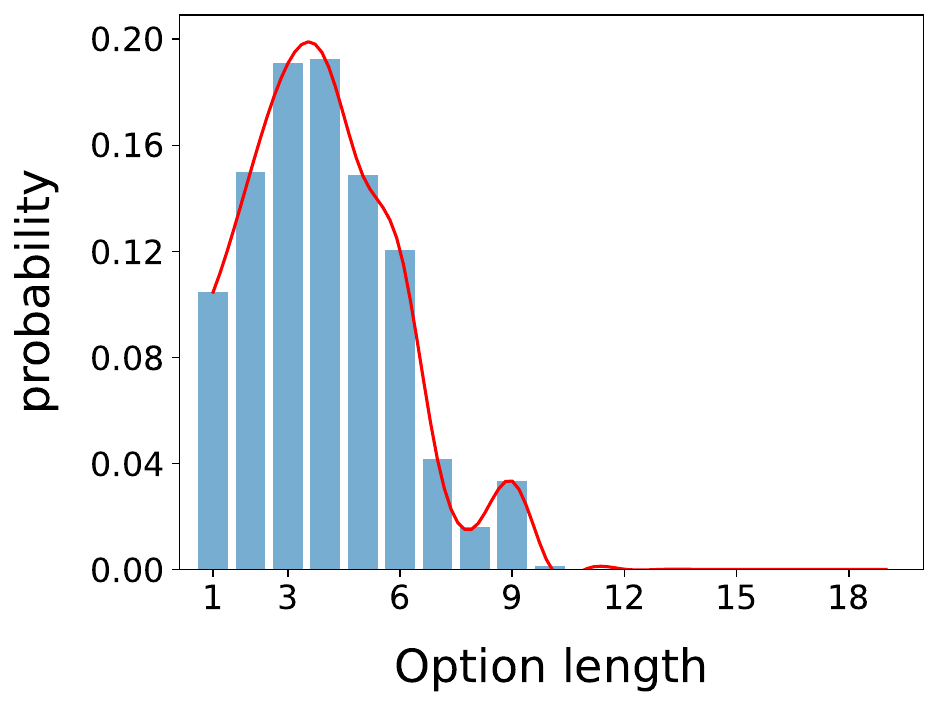}
        \caption{press-button}
        \end{subfigure}
    \caption{ The distribution of options chosen in training. The x-axis represents the length of the options and the y-axis indicates the probability of each length, calculated based on how often each length is chosen across all options.}
    \label{fig: option distribution}
\end{figure*}

\begin{figure*}[t]
    \centering
        \begin{subfigure}[b]{150pt}
        \centering
        \includegraphics[width=\textwidth]{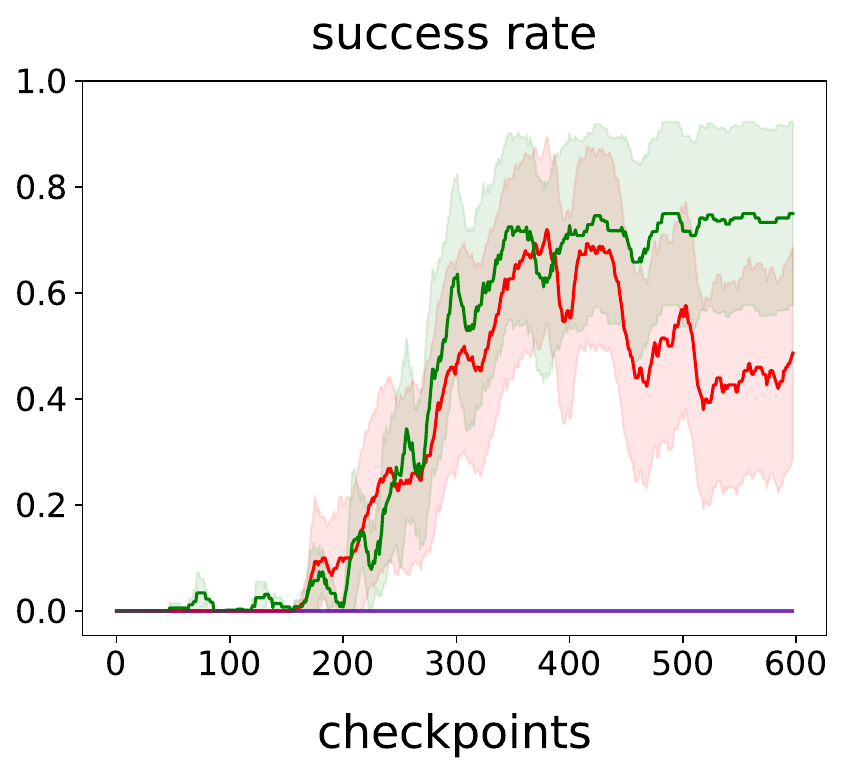}
        \caption{Wall-maze}
        \end{subfigure}
        \begin{subfigure}[b]{150pt}
        \centering
        \includegraphics[width=\textwidth]{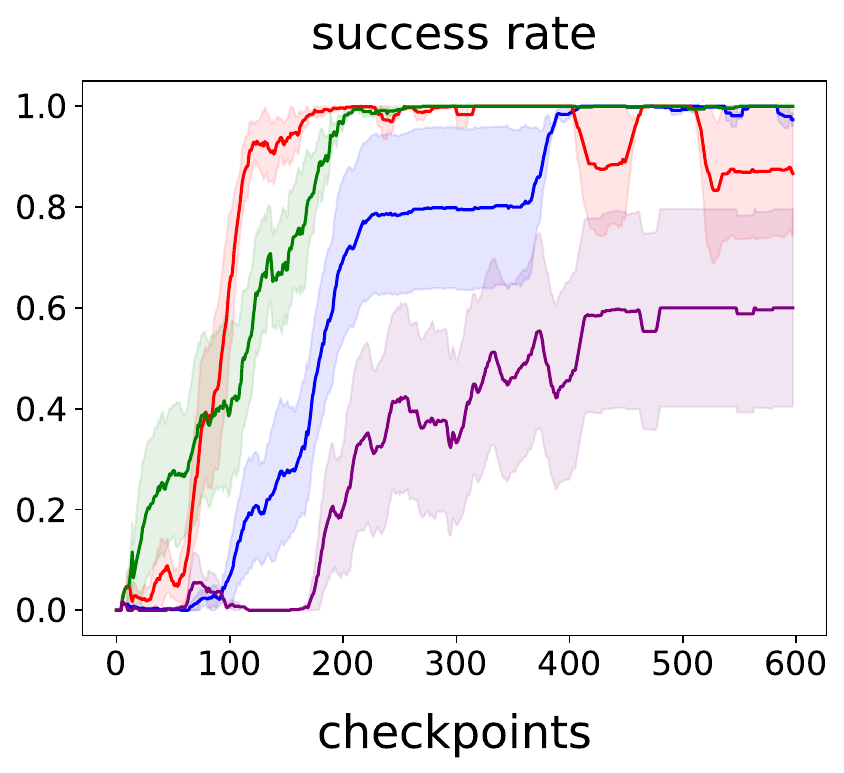}
        \caption{U-maze}
        \end{subfigure}
        \begin{subfigure}[b]{150pt}
        \centering
        \includegraphics[width=\textwidth]{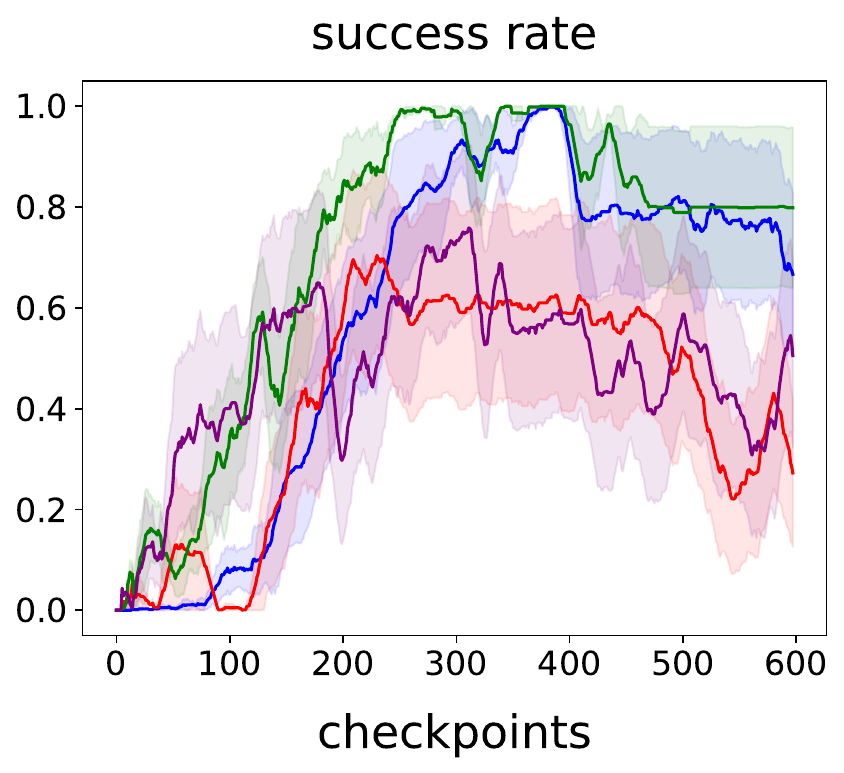}
        \caption{Point-push}
        \end{subfigure}
        \hspace{-20pt} \includegraphics[width=\textwidth]{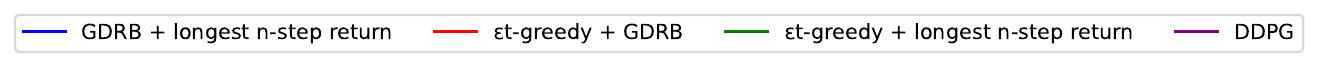}
    \caption{ The ablation study for two-component combinations in navigation environments. The success rates are averaged over 5 runs with different random seeds and the shaded areas represent one standard deviation.}
    \label{fig: leave-one-out ablation}
\end{figure*}

\subsection{Effectiveness of Each New Component in ETGL-DDPG}\label{ablation}

We evaluated the performance of ETGL-DDPG, and now we assess the impact of each component on DDPG separately. Figure \ref{fig: individual impact} presents the results for all environments. $\epsilon t$-greedy demonstrates the most improvement across all environments and is the only method that enhances the performance of DDPG in the Wall-maze, highlighting the critical role of our exploration strategy. GDRB shows a positive impact on DDPG performance in all environments, except for soccer, where DDPG alone outperforms all baselines. Additionally, we replaced reservoir sampling with FIFO as the retention policy in GDRB and observed similar results. The longest n-step return has a positive effect only in U-maze and press-button tasks, while it negatively impacts performance in soccer and Point-push. We attribute this to the inherently high variance of multi-step updates. A comparison of Figures \ref{fig: overall performance} and \ref{fig: individual impact} across all environments shows that ETGL-DDPG consistently outperforms the use of each component individually, supporting the effectiveness of their combination. 

To conclude the analysis of each component, we conduct three additional experiments in this section. First, we perform an ablation study by considering all possible two-component combinations, where the omitted component in each combination is replaced with its DDPG counterpart. The results, shown in Figure~\ref{fig: leave-one-out ablation} for the navigation environments, indicate that all combinations outperform DDPG in every environment except \emph{Wall-maze}, where the agent fails to achieve a non-zero success rate without $\epsilon$-greedy exploration. Moreover, in \emph{Point-push}, omitting the longest $n$-step return while following the original DDPG strategy causes a noticeable decline in success rates during the later stages of training.

Second, we compare the longest $n$-step return with the method proposed by \citet{meng2021effect}, who experimented with $n$ ranging from 1 to 8 and tested both the minimum of the $n$-step returns and their average. In their experiments, \emph{MMDDPG}, which computes the average return over 1 to 8 steps (\emph{avg8-step}), outperforms other variants in robotic tasks. Here, we replace the longest $n$-step return in ETGL-DDPG with the avg8-step return to evaluate its effectiveness in reward propagation. The results, presented in Appendix~\ref{n-step section}, show that both methods perform similarly across most environments, except in \emph{Wall-maze}, where the longest $n$-step return outperforms the avg8-step approach. Notably, the avg8-step method converges more rapidly, likely due to the stabilizing effect of averaging multiple updates.

Lastly, we compare GDRB with Hindsight Experience Replay (HER)~\cite{andrychowicz2017hindsight}, a method for sparse-reward tasks that increases feedback to the agent by treating certain states in unsuccessful episodes as imaginary goals. By assigning rewards to these artificially generated goals, HER can leverage unsuccessful episodes more effectively. Figure~\ref{fig: her-gdrb comparison} in Appendix~\ref{sec:her} compares the impact of GDRB and HER on DDPG in navigation environments. In \emph{Wall-maze}, HER enables the agent to learn from previously unsuccessful trajectories, allowing it to achieve success rates above zero. In the other two environments, both GDRB and HER exhibit similar performance, although HER converges more quickly due to its reward-reshaping mechanism, which provides stronger guidance in the initial stages of training.  

\begin{figure*}[t]
    \centering
        \begin{subfigure}[b]{150pt}
        \centering
        \includegraphics[width=\textwidth]{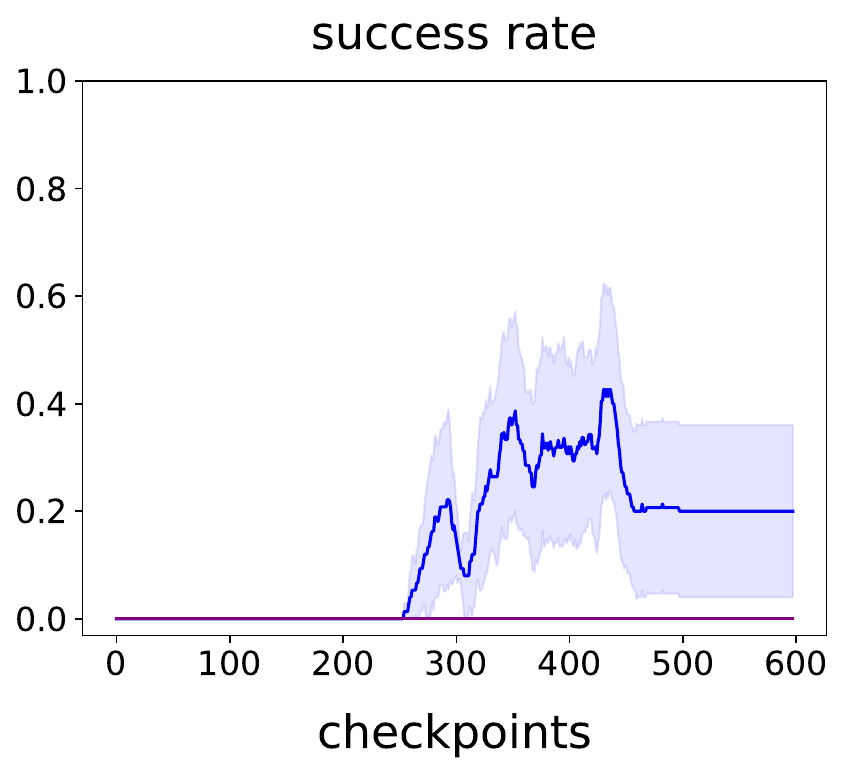}
        \caption{Wall-maze}
        \vspace{10pt}
        \end{subfigure}
        \begin{subfigure}[b]{150pt}
        \centering
        \includegraphics[width=\textwidth]{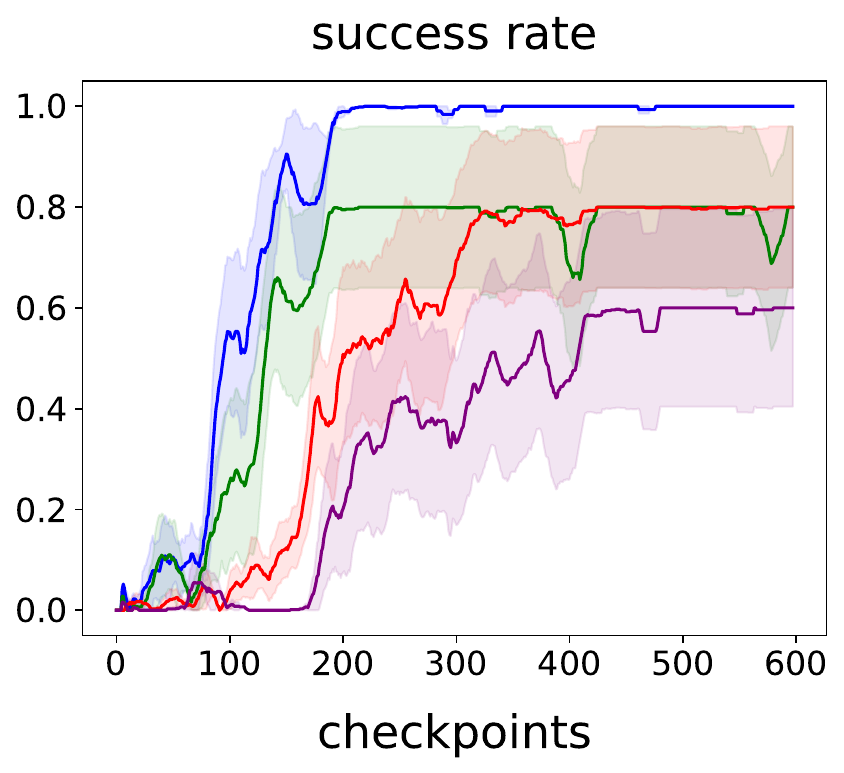}
        \caption{U-maze}
        \vspace{10pt}
        \end{subfigure}
        \begin{subfigure}[b]{150pt}
        \centering
        \includegraphics[width=\textwidth]{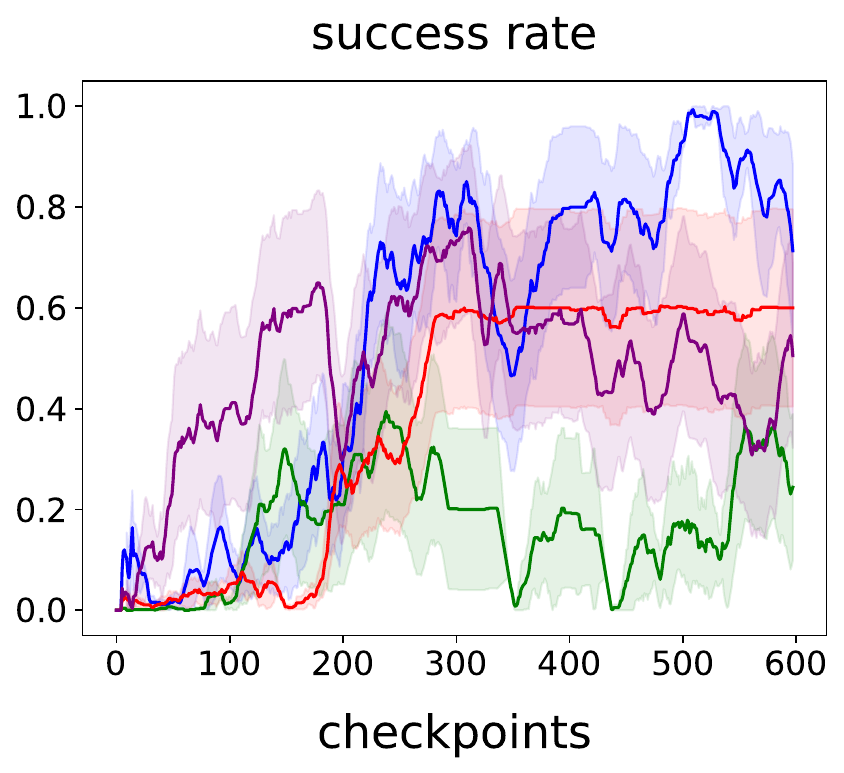}
        \caption{Point-push}
        \vspace{10pt}
        \end{subfigure}
        \begin{subfigure}[b]{150pt}
        \centering
        \includegraphics[width=\textwidth]{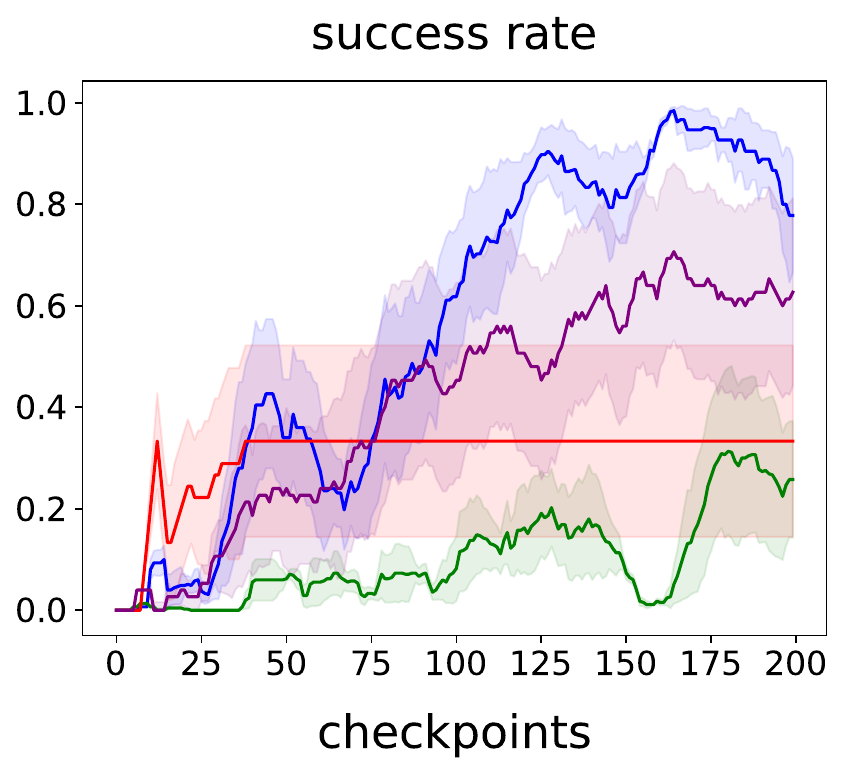}
        \caption{soccer}
        \end{subfigure}
        \begin{subfigure}[b]{150pt}
        \centering
        \includegraphics[width=\textwidth]{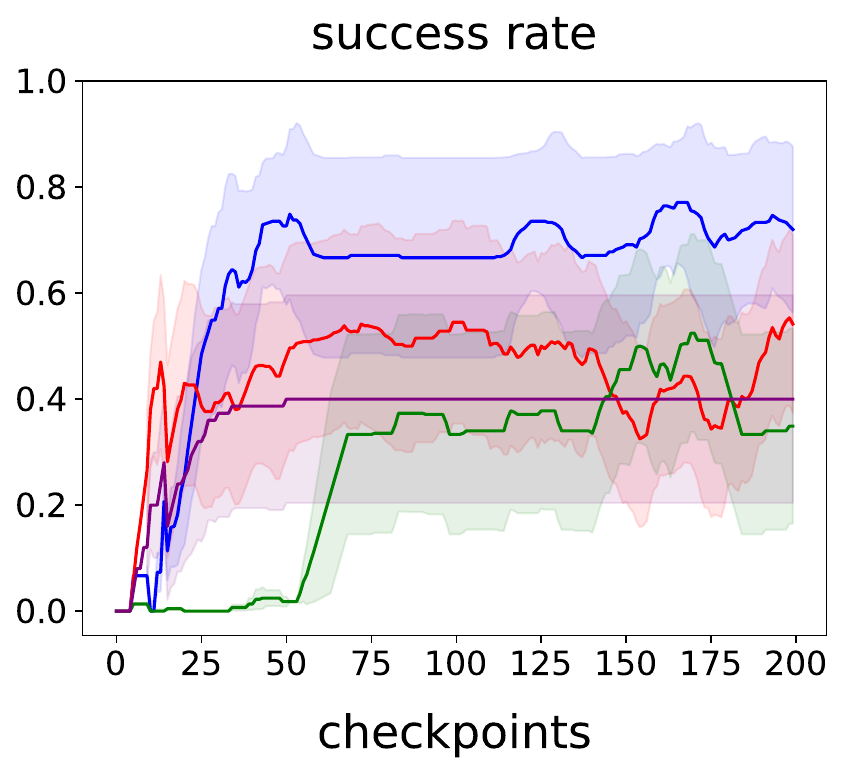}
        \caption{window-open}
        \end{subfigure}
        \begin{subfigure}[b]{150pt}
        \centering
        \includegraphics[width=\textwidth]{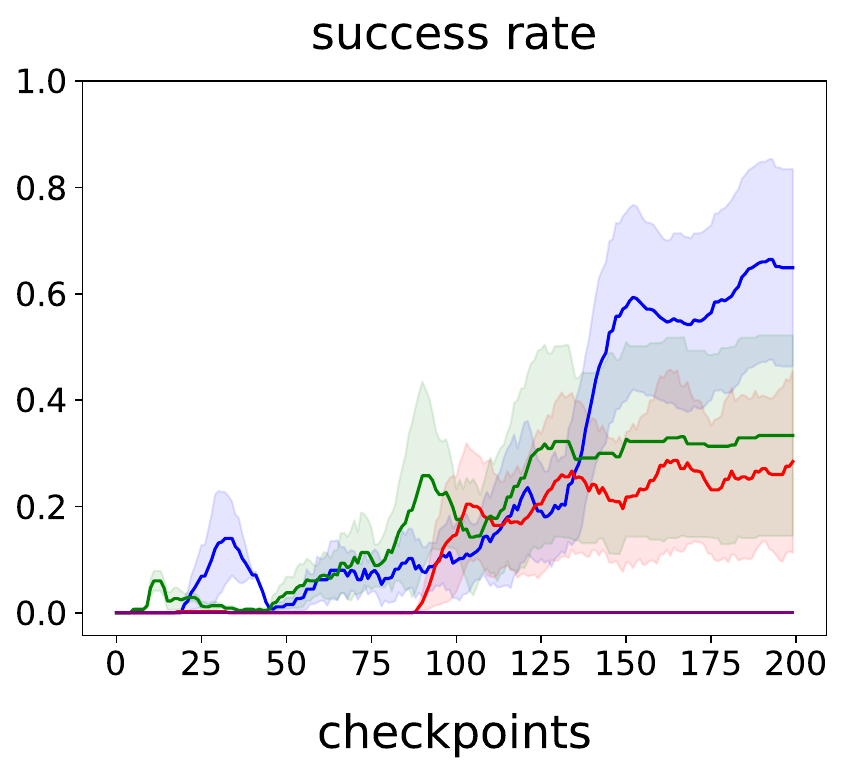}
        \caption{press-button}
        \end{subfigure}
        \hspace{-20pt} \includegraphics[width=0.80\textwidth]{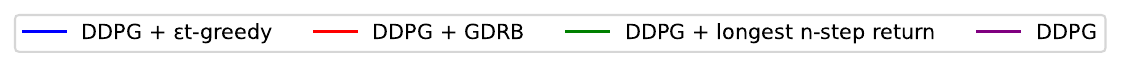}
    \caption{Analyzing the individual impact of three components on DDPG: $\epsilon t$-greedy, GDRB, and longest n-step return.}
    \label{fig: individual impact}
\end{figure*}

\section{Related Work}

\textbf{Intrinsic Motivation.} Intrinsic motivation methods ~\citep{burda2018exploration, pathak2017curiosity, ostrovski2017count,tang2017exploration} provide a reward bonus for unexplored areas of the state space, enabling the agent to receive more feedback in sparse-reward scenarios. However, since a transition can receive varying rewards at different time steps, these methods make the reward function non-stationary, violating the Markov assumption of MDPs. They are also sensitive to hyperparameters, requiring extensive tuning for each environment~\citep{schafer2021decoupled}. Decoupled RL algorithms~\citep{schafer2021decoupled, badia2019never} address this issue by training separate policies for exploration and exploitation, where the exploitation policy optimizes sparse rewards while leveraging data collected by the exploration policy. Although this resolves the non-stationarity issue, it doubles the computational cost. $\epsilon t$-greedy is similar to~\citet{tang2017exploration} in identifying less-explored states using the SimHash function. However, unlike intrinsic motivation methods, $\epsilon t$-greedy does not modify the reward function. Instead, it directs the agent toward less-explored states, preserving the stationarity of the reward function while avoiding the need for additional computation.
    
\textbf{Other Exploration Techniques.} There are also other strategies to improve exploration without using reward bonuses. \citet{colas2018gep} use a policy search process to generate diverse data to improve training of DDPG. \citet{liu2018competitive} introduce a competition-based exploration method where two agents (A and B) compete with each other. Agent A is penalized for visiting states visited by B, while B is rewarded for visiting states discovered by A. \citet{plappert2018parameter} directly inject noise into the policy's parameter space instead of the action space. \citet{eysenbach2019search} build a graph using states in the replay buffer, allowing the agent to navigate distant regions of the environment by applying Dijkstra's algorithm. $\epsilon t$-greedy shares similarities with \citet{eysenbach2019search} in constructing a tree structure within the state space. However, while \citet{eysenbach2019search} builds a single global graph to connect the start state to the goal state, $\epsilon t$-greedy generates local trees dynamically, focusing on accessing less-explored areas. \citet{ecoffet2021first} introduce \emph{Go-Explore} which separates exploration and policy optimization into two distinct phases. During the exploration phase, the algorithm systematically explores the environment by maintaining an archive of visited states along with their corresponding trajectories. In the robustification phase, these trajectories are leveraged to train a robust policy using standard reinforcement learning methods, such as PPO. However, a key limitation of Go-Explore is its reliance on resetting trajectories to return to previously visited states, a feature that is often unavailable in many real-world scenarios. \citet{lobel2022optimistic} present Deep Optimistic Initialization for Exploration (DOIE), which improves exploration in continuous control tasks by maintaining optimism in state-action value estimates. \citet{lobel2023flipping} demonstrate that DOIE can estimate visit counts by averaging samples from the Rademacher distribution instead of using density models. \citet{dey2024continual} present COIN, a continual optimistic initialization strategy that extends DOIE to stochastic and non-stationary environments. \citet{chakraborty2023dealing} leverage heavy-tailed action distributions to enhance exploration in continuous control tasks. \citet{wang2022learning} propose a framework in which learned high-level policies select among a set of pre-designed low-level base controllers, allowing the agent to decompose complex long-horizon tasks into simpler control primitives. \citet{dawood2023handling} use model predictive control (MPC), where a learned high-level policy proposes goals or waypoints, and MPC is used to generate low-level actions that guide the agent toward those goals.

\textbf{Experience Replay Buffer.} Rather than uniformly sampling from the buffer, Prioritized Experience Replay (PER) \citep{schaul2015prioritized} prioritizes transitions in the buffer based on reward, recency, or TD error at the expense of $O(\log N)$ per sample, where $N$ is the buffer size. CER~\citep{zhang2017deeper} includes the last transition from the buffer to each sampled batch with $O(1)$ complexity.~\citet{zhang2022replay} learn a conservative value regularizer only from the observed transitions in the replay buffer to improve the sample efficiency of DQN. \citet{pan2022understanding} theoretically show why PER has a better convergence rate than uniform sampling policy when minimizing mean squared error. Furthermore, \citet{pan2022understanding} identify two limitations of PER: outdated priorities and insufficient coverage of the state space. \citet{kompella2023event} propose Stratified Sampling from Event Tables (SSET), a method that partitions the experience replay buffer into distinct Event Tables. Each Event Table captures significant subsequences of optimal behavior, allowing for more targeted sampling during training. GDRB aligns conceptually with SSET if achieving a goal is defined as the sole event. However, SSET employs FIFO for inserting new data into its buffers, whereas GDRB uses reservoir sampling for its success-event buffer. Moreover, in SSET, the sampling ratios between buffers are fixed, while GDRB dynamically adjusts these ratios over time to favor the smaller buffer. 

\textbf{Reward Propagation.} Reward shaping~\citep{laud2004theory, hu2020learning} creates artificial intermediate rewards to facilitate reward propagation. However, designing appropriate intermediate rewards is hard and often problem-specific. \citet{trott2019keeping} address this issue by introducing ~\emph{self-balancing reward shaping} in the context of on-policy learning. Their approach defines the local optimum as the closest visited state to the goal and calculates the reward for each state based on its distance from the local optimum and the goal. By dynamically adjusting the local optimum, the agent is guided toward the optimal behavior. To extract more information from an unsuccessful episode, \citet{andrychowicz2017hindsight} introduce \emph{imaginary goals}. An imaginary goal for state $s$ is a state that is encountered later in the episode. Learning to achieve these imaginary goals helps the agent understand the structure of the environment. \citet{memarianself} propose \emph{Self-Supervised Online Reward Shaping (SORS)}, a method that alternates between ranking trajectories using sparse rewards and training a classifier to infer a dense reward function from these rankings. The inferred dense rewards are then utilized to update the agent's policy. \citet{devidze2024informativeness} present a reward informativeness criterion that adaptively constructs interpretable reward functions based on the agent’s current policy. \citet{wilcox2022monte} propose a Monte Carlo Augmented Actor-Critic (MCA2C) framework that integrates Monte Carlo returns from suboptimal demonstration trajectories into the critic update to improve value estimation in sparse-reward settings. By combining temporal-difference targets with demonstration-based Monte Carlo returns, the method provides more informative learning signals during early training, enabling the agent to better propagate sparse rewards. \citet{barth2018distributed} propose D4PG by utilizing distributional critic updates, $n$-step returns, and distributed training across multiple agents, which leads to improvements in wall-clock training time and data efficiency.

\section{Conclusions and Future Work}

We have introduced the ETGL-DDPG algorithm with three orthogonal components that improve the performance of DDPG for sparse-reward goal-conditioned environments. $\epsilon t$-greedy is a temporally-extended version of $\epsilon$-greedy using options generated by search. We prove that $\epsilon t$-greedy achieves a polynomial sample complexity under specific MDP structural assumptions. GDRB employs an extra buffer to separate successful episodes. The longest $n$-step return bootstraps from the Q-value of the final state in unsuccessful episodes and becomes a Monte Carlo update in successful episodes. ETGL-DDPG uses these components with DDPG and outperforms state-of-the-art methods, at the expense of about 1.5x wall-clock time w.r.t DDPG. The current limitation of our work is that we approximate visit counts through static hashing. For image-based problems such as real-world navigation, the future direction is to leverage dynamic hashing techniques such as \emph{normalizing flows}~\citep{papamakarios2021normalizing} as these tasks demand more intricate representation learning. Additionally, $\epsilon t$-greedy and GDRB are designed for deterministic domains and would require adaptation for stochastic environments. Extending these components to stochastic domains represents another promising direction for future research.

\bibliography{main}

@inproceedings{Kakade2003OnTS,
  title={On the sample complexity of reinforcement learning.},
  author={Sham M. Kakade},
  year={2003},
  url={https://api.semanticscholar.org/CorpusID:260534783}
}

@article{EvanEyalMansour,
author = {Even-Dar, Eyal and Mansour, Yishay},
title = {Learning Rates for Q-learning},
year = {2004},
issue_date = {12/1/2004},
publisher = {JMLR.org},
volume = {5},
issn = {1532-4435},
journal = {J. Mach. Learn. Res.},
month = dec,
pages = {1–25},
numpages = {25}
}

@article{lillicrap2015continuous,
  title={Continuous control with deep reinforcement learning},
  author={Lillicrap, Timothy P and Hunt, Jonathan J and Pritzel, Alexander and Heess, Nicolas and Erez, Tom and Tassa, Yuval and Silver, David and Wierstra, Daan},
  journal={arXiv preprint arXiv:1509.02971},
  year={2015}
}

@article{DBLP:journals/corr/abs-1805-09045,
  author       = {Yao Liu and
                  Emma Brunskill},
  title        = {When Simple Exploration is Sample Efficient: Identifying Sufficient
                  Conditions for Random Exploration to Yield {PAC} {RL} Algorithms},
  journal      = {CoRR},
  volume       = {abs/1805.09045},
  year         = {2018},
  eprinttype    = {arXiv},
  eprint       = {1805.09045},
  bibsource    = {dblp computer science bibliography, https://dblp.org}
}

@article{matheron2019problem,
  title={The problem with \uppercase{DDPG}: understanding failures in deterministic environments with sparse rewards},
  author={Matheron, Guillaume and Perrin, Nicolas and Sigaud, Olivier},
  journal={arXiv preprint arXiv:1911.11679},
  year={2019}
}

@inproceedings{haarnoja2018soft,
  title={Soft actor-critic: Off-policy maximum entropy deep reinforcement learning with a stochastic actor},
  author={Haarnoja, Tuomas and Zhou, Aurick and Abbeel, Pieter and Levine, Sergey},
  booktitle={International Conference on Machine Learning},
  pages={1861--1870},
  year={2018},
  organization={PMLR}
}

@article{andrychowicz2017hindsight,
  title={Hindsight experience replay},
  author={Andrychowicz, Marcin and Wolski, Filip and Ray, Alex and Schneider, Jonas and Fong, Rachel and Welinder, Peter and McGrew, Bob and Tobin, Josh and Pieter Abbeel, OpenAI and Zaremba, Wojciech},
  journal={Advances in \uppercase{n}eural \uppercase{i}nformation \uppercase{p}rocessing \uppercase{s}ystems},
  volume={30},
  year={2017}
}

@inproceedings{fujimoto2018addressing,
  title={Addressing function approximation error in actor-critic methods},
  author={Fujimoto, Scott and Hoof, Herke and Meger, David},
  booktitle={International Conference on Machine Learning},
  pages={1587--1596},
  year={2018},
  organization={PMLR}
}

@article{trott2019keeping,
  title={Keeping your distance: Solving sparse reward tasks using self-balancing shaped rewards},
  author={Trott, Alexander and Zheng, Stephan and Xiong, Caiming and Socher, Richard},
  journal={Advances in Neural Information Processing Systems},
  volume={32},
  year={2019}
}

@inproceedings{burda2018exploration,
  title={Exploration by random network distillation},
  author={Burda, Yuri and Edwards, Harrison and Storkey, Amos and Klimov, Oleg},
  booktitle={International Conference on Learning Representations},
  year={2018}
}

@inproceedings{pathak2017curiosity,
  title={Curiosity-driven exploration by self-supervised prediction},
  author={Pathak, Deepak and Agrawal, Pulkit and Efros, Alexei A and Darrell, Trevor},
  booktitle={International \uppercase{c}onference on \uppercase{m}achine \uppercase{l}earning},
  pages={2778--2787},
  year={2017},
  organization={PMLR}
}

@article{schafer2021decoupled,
  title={Decoupled reinforcement learning to stabilise intrinsically-motivated exploration},
  author={Sch{\"a}fer, Lukas and Christianos, Filippos and Hanna, Josiah P and Albrecht, Stefano V},
  journal={arXiv preprint arXiv:2107.08966},
  year={2021}
}

@inproceedings{ostrovski2017count,
  title={Count-based exploration with neural density models},
  author={Ostrovski, Georg and Bellemare, Marc G and Oord, A{\"a}ron and Munos, R{\'e}mi},
  booktitle={International \uppercase{c}onference on \uppercase{m}achine \uppercase{l}earning},
  pages={2721--2730},
  year={2017},
  organization={PMLR}
}

@inproceedings{dabney2020temporally,
  title={Temporally-Extended $\varepsilon$-Greedy Exploration},
  author={Dabney, Will and Ostrovski, Georg and Barreto, Andre},
  booktitle={International Conference on Learning Representations},
  year={2020}
}

@article{zhang2017deeper,
  title={A deeper look at experience replay},
  author={Zhang, Shangtong and Sutton, Richard S},
  journal={arXiv preprint arXiv:1712.01275},
  year={2017}
}

@article{schaul2015prioritized,
  title={Prioritized experience replay},
  author={Schaul, Tom and Quan, John and Antonoglou, Ioannis and Silver, David},
  journal={arXiv preprint arXiv:1511.05952},
  year={2015}
}

@inproceedings{zhang2019framework,
  title={A framework of dual replay buffer: balancing forgetting and generalization in reinforcement learning},
  author={Zhang, Linjing and Zhang, Zongzhang and Pan, Zhiyuan and Chen, Yingfeng and Zhu, Jiangcheng and Wang, Zhaorong and Wang, Meng and Fan, Changjie},
  booktitle={Proceedings of the 2nd Workshop on Scaling Up Reinforcement Learning (SURL), International Joint Conference on Artificial Intelligence (IJCAI)},
  year={2019}
}

@inproceedings{meng2021effect,
  title={The effect of multi-step methods on overestimation in deep reinforcement learning},
  author={Meng, Lingheng and Gorbet, Rob and Kuli{\'c}, Dana},
  booktitle={2020 25th International Conference on Pattern Recognition (ICPR)},
  pages={347--353},
  year={2021},
  organization={IEEE}
}

@article{vitter1985random,
  title={Random sampling with a reservoir},
  author={Vitter, Jeffrey S},
  journal={ACM Transactions on Mathematical Software (TOMS)},
  volume={11},
  number={1},
  pages={37--57},
  year={1985},
  publisher={ACM New York, NY, USA}
}

@article{sutton1999between,
  title={Between \uppercase{MDP}s and semi-\uppercase{MDP}s: A framework for temporal abstraction in reinforcement learning},
  author={Sutton, Richard S and Precup, Doina and Singh, Satinder},
  journal={Artificial intelligence},
  volume={112},
  number={1-2},
  pages={181--211},
  year={1999},
  publisher={Elsevier}
}

@inproceedings{silver2014deterministic,
  title={Deterministic policy gradient algorithms},
  author={Silver, David and Lever, Guy and Heess, Nicolas and Degris, Thomas and Wierstra, Daan and Riedmiller, Martin},
  booktitle={International \uppercase{c}onference on \uppercase{m}achine \uppercase{l}earning},
  pages={387--395},
  year={2014},
  organization={\uppercase{Pmlr}}
}

@inproceedings{colas2018gep,
  title={\uppercase {GEP-PG}: Decoupling exploration and exploitation in deep reinforcement learning algorithms},
  author={Colas, C{\'e}dric and Sigaud, Olivier and Oudeyer, Pierre-Yves},
  booktitle={International Conference on Machine Learning},
  pages={1039--1048},
  year={2018},
  organization={PMLR}
}

@article{papamakarios2021normalizing,
  title={Normalizing flows for probabilistic modeling and inference},
  author={Papamakarios, George and Nalisnick, Eric and Rezende, Danilo Jimenez and Mohamed, Shakir and Lakshminarayanan, Balaji},
  journal={The Journal of Machine Learning Research},
  volume={22},
  number={1},
  pages={2617--2680},
  year={2021},
  publisher={JMLRORG}
}

@inproceedings{thrun1993issues,
  title={Issues in using function approximation for reinforcement learning},
  author={Thrun, Sebastian and Schwartz, Anton},
  booktitle={Proceedings of the Fourth Connectionist Models Summer School},
  volume={255},
  pages={263},
  year={1993},
  organization={Hillsdale, NJ}
}

@inproceedings{hessel2018rainbow,
  title={Rainbow: Combining improvements in deep reinforcement learning},
  author={Hessel, Matteo and Modayil, Joseph and Van Hasselt, Hado and Schaul, Tom and Ostrovski, Georg and Dabney, Will and Horgan, Dan and Piot, Bilal and Azar, Mohammad and Silver, David},
  booktitle={Proceedings of the AAAI conference on artificial intelligence},
  volume={32},
  number={1},
  year={2018}
}

@inproceedings{mnih2016asynchronous,
  title={Asynchronous methods for deep reinforcement learning},
  author={Mnih, Volodymyr and Badia, Adria Puigdomenech and Mirza, Mehdi and Graves, Alex and Lillicrap, Timothy and Harley, Tim and Silver, David and Kavukcuoglu, Koray},
  booktitle={International \uppercase{c}onference on \uppercase{m}achine \uppercase{l}earning},
  pages={1928--1937},
  year={2016},
  organization={PMLR}
}

@inproceedings{lehman2018more,
  title={E\uppercase{S} is more than just a traditional finite-difference approximator},
  author={Lehman, Joel and Chen, Jay and Clune, Jeff and Stanley, Kenneth O},
  booktitle={Proceedings of the genetic and evolutionary computation conference},
  pages={450--457},
  year={2018}
}

@misc{github,
  author={Yuji Kanagawa},
  title={mujoco-maze},
  year={2021},
  publisher = {GitHub},
  journal = {GitHub repository},
  howpublished = {\url{https://github.com/kngwyu/mujoco-maze}},
}

@inproceedings{todorov2012mujoco,
  title={Mujoco: A physics engine for model-based control},
  author={Todorov, Emanuel and Erez, Tom and Tassa, Yuval},
  booktitle={2012 IEEE/RSJ \uppercase{i}nternational \uppercase{c}onference on \uppercase{i}ntelligent \uppercase{r}obots and \uppercase{s}ystems},
  pages={5026--5033},
  year={2012},
  organization={IEEE}
}

@misc{baselines,
  author = {Dhariwal, Prafulla and Hesse, Christopher and Klimov, Oleg and Nichol, Alex and Plappert, Matthias and Radford, Alec and Schulman, John and Sidor, Szymon and Wu, Yuhuai and Zhokhov, Peter},
  title = {Open\uppercase{AI} Baselines},
  year = {2017},
  publisher = {GitHub},
  journal = {GitHub repository},
  howpublished = {\url{https://github.com/openai/baselines}},
}

@article{lavalle1998rapidly,
  title={Rapidly-exploring random trees: A new tool for path planning},
  author={LaValle, Steven},
  journal={Research Report 9811},
  year={1998},
  publisher={Department of Computer Science, Iowa State University}
}

@inproceedings{badia2019never,
  title={Never Give Up: Learning Directed Exploration Strategies},
  author={Badia, Adri{\`a} Puigdom{\`e}nech and Sprechmann, Pablo and Vitvitskyi, Alex and Guo, Daniel and Piot, Bilal and Kapturowski, Steven and Tieleman, Olivier and Arjovsky, Martin and Pritzel, Alexander and Bolt, Andrew and others},
  booktitle={International Conference on Learning Representations},
  year={2019}
}

@book{sutton2018reinforcement,
  title={Reinforcement learning: An introduction},
  author={Sutton, Richard S and Barto, Andrew G},
  year={2018},
  publisher={MIT press}
}

@article{uhlenbeck1930theory,
  title={On the theory of the \uppercase{B}rownian motion},
  author={Uhlenbeck, George E and Ornstein, Leonard S},
  journal={Physical review},
  volume={36},
  number={5},
  pages={823},
  year={1930},
  publisher={APS}
}

@book{laud2004theory,
  title={Theory and application of reward shaping in reinforcement learning},
  author={Laud, Adam Daniel},
  year={2004},
  publisher={University of Illinois at Urbana-Champaign}
}

@article{mnih2013playing,
  title={Playing atari with deep reinforcement learning},
  author={Mnih, Volodymyr and Kavukcuoglu, Koray and Silver, David and Graves, Alex and Antonoglou, Ioannis and Wierstra, Daan and Riedmiller, Martin},
  journal={arXiv preprint arXiv:1312.5602},
  year={2013}
}

@article{eysenbach2019search,
  title={Search on the replay buffer: Bridging planning and reinforcement learning},
  author={Eysenbach, Ben and Salakhutdinov, Russ R and Levine, Sergey},
  journal={Advances in Neural Information Processing Systems},
  volume={32},
  year={2019}
}

@inproceedings{charikar2002similarity,
  title={Similarity estimation techniques from rounding algorithms},
  author={Charikar, Moses S},
  booktitle={Proceedings of the thiry-fourth annual ACM symposium on Theory of computing},
  pages={380--388},
  year={2002}
}

@inproceedings{liu2018competitive,
  title={Competitive experience replay},
  author={Liu, Hao and Trott, Alexander and Socher, Richard and Xiong, Caiming},
  booktitle={International Conference on Learning Representations},
  year={2018}
}

@inproceedings{plappert2018parameter,
  title={Parameter Space Noise for Exploration},
  author={Plappert, Matthias and Houthooft, Rein and Dhariwal, Prafulla and Sidor, Szymon and Chen, Richard Y and Chen, Xi and Asfour, Tamim and Abbeel, Pieter and Andrychowicz, Marcin},
  booktitle={International Conference on Learning Representations},
  year={2018}
}

@article{hu2020learning,
  title={Learning to utilize shaping rewards: A new approach of reward shaping},
  author={Hu, Yujing and Wang, Weixun and Jia, Hangtian and Wang, Yixiang and Chen, Yingfeng and Hao, Jianye and Wu, Feng and Fan, Changjie},
  journal={Advances in Neural Information Processing Systems},
  volume={33},
  pages={15931--15941},
  year={2020}
}

@inproceedings{zhang2022replay,
  title={Replay memory as an empirical MDP: Combining conservative estimation with experience replay},
  author={Zhang, Hongming and Xiao, Chenjun and Wang, Han and Jin, Jun and M{\"u}ller, Martin and others},
  booktitle={The Eleventh International Conference on Learning Representations},
  year={2022}
}

@inproceedings{pan2022understanding,
  title={Understanding and mitigating the limitations of prioritized experience replay},
  author={Pan, Yangchen and Mei, Jincheng and Farahmand, Amir-massoud and White, Martha and Yao, Hengshuai and Rohani, Mohsen and Luo, Jun},
  booktitle={Uncertainty in Artificial Intelligence},
  pages={1561--1571},
  year={2022},
  organization={PMLR}
}

@article{bloom1970space,
  title={Space/time trade-offs in hash coding with allowable errors},
  author={Bloom, Burton H},
  journal={Communications of the ACM},
  volume={13},
  number={7},
  pages={422--426},
  year={1970},
  publisher={ACM New York, NY, USA}
}

@inproceedings{yu2020meta,
  title={Meta-world: A benchmark and evaluation for multi-task and meta reinforcement learning},
  author={Yu, Tianhe and Quillen, Deirdre and He, Zhanpeng and Julian, Ryan and Hausman, Karol and Finn, Chelsea and Levine, Sergey},
  booktitle={Conference on Robot Learning},
  pages={1094--1100},
  year={2020},
  organization={PMLR}
}

@inproceedings{lobel2022optimistic,
  title={Optimistic initialization for exploration in continuous control},
  author={Lobel, Sam and Gottesman, Omer and Allen, Cameron and Bagaria, Akhil and Konidaris, George},
  booktitle={Proceedings of the AAAI Conference on Artificial Intelligence},
  volume={36},
  number={7},
  pages={7612--7619},
  year={2022}
}

@article{tang2017exploration,
  title={\# exploration: A study of count-based exploration for deep reinforcement learning},
  author={Tang, Haoran and Houthooft, Rein and Foote, Davis and Stooke, Adam and Xi Chen, OpenAI and Duan, Yan and Schulman, John and DeTurck, Filip and Abbeel, Pieter},
  journal={Advances in neural information processing systems},
  volume={30},
  year={2017}
}

@inproceedings{asadi2021deep,
  title={Deep radial-basis value functions for continuous control},
  author={Asadi, Kavosh and Parikh, Neev and Parr, Ronald E and Konidaris, George D and Littman, Michael L},
  booktitle={Proceedings of the AAAI Conference on Artificial Intelligence},
  volume={35},
  number={8},
  pages={6696--6704},
  year={2021}
}

@inproceedings{duan2016benchmarking,
  title={Benchmarking deep reinforcement learning for continuous control},
  author={Duan, Yan and Chen, Xi and Houthooft, Rein and Schulman, John and Abbeel, Pieter},
  booktitle={International conference on machine learning},
  pages={1329--1338},
  year={2016},
  organization={PMLR}
}

@inproceedings{dey2024continual,
  title={Continual Optimistic Initialization for Value-Based Reinforcement Learning},
  author={Dey, Sheelabhadra and Ault, James and Sharon, Guni},
  booktitle={Proceedings of the 23rd International Conference on Autonomous Agents and Multiagent Systems},
  pages={453--462},
  year={2024}
}

@inproceedings{lobel2023flipping,
  title={Flipping coins to estimate pseudocounts for exploration in reinforcement learning},
  author={Lobel, Sam and Bagaria, Akhil and Konidaris, George},
  booktitle={International Conference on Machine Learning},
  pages={22594--22613},
  year={2023},
  organization={PMLR}
}

@inproceedings{devidze2024informativeness,
  title={Informativeness of Reward Functions in Reinforcement Learning},
  author={Devidze, Rati and Kamalaruban, Parameswaran and Singla, Adish},
  booktitle={23rd International Conference on Autonomous Agents and Multiagent Systems},
  pages={444--452},
  year={2024},
  organization={ACM}
}

@article{pan2020softmax,
  title={Softmax deep double deterministic policy gradients},
  author={Pan, Ling and Cai, Qingpeng and Huang, Longbo},
  journal={Advances in neural information processing systems},
  volume={33},
  pages={11767--11777},
  year={2020}
}

@inproceedings{DBLP:conf/iclr/Barth-MaronHBDH18,
  author       = {Gabriel Barth{-}Maron and
                  Matthew W. Hoffman and
                  David Budden and
                  Will Dabney and
                  Dan Horgan and
                  Dhruva TB and
                  Alistair Muldal and
                  Nicolas Heess and
                  Timothy P. Lillicrap},
  title        = {Distributed Distributional Deterministic Policy Gradients},
  booktitle    = {6th International Conference on Learning Representations, {ICLR} 2018,
                  Vancouver, BC, Canada, April 30 - May 3, 2018, Conference Track Proceedings},
  publisher    = {OpenReview.net},
  year         = {2018},
  url          = {https://openreview.net/forum?id=SyZipzbCb},
  bibsource    = {dblp computer science bibliography, https://dblp.org}
}

@inproceedings{tiapkin2024generative,
  title={Generative flow networks as entropy-regularized rl},
  author={Tiapkin, Daniil and Morozov, Nikita and Naumov, Alexey and Vetrov, Dmitry P},
  booktitle={International Conference on Artificial Intelligence and Statistics},
  pages={4213--4221},
  year={2024},
  organization={PMLR}
}

@inproceedings{liu2023metric,
  title={Metric Residual Network for Sample Efficient Goal-Conditioned Reinforcement Learning},
  author={Liu, Bo and Feng, Yihao and Liu, Qiang and Stone, Peter},
  booktitle={Proceedings of the AAAI Conference on Artificial Intelligence},
  volume={37},
  number={7},
  pages={8799--8806},
  year={2023}
}

@inproceedings{wang2023optimal,
  title={Optimal goal-reaching reinforcement learning via quasimetric learning},
  author={Wang, Tongzhou and Torralba, Antonio and Isola, Phillip and Zhang, Amy},
  booktitle={International Conference on Machine Learning},
  pages={36411--36430},
  year={2023},
  organization={PMLR}
}

@article{kiran2021deep,
  title={Deep reinforcement learning for autonomous driving: A survey},
  author={Kiran, B Ravi and Sobh, Ibrahim and Talpaert, Victor and Mannion, Patrick and Al Sallab, Ahmad A and Yogamani, Senthil and P{\'e}rez, Patrick},
  journal={IEEE Transactions on Intelligent Transportation Systems},
  volume={23},
  number={6},
  pages={4909--4926},
  year={2021},
  publisher={IEEE}
}

@inproceedings{luck2019improved,
  title={Improved exploration through latent trajectory optimization in deep deterministic policy gradient},
  author={Luck, Kevin Sebastian and Vecerik, Mel and Stepputtis, Simon and Amor, Heni Ben and Scholz, Jonathan},
  booktitle={2019 IEEE/RSJ International Conference on Intelligent Robots and Systems (IROS)},
  pages={3704--3711},
  year={2019},
  organization={IEEE}
}

@article{ecoffet2021first,
  title={First return, then explore},
  author={Ecoffet, Adrien and Huizinga, Joost and Lehman, Joel and Stanley, Kenneth O and Clune, Jeff},
  journal={Nature},
  volume={590},
  number={7847},
  pages={580--586},
  year={2021},
  publisher={Nature Publishing Group}
}

@article{
kompella2023event,
title={Event Tables for Efficient Experience Replay},
author={Varun Raj Kompella and Thomas Walsh and Samuel Barrett and Peter R. Wurman and Peter Stone},
journal={Transactions on Machine Learning Research},
issn={2835-8856},
year={2023},
url={https://openreview.net/forum?id=XejzjAjKjv},
note={}
}

@inproceedings{memarianself,
  title={Self-supervised online reward shaping in sparse-reward environments. In 2021 IEEE},
  author={Memarian, Farzan and Goo, Wonjoon and Lioutikov, Rudolf and Niekum, Scott and Topcu, Ufuk},
  booktitle={RSJ International Conference on Intelligent Robots and Systems (IROS)},
  pages={2369--2375}
}

@article{baronchelli2013levy,
  title={L{\'e}vy flights in human behavior and cognition},
  author={Baronchelli, Andrea and Radicchi, Filippo},
  journal={Chaos, Solitons \& Fractals},
  volume={56},
  pages={101--105},
  year={2013},
  publisher={Elsevier}
}

@inproceedings{chakraborty2023dealing,
  title={Dealing with sparse rewards in continuous control robotics via heavy-tailed policy optimization},
  author={Chakraborty, Souradip and Bedi, Amrit Singh and Weerakoon, Kasun and Poddar, Prithvi and Koppel, Alec and Tokekar, Pratap and Manocha, Dinesh},
  booktitle={2023 IEEE International Conference on Robotics and Automation (ICRA)},
  pages={989--995},
  year={2023},
  organization={IEEE}
}

@article{wang2022learning,
  title={Learning of long-horizon sparse-reward robotic manipulator tasks with base controllers},
  author={Wang, Guangming and Xin, Minjian and Wu, Wenhua and Liu, Zhe and Wang, Hesheng},
  journal={IEEE Transactions on Neural Networks and Learning Systems},
  volume={35},
  number={3},
  pages={4072--4081},
  year={2022},
  publisher={IEEE}
}

@inproceedings{dawood2023handling,
  title={Handling sparse rewards in reinforcement learning using model predictive control},
  author={Dawood, Murad and Dengler, Nils and de Heuvel, Jorge and Bennewitz, Maren},
  booktitle={2023 IEEE International Conference on Robotics and Automation (ICRA)},
  pages={879--885},
  year={2023},
  organization={IEEE}
}

@article{wilcox2022monte,
  title={Monte carlo augmented actor-critic for sparse reward deep reinforcement learning from suboptimal demonstrations},
  author={Wilcox, Albert and Balakrishna, Ashwin and Dedieu, Jules and Benslimane, Wyame and Brown, Daniel and Goldberg, Ken},
  journal={Advances in neural information processing systems},
  volume={35},
  pages={2254--2267},
  year={2022}
}

@inproceedings{barth2018distributed,
  title={Distributed Distributional Deterministic Policy Gradients},
  author={Barth-Maron, Gabriel and Hoffman, Matthew W and Budden, David and Dabney, Will and Horgan, Dan and Dhruva, TB and Muldal, Alistair and Heess, Nicolas and Lillicrap, Timothy},
  booktitle={International Conference on Learning Representations},
  year={2018}
}
\bibliographystyle{tmlr}
\newpage
\appendix
\section{Appendix}

\counterwithin*{theorem}{section}
\counterwithin*{theorem}{subsection}
\counterwithin*{preposition}{section}
\counterwithin*{preposition}{subsection}

\subsection{$\epsilon t$-greedy Sample Efficiency : Proofs} \label{thm:PAC-RL}

In this section, we first provide an overview of the proof, presenting the key ideas at a high level. Then, we present the detailed formal proof of Theorem \ref{theorem : worstcase} and Theorem \ref{theorem : main}.

\paragraph{Proof Overview.} We aim to show that the $\epsilon t$-greedy algorithm falls into the PAC-MDP category. According to \citet{DBLP:journals/corr/abs-1805-09045}, an algorithm $\mathcal{A}$ is PAC-MDP if the covering time induced by $\mathcal{A}$ is polynomially bounded. In \citet{DBLP:journals/corr/abs-1805-09045}, the authors further demonstrate that bounding the covering time requires bounding both the Laplacian eigenvalues and the stationary distribution over the states induced by the random walk policy. This is presented as Proposition \ref{p1}. According to Theorem \ref{theorem : main}, two conditions are satisfied: $N \leq \Theta(|\mathcal{S}||\mathcal{A}|)$ and a lower bound on the probability of the sampled option, $\mathcal{P}_{\mathcal{X}} \geq \frac{1}{\Theta(|\mathcal{S}||\mathcal{A}|)}$. These two conditions are necessary and are met by our problem setting and the exploration algorithm (Algorithm \ref{alg:exp with replay buffer}). To prove that $\mathcal{P}_{\mathcal{X}} \geq \frac{1}{\Theta(|\mathcal{S}||\mathcal{A}|)}$, we construct a worst-case tree structure $\mathcal{X}$, where we aim to identify the option induced by the tree $\mathcal{X}$ with the lowest probability, referred to informally as the ``hardest option". We then show that this lower bound satisfies the condition specified in Theorem \ref{th1}.

We now proceed with the proof of Theorem \ref{th1}, as demonstrated below.

\begin{theorem}[\textbf{Worst-Case Sampling}] \label{th1}
    Given a tree $\mathcal{X}$ with $N$ nodes ($s_1$ to $s_N$), for any $\omega \in \Omega_{\mathcal{X}}$, the sampling probability satisfies:
    \begin{align}
        \mathcal{P}_{\mathcal{X}}[\omega] \ge \frac{1}{N!(\max_{i \in [N]}|\phi(s_i)|)^{N-1}} \ge \frac{1}{\Theta(|\mathcal{S}||\mathcal{A}|)}
    \end{align} where $N \le \frac{\log(|\mathcal{S}||\gA|)}{\log\log(|\mathcal{S}||\gA|)}$
Here, $\mathcal{S}$ and $\mathcal{A}$ represent the state space and action space, respectively.
\end{theorem}
\begin{proof}
As outlined in the proof overview, we need to construct an option with the lowest sampling probability. Given a tree $\mathcal{X}$, we define $\mathcal{X}_i$ (for $1 \leq i \leq N$) as the tree constructed up to the $i$-th time step. At each step $\mathcal{X}_i$, we track the tuple of added states, denoted by $\mathcal{S}_{i}^{\mathcal{X}}$, the uniformly sampled state $s_{x}$ from $\mathcal{S}_{i}^{\mathcal{X}}$, and the state with the fewest visits, $s_{min}$. The notation $s_x$ and $s_{min}$ follows Algorithm \ref{alg:exp with replay buffer}. Without loss of generality, we assume that each next state $s_{x'}$ in line 9 of Algorithm \ref{alg:exp with replay buffer} satisfies $n(\phi(s_{x'})) \neq 0$. Specifically, we consider a worst-case tree $\mathcal{X}$ fully populated with states from $s_1$ to $s_N$. Therefore, at time step $N$, $\mathcal{S}_{N}^{\mathcal{X}} = (s_1, s_2, \dots, s_N)$, and we have the following relation:

\begin{align}
n(\phi(s_1)) \geq n(\phi(s_2)) \geq n(\phi(s_3)) \dots \geq n(\phi(s_N)).
\label{eq:seq-visitation}
\end{align}

\Eqref{eq:seq-visitation} provides a decreasing sequence of visitations for newly added nodes in tree $\mathcal{X}$, emphasizing line 15 of Algorithm \ref{alg:exp with replay buffer}, which causes the state $s_{min}$ to change over $N$ iterations.
We assume a specific structure for each $\phi(s_i)$, where for all $i \in [N]$, at each bucket $\phi(s_i)$, there exists only one state denoted by $s_{i+1}$, such that $n(\phi(s_{i+1})) \leq n(\phi(s_i))$. Additionally, we assume that at each time step in $\mathcal{X}_t$, the newly added node connects only to the most recently added node in the tree. The two key stochastic events are summarized as follows:
\begin{itemize}
    \item $\mathcal{E}_1$: The event in which nodes are sampled in Line 24 from buckets satisfying the increasing sequence above.
    \item $\mathcal{E}_2$: The event in which nodes are selected in Line 8.
\end{itemize}
We now define the probability of interest, which we aim to bound:
\begin{align}
\mathcal{P}[\text{option returned from } s_{\text{root}} \text{ to } s_N | \mathcal{E}_1 \text{ and } \mathcal{E}_2].
\end{align}
We expand this probability as follows:
\begin{align*}
    \mathcal{P}[\text{option returned from } s_{\text{root}} \text{ to } s_N \mid \mathcal{E}_1 \text{ and } \mathcal{E}_2] 
    &= \prod_{i=2}^{N} \mathcal{P}[\text{(State } s_i \text{ added to tree } \mathcal{X}) \wedge (s_i = s_{min}) \wedge (s_x = s_{i-1} \text{ in Line 8})] \\
    &= \prod_{i=2}^{N} \frac{1}{(i-1)|\phi(s_{i-1})|} \\
    &= \frac{1}{(N-1)!} \times \frac{1}{|\phi(s_1)||\phi(s_2)|\dots|\phi(s_N)|} \\
    &> \frac{1}{N!} \times \frac{1}{(\max_{i \in [N]} |\phi(s_i)|)^{N-1}} \\
    &> \frac{1}{|\mathcal{S}||\mathcal{A}|}.
\end{align*}

To prove the final inequality, note that $N \leq \frac{\log(|\mathcal{S}||\mathcal{A}|)}{\log \log(|\mathcal{S}||\mathcal{A}|)}$. Since the size of the sets $\mathcal{S}$ and $\mathcal{A}$ is large and $N$ is sub-logarithmic in $|\mathcal{S}||\mathcal{A}|$, i.e., $N \ll \log(|\mathcal{S}||\mathcal{A}|)$, we can say $N \leq \frac{\log(|\mathcal{S}||\mathcal{A}|)}{\log(N)}$. Let us denote $\log(\max_{i \in [N]}|\phi(s_i)|)$ as a constant $c_0$.

Now by the series of following inequalities we prove that $\frac{1}{N!}\times \frac{1}{(\max_{i \in [N]}|\phi(s_i)|)^{N-1}} > \frac{1}{|\mathcal{S}||\mathcal{A}|}$.
\begin{align}
    N &\le \frac{\log(|\gS||\gA|)}{\log(N)} 
    \Rightarrow N\log(N) \le \log(|\gS||\gA|) \\
    &\Rightarrow N\log(N) + (N-1)c_0 - N \le \log(|\gS||\gA|) 
    \quad\quad\quad  \text{(since $|\gS| |\gA| \gg N, c_0$)} \\ 
    &\Rightarrow \log(N!) + (N-1)c_0 \le \log(|\gS||\gA|) 
    \quad\quad\quad  \text{(Based on the Moivre–Stirling approximation)} \\ 
    &\Rightarrow \log(N!) + (N-1)c_0 \le \log(|\gS||\gA|) \\
    &\Rightarrow \log(N!) + \log \left( (\max_{i \in [N]} |\phi(s_i)|)^{N-1} \right) 
    \le \log(|\gS||\gA|) \\ 
    &\Rightarrow \log\left(N! \cdot (\max_{i \in [N]} |\phi(s_i)|)^{N-1}\right) \le \log(|\gS||\gA|) \\
    &\Rightarrow \frac{1}{N! \cdot (\max_{i \in [N]} |\phi(s_i)|)^{N-1}} 
    \ge \frac{1}{|\gS||\gA|}
\end{align}

\end{proof}
Now we provide the main proof which demonstrates polynomial sample complexity under certain criteria.
\begin{theorem}[\textbf{\bm{$\epsilon t$}-greedy Sample Efficiency}] \label{theorem : main}
    Given a state space $\mathcal{S}$, action space $\mathcal{A}$, and a set of options $\Omega_{\mathcal{X}}$ generated by $\epsilon t$-greedy for each tree $\mathcal{X}$, if $\mathcal{P}_{\mathcal{X}}[\omega] \ge \frac{1}{\Theta(|\mathcal{S}||\mathcal{A}|)}$, $\epsilon t$-greedy achieves polynomial sample complexity or i.e. is PAC-MDP.
\end{theorem}

\begin{proof}
    First note that if $\mathcal{P}_{\mathcal{X}}[\omega] \ge \frac{1}{\Theta(|\mathcal{S}||\mathcal{A}|)}$ then based on Theorem \ref{th1} we need to have $N \le \frac{\log(|\mathcal{S}||\gA|)}{\log\log(|\mathcal{S}||\gA|)}$, and this implies that $N \le \Theta(|\gS||\gA|)$. 
    Based on the paper by \citep{DBLP:journals/corr/abs-1805-09045}, and the analysis of the covering length when following a random policy, we have the following preposition:
    \begin{preposition}
        [\textbf{\cite{DBLP:journals/corr/abs-1805-09045}}] : For any irreducable MDP M, we define $P_{\pi_{RW}}$ as a transition matrix induced by random walk policy $\pi_{RW}$ over $M$ and $L(P_{\pi_{RW}})$ is denoted as the Laplacian of this transition matrix. Suppose $\lambda$ is the smallest non-zero eigenvalue of $L$ and $\Psi(s)$ is the stationary distribution over states which is induced by random walk policy, then Q-learning with random walk exploration is a PAC RL algorithm if:  $\frac{1}{\lambda}$,$\frac{1}{\min_{s}\Psi(s)}$ are  Poly($|\gS||\gA|$). 
    \label{p1}
    \end{preposition}Note that Preposition \ref{p1} is not limited to an MDP with primitive actions. Therefore, we can broaden its scope by incorporating options into this proposition and demonstrate that both $\frac{1}{\lambda}$ and $\frac{1}{\min_{s}\Psi(s)}$ can be polynomially bounded in terms of MDP parameters—in this case, states and actions in our approach.

    Let's begin by examining the upper-bound for $\frac{1}{\min_{s}\Psi(s)}$. Suppose we are at exploration tree $\gX$. Without a loss of generality, we consider that capacity of tree $\gX$ is full, and we have $N$ states. In this tree, let's designate $s_{root}$ as the state assigned as the root of the tree during the exploration phase. Now, consider another random state (excluding $s_{root}$) within this tree structure, denoted as $s_{rand}$. We acknowledge that, when considering the entire state space, there can be multiple options constructed from $s_{root}$ to $s_{rand}$. Each tree $\gX$ provides one of these options. $\Psi(s)$ is defined over all states, and $\omega$ is the option with a limited size because of the constrained tree budget.
    
    we can calculate the upper-bound for $\frac{1}{\min_{s}\Psi(s)}$ as follows:

    \begin{equation}
    \begin{aligned}
                \Psi(s_{rand}) = \sum_{\omega \in \Omega_{\gX}} \gP_{\gX}[\omega]\Psi(s_{root}) \Rightarrow \Psi(s_{rand}) \ge \gP[\omega] \Psi(s_{root}) ,\\ \frac{1}{\Psi(s_{rand})} \le \frac{1}{\gP[\omega]}\frac{1}{\Psi(s_{root})} \Rightarrow  \frac{1}{\Psi(s_{rand})} \le \frac{\Theta(|\gS||\gA|)}{\Psi(s_{root})}
    \end{aligned}
    \end{equation}
Since $s_{rand}$ can represent any of the states encountered in the tree, we can regard it as the state assigned the least probability in the stationary distribution. Therefore, we have:
    \begin{equation}
        \frac{1}{\Psi(s_{rand})} \le \frac{\Theta(|\gS||\gA|)}{\Psi(s_{root})} \Rightarrow \frac{1}{\min_{s} \Psi(s)} \le \frac{\Theta(|\gS||\gA|)}{\Psi(s_{root})}
    \end{equation}

    So, $\frac{1}{\min_{s}\Psi(s)}$ is polynomially bounded. Now, we need to demonstrate that $\frac{1}{\lambda}$ is also polynomially bounded. To bound $\lambda$, we first need to recall the definition of the Cheeger constant, $h$. Drawing from graph theory, if we denote $V(G)$ and $E(G)$ as the set of vertices and edges of an undirected graph $G$, respectively, and considering the subset of vertices denoted by $V_{s}$, we can define $\sigma V_{s}$ as follows:
    
    \begin{equation}
    \sigma V_{s} := \{(n_1,n_2) \in E(G): n_1 \in V_{s} , n_2 \in V(G)\setminus V_{s}\}
    \end{equation}

So, $\sigma V_{s}$ can be regarded as a collection of all edges going from $V_s$ to the vertex set outside of $V_s$. In the above definition, $(n_1, n_2)$ is considered as a graph edge. Now, we can define a Cheeger constant:

\begin{equation}
    h(G) := \min \{\frac{|\sigma V_{s}|}{|V_{s}|}: V_s \subseteq V(G), 0 < V_{s} \le \frac{1}{2}|V(G)|\}
\end{equation}

We are aware that $h \ge \lambda \ge \frac{h^2}{2}$, and by polynomially bounding $h$, we can ensure that $\lambda$ is also bounded. In a related work \citep{DBLP:journals/corr/abs-1805-09045}, an alternative variation of the Cheeger constant is utilized, which is based on the flow $F$ induced by the stationary distribution $\Psi$ of a random walk on the graph. Suppose for nodes $n_1, n_2$ and subset of nodes $N_1$ in the graph, we have:

\begin{align}
    &F(n_1,n_2)= \Psi(n_1)P(n_1,n_2) , \\ 
    &F(\sigma N_1)=\hspace{-15pt}\sum_{n_1 \in N_1 , n_2 \notin N_1}\hspace{-15pt}F(n_1,n_2), \\
    &F(N_1)=\sum_{n_1 \in N_1} \Psi(n_1)
\end{align}

Building upon the aforementioned definition, the Cheeger constant is defined as:

\begin{equation}
h :=  \inf_{N_1} \frac{F(\sigma N_1)}{\min\{F(N_1),F(\bar{N_1})\}}    
\end{equation}

Suppose $N_{rand} = \{s_{root}\}$; we will now demonstrate that $\frac{1}{h}$ can be polynomially bounded :

\begin{equation*}
\begin{aligned}
    h =  \inf_{N_1} \frac{F(\sigma N_1)}{\min\{F(N_1),F(\bar{N_1})\}} \ge  \frac{F(\sigma N_{rand})}{\min\{F(N_{rand}),F(\overline{N_{rand}})\}} \ge \frac{\sum_{s \neq  s_{root} } \Psi(s_{root})P_{\pi_{RW}}(s_{root},s)}{\Psi(s_{root})}, \\
= \sum_{s \neq S_{root}}P_{\pi_{RW}}(s_{root},s) \ge \gP_{\gX}[\omega] \Rightarrow \frac{1}{h} \le \Theta(|\gS||\gA|)
\end{aligned}
\end{equation*}

We demonstrate that both terms stated in Preposition \ref{p1} are polynomially bounded, and thus, the proof is complete.
\end{proof}

\begin{algorithm}[h] \small
\caption{Generating exploratory option with tree search using a perfect model} 
\label{alg:exp with perfect model}   
\begin{algorithmic}[1] 
\Function{}{}\textbf{generate\_option}(state s, hash function $\phi$, budget N)
\State frontier\_nodes $\gets \{ \}$
\State Initialize root using $s$: $\mathit{root} \gets \mathit{TreeNode}(s)$
\State frontier\_nodes $\gets$ frontier\_nodes $\cup $ \{root\};
\State $s_{\min}$ $\gets$ root
\State $i \gets 0$
\While{$ i < N$}
    \State $s_{x}$ $\sim$ \textit{UniformRandom}(frontier\_nodes)
    \State $s_{x^{\prime}}$= \textbf{next\_state\_from\_env}($s_{x}$)
    \If {$\phi$($n(s_{x^{\prime}})$)=0}
        \State Extract option $o$ by actions $root$ to $s_{x^\prime}$
        \State \Return $o$
    \EndIf
    \If {$n(\phi$($s_{x^{\prime}})$) $<$ $n(\phi$($s_{\min})$)}
        \State $s_{\min}$=$s_{x^{\prime}}$
    \EndIf
    \State $i \gets i + 1$
\EndWhile
\State Extract option $o$ by actions $root$ to $s_{\min}$
\State \Return $o$
\EndFunction
\State
\Function{}{}\textbf{next\_state\_from\_env}($s_{x}$, frontier\_nodes)
    \State a $\sim$ \textit{UniformRandom}($\mathcal{A}(s_x)$)
    \State $s_{x^{\prime}}$ $\leftarrow$ $\mathcal{T}$($s_{x}$, a)
    \State $s_{x}$.add\_child ($s_{x^{\prime}}$)
    \State frontier\_nodes $\gets$ frontier\_nodes $\cup$ \{$s_{x^{\prime}}$\}
    \State \Return $s_{x^{\prime}}$
\EndFunction
\end{algorithmic}
\end{algorithm}

\begin{algorithm}[t]
\caption{ETGL-DDPG}
\label{alg:ETGL algorithm}  
\begin{algorithmic}[] 
\State Randomly initialize critic network $Q(s,a,g|\theta^{Q})$ and actor $\mu(s,g|\theta^{\mu})$ with weights $\theta^{Q}$ and $\theta^{\mu}$
\State Initialize target networks $Q^{\prime}$ and $\mu^{\prime}$ with weights $\theta^{Q^{\prime}} \leftarrow \theta^{Q}$, $\theta^{\mu^{\prime}} \leftarrow \theta^{\mu}$
\State Initialize replay buffers $D_{\beta}$, $D_{e}$, hash function $\phi$, exploration budget $N$
\State
\Function{}{}\textbf{train}($Q$, $\mu$, $\phi$)
    \For{episodes=1,M}
        \State Receive initial observation state $s_{1}$ and goal $g$
        \State \textbf{run\_episode}($s_{1}$, $g$)
        \State \textbf{update}($success$)
    \EndFor
\EndFunction
\State
\Function{}{}\textbf{run\_episode}($s$, $g$)
    \State $success \leftarrow false$, $l \leftarrow 0$
    \While{t $\leq$ T \textbf{and} \textbf{not}(\textit{success})}
        \If{$l$==0}
            \If{random()$< \epsilon$}
                \State Exploratory option  $w \leftarrow$ \textbf{generate\_option}($s$, $\phi$, $N$)
                \State Assign action : $a_{t} \leftarrow w$
                \State $l\leftarrow$ length($w$)
            \Else
                \State Greedy action : $a_{t} \leftarrow \mu(s_{t},g|\theta^{\mu})$
            \EndIf
        \Else
            \State Assign action : $a_{t} \leftarrow w$
            \State $l \leftarrow l-1$
        \EndIf
        \State Execute action $a_{t}$ and observe reward $r_{t}$ and next state $s_{t+1}$
        \If{is\_goal($s_{t+1}$)}
            \State $success \leftarrow true$
        \EndIf
    \EndWhile
\EndFunction
\State
\Function{}{}\textbf{update}($success$)
    \State
    \State
    $R=\left\{\begin{matrix}
     r_{t}  & success \\ 
     0   & o.w
    \end{matrix}\right.$
    \vspace{2mm}
    \ \ \ $bootstrap=\left\{\begin{matrix}
     0  & success \\ 
     1   & o.w
    \end{matrix}\right.$
    \For{$i \in \{t-1,...,t_{start}\}$}
        \State $R \leftarrow r_{i}+\gamma R$
        \If{$success$}
            \State store transition ($s_{i},g,a_{i},R,s_{t},bootstrap$) in $D_{\beta}$, $D_{e}$
        \Else
            \State store transition ($s_{i},g,a_{i},R,s_{t},bootstrap$) in $D_{\beta}$
        \EndIf
    \EndFor
    \State Sample C random mini-batches of k transitions ($s_{j},g_{j},a_{j},r_{j},s_{j+1},bootstrap_{j}$) by $\tau_{\beta}$ and $\tau_{e}$ ratios from $D_{\beta}$ and $D_{e}$
    \State set $y_{j}=r_{j}+bootstrap_{j}*\gamma Q^{\prime}(s_{j+1},g_{j},\mu^{\prime}(s_{j+1},g_{j}|\theta^{\mu^{\prime}})|\theta^{Q^{\prime}})$
    \State update critic by minimizing the loss: $L=\frac{1}{k}\sum_{j}(y_{j}-Q(s_{j},g_{j},a_{j}|\theta^{Q}))$
    \State update the actor: $\nabla_{\theta^{\mu}} J \approx \frac{1}{k} \sum_{j}\nabla_{a}Q(s,g,a|\theta^{Q})|_{s=s_{j},g=g_{j},a=\mu(s_{j},g_{j})}\nabla_{\theta^{\mu}}\mu(s,g|\theta^{\mu})|_{s_{j}}$
    \State update the target networks: $\theta^{Q^{\prime}} \leftarrow \tau \theta^{Q}+(1-\tau)\theta^{Q^{\prime}},\ \theta^{\mu{\prime}} \leftarrow \tau\theta^{\mu}+(1-\tau)\theta^{\mu^{\prime}}$
\EndFunction
\end{algorithmic}
\end{algorithm}

\subsection{n-step Methods Comparison} \label{n-step section}

In this section, we compare two update rules for Q-values in DDPG: the longest n-step return and the average return over 1 to 8 steps (avg8-step). The results are shown in Figure \ref{fig: n-step comparison}. Both methods performed similarly overall; however, in the Wall-maze environment, the longest n-step return outperformed the avg8-step approach. The avg8-step method converged more quickly in U-maze and Point-push due to the stabilizing effect of averaging multiple updates.

\subsection{Comparison of GDRB and Hindsight Experience Replay (HER)} \label{sec:her}

In this section, we compare two strategies for storing transitions in the replay buffer: GDRB and HER. Figure~\ref{fig: her-gdrb comparison} shows the results in the navigation environments. In Wall-maze, HER achieves success rates above zero during the early stages of training; however, it fails to maintain any success in later stages. This observation is consistent with the results reported by \citet{trott2019keeping} in the same environment. In the other two environments, both methods exhibit comparable performance, but HER discovers paths to the goal more quickly, as it leverages reward shaping in unsuccessful episodes to provide additional information about the structure of the environment.

\begin{figure*}[t]
    \centering
        \begin{subfigure}[b]{150pt}
        \centering
        \includegraphics[width=\textwidth]{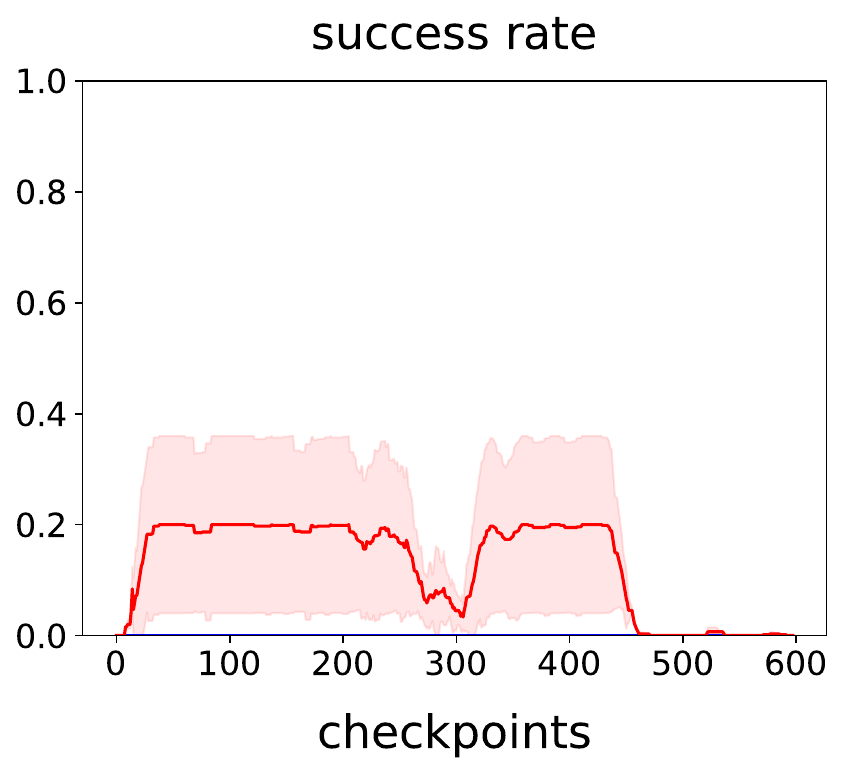}
        \caption{Wall-maze}
        \end{subfigure}
        \begin{subfigure}[b]{150pt}
        \centering
        \includegraphics[width=\textwidth]{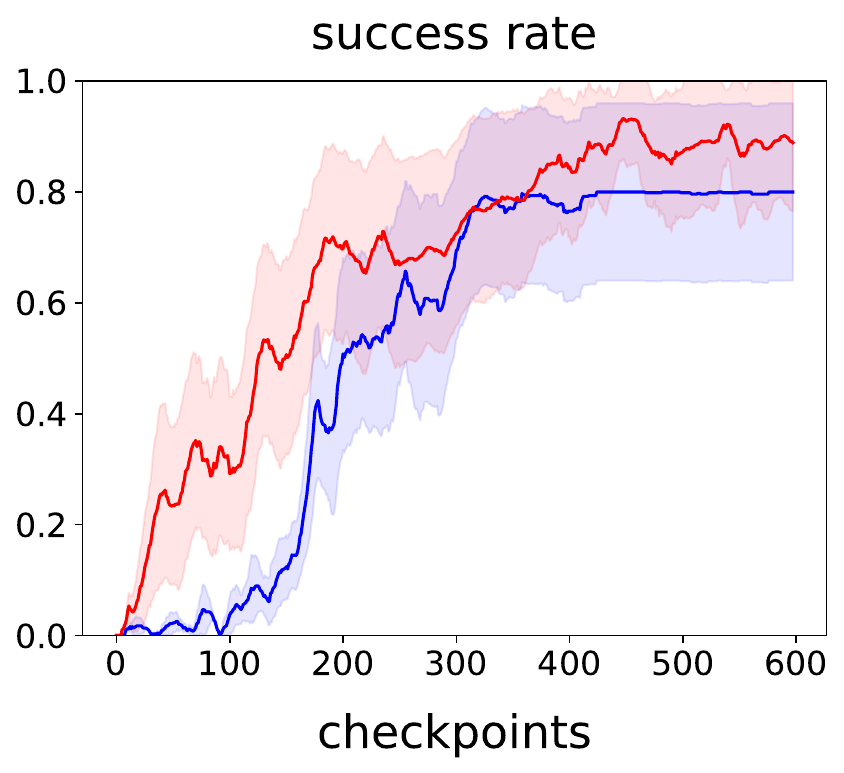}
        \caption{U-maze}
        \end{subfigure}
        \begin{subfigure}[b]{150pt}
        \centering
        \includegraphics[width=\textwidth]{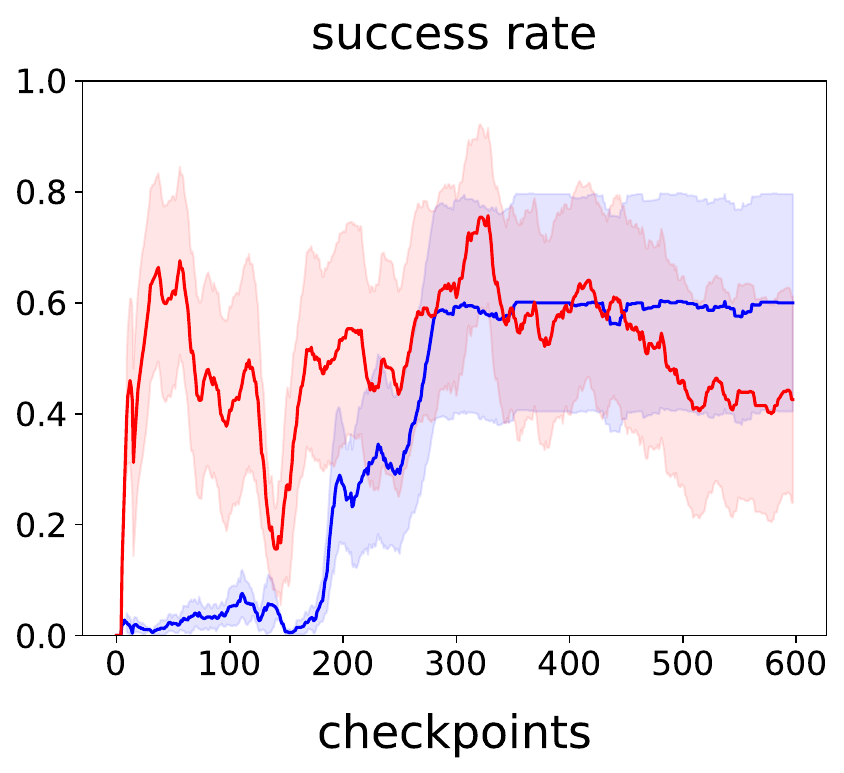}
        \caption{Point-push}
        \end{subfigure}
        \hspace{-20pt} \includegraphics[width=0.4\textwidth]{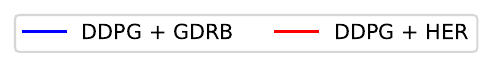}
    \caption{ The comparison between two replay buffer strategies: GDRB and HER.}
    \label{fig: her-gdrb comparison}
\end{figure*}

\begin{figure*}[t]
    \centering
        \begin{subfigure}[b]{150pt}
        \centering
        \includegraphics[width=\textwidth]{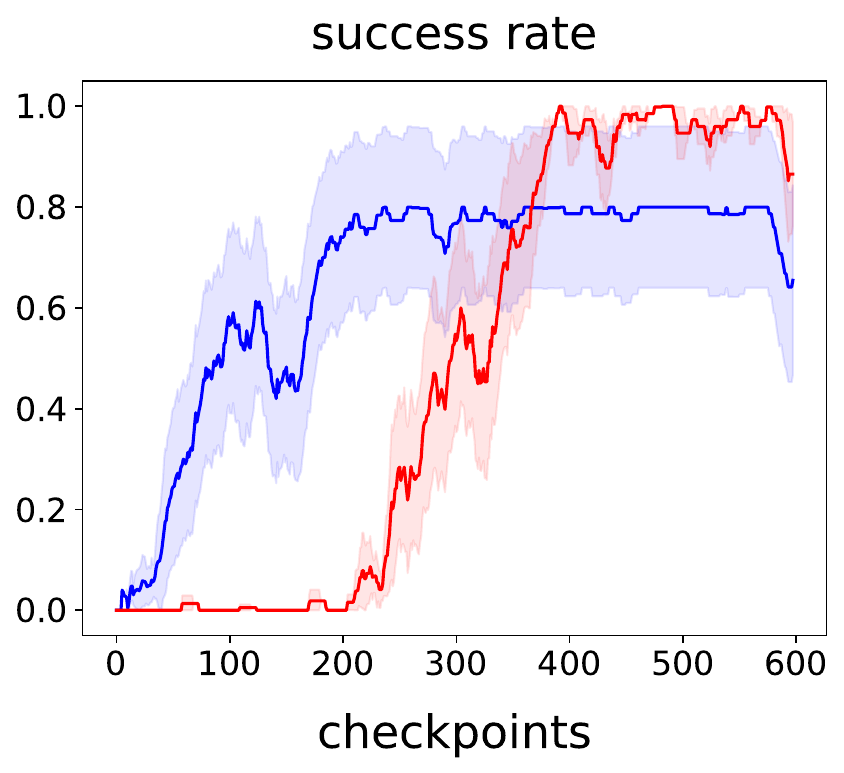}
        \caption{Wall-maze}
        \end{subfigure}
        \begin{subfigure}[b]{150pt}
        \centering
        \includegraphics[width=\textwidth]{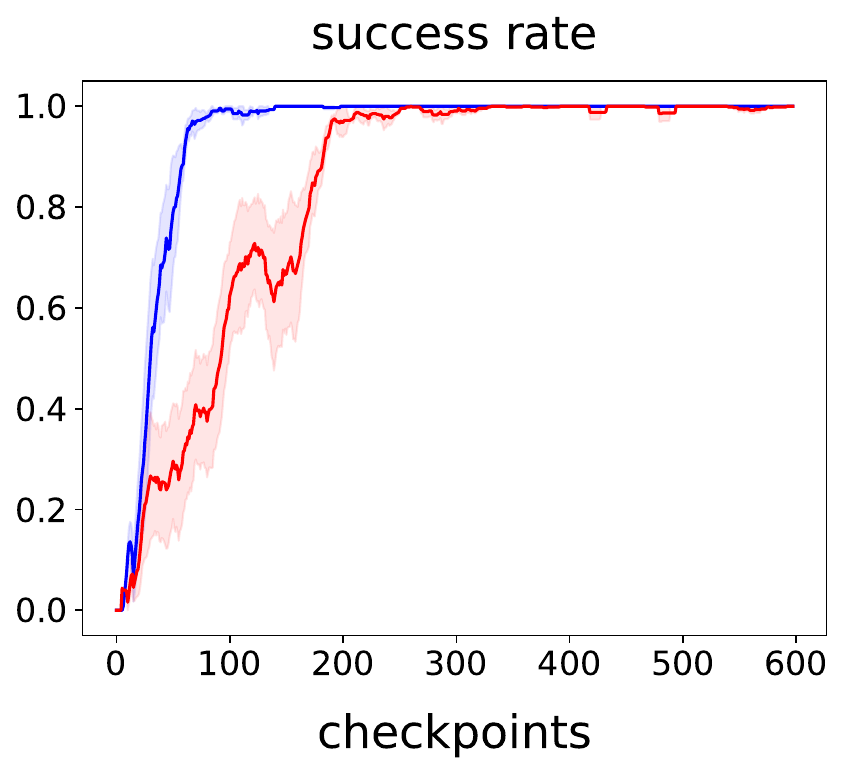}
        \caption{U-maze}
        \end{subfigure}
        \begin{subfigure}[b]{150pt}
        \centering
        \includegraphics[width=\textwidth]{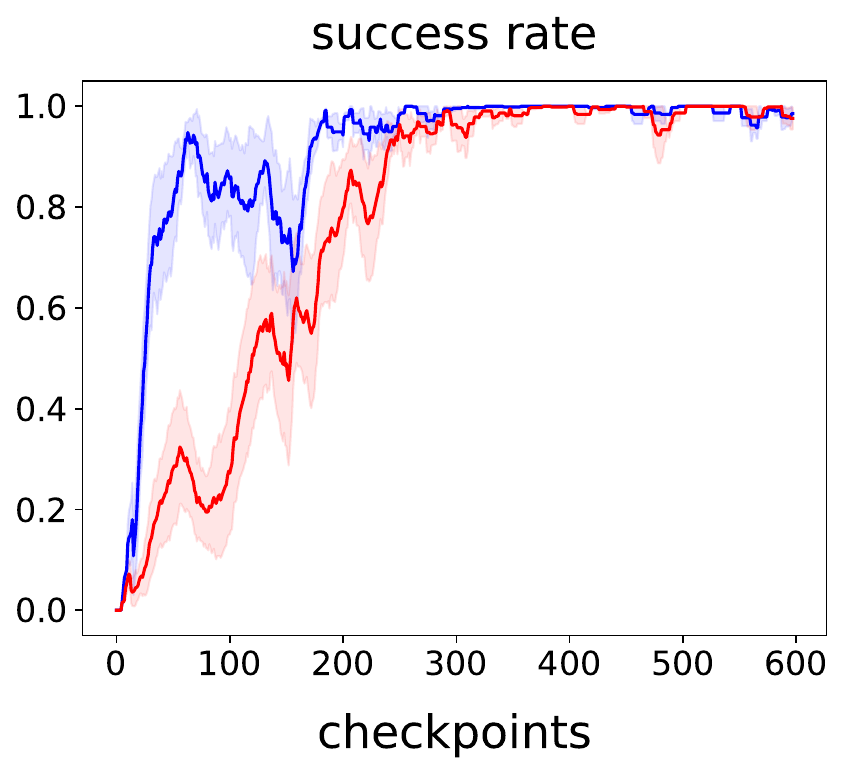}
        \caption{Point-push}
        \end{subfigure}
        \hspace{-20pt} \includegraphics[width=0.65\textwidth]{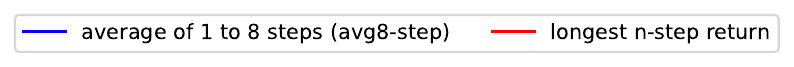}
    \caption{The comparison between two update rules for Q-values: the longest n-step return and the average return over 1 to 8 steps (avg8-step).}
    \label{fig: n-step comparison}
\end{figure*}

\subsection{Exploration with a Perfect Model} \label{sec: exp with perfect model}

Since the DDPG algorithm is model-free, we utilize the replay buffer to construct the tree for $\epsilon t$-greedy. However, $\epsilon t$-greedy can also take advantage of a perfect model when available. The pseudocode for option generation using a perfect model is provided in Algorithm \ref{alg:exp with perfect model}. The key difference from Algorithm \ref{alg:exp with replay buffer} is the use of the \texttt{next\_state\_from\_env} function instead of \texttt{next\_state\_from\_replay\_buffer} to generate child nodes. In this case, an action is uniformly sampled from the action space, and the environment's transition function $\mathcal{T}$ is directly used to determine the next state (line 25). Figure \ref{fig: model-buffer} compares the performance of ETGL-DDPG in navigation environments using a perfect model versus a replay buffer. The results show a clear advantage when using a perfect model, as the agent reaches a success rate of 1 more quickly and with less deviation. 

\begin{figure*}[h]
    \centering
        \begin{subfigure}[b]{150pt}
        \centering
        \includegraphics[width=\textwidth,height=3.5cm]{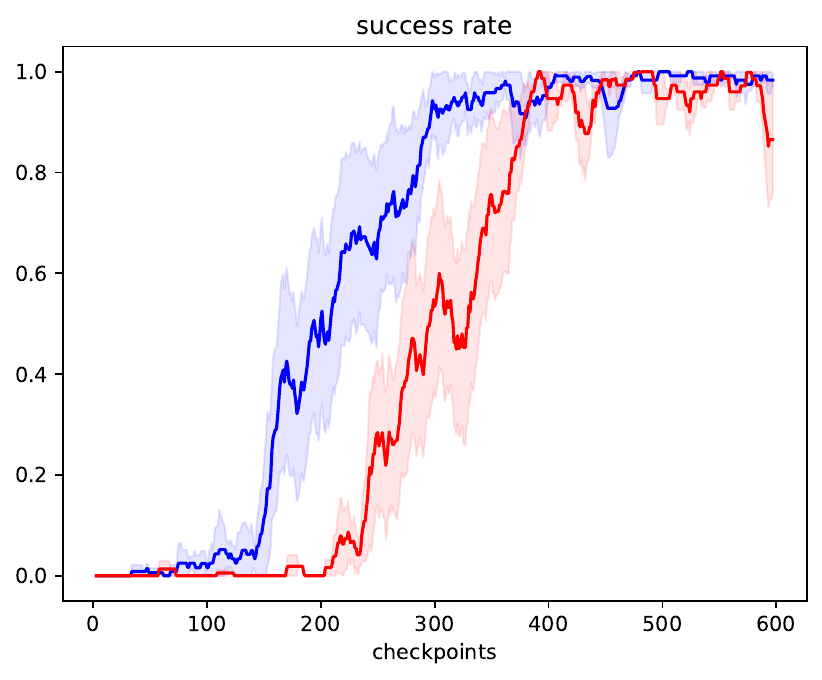}
        \caption{Wall-maze}
        \end{subfigure}
        \begin{subfigure}[b]{150pt}
        \centering
        \includegraphics[width=\textwidth,height=3.5cm]{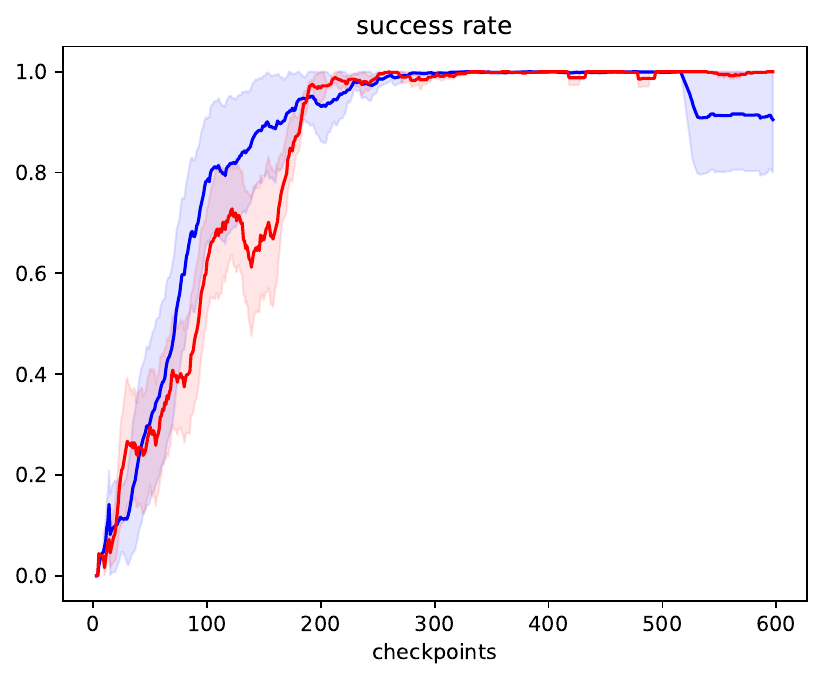}
        \caption{U-maze}
        \end{subfigure}
        \begin{subfigure}[b]{150pt}
        \centering
        \includegraphics[width=\textwidth,height=3.5cm]{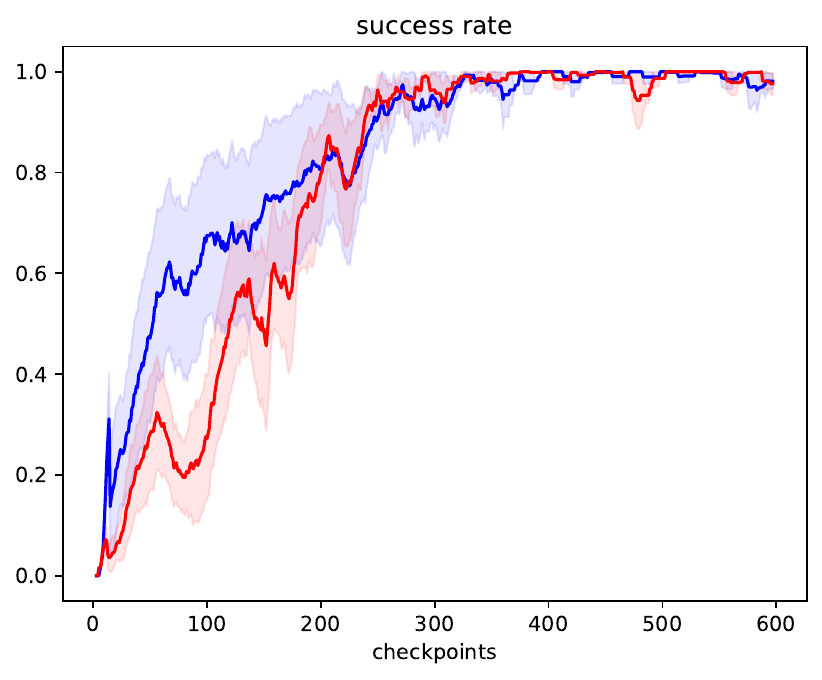}
        \caption{Point-push}
        \end{subfigure}
        \hspace{-20pt} \includegraphics[width=0.35\textwidth]{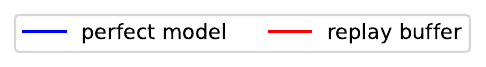}
    \caption{Comparison of ETGL-DDPG performance in navigation environments using a perfect model vs. replay buffer.}
    \label{fig: model-buffer}
\end{figure*}

\subsection{Implementation Details and Experimental Hyperparameters} \label{sec: exp setup}

Here, we describe the implementation details and hyperparameters for all methods used in this paper. All experiments were run on a system with 5 vCPU on a cluster of Intel Xeon E5-2650 v4 2.2GHz CPUs and one 2080Ti GPU. Table~\ref{details for environments} displays the details for environments. Tables \ref{details for ETGL-DDPG}, \ref{SAC details}, and \ref{DOIE details} showcase the hyperparameters utilized in ETGL-DDPG and the baselines.

\vspace{50pt}
\begin{table*}[h!]
\begin{center}
\caption{Implementation details for ETGL-DDPG.}
\label{details for ETGL-DDPG}
\begin{tabular}{c|c|c|c|c|c|c}
    \textbf{Hyperparameter}  &  {wall-maze} & {U-maze} & {Point-push} & {window-open} & {soccer} & {button-press} \\
    \hline
     batch size & \multicolumn{3}{c|}{128} & \multicolumn{3}{c}{512}      \\
    \hline
     number of updates per episode &  \multicolumn{3}{c|}{20} & \multicolumn{3}{c}{200}    \\
     \hline
     epsilon decay rate &  \multicolumn{3}{c|}{0.9999988} & \multicolumn{3}{c}{0.9999992}  \\
     \hline
     exploration budget \textit{N} & 20 & \multicolumn{2}{|c|}{40} & \multicolumn{3}{|c}{60}  \\
     \hline
     SimHash dimension  &  \multicolumn{3}{c|}{$k=9$} & \multicolumn{3}{c}{$k=16$}  \\
     \hline
     soft target updates $\tau$ &   \multicolumn{6}{c} {$10^{-2}$}   \\
      \hline
      discount factor $\gamma$ & \multicolumn{6}{c}{0.99}  \\
     \hline 
     warmup period  & \multicolumn{6}{c}{$2*10^{5}$ steps}\\
     \hline 
     exploration buffer size & \multicolumn{6}{c}{$10^{6}$}    \\
     \hline 
     exploitation buffer size & \multicolumn{6}{c}{$5 * 10^{4}$}\\
     \hline
     actor learning rate & \multicolumn{6}{c}{$10^{-4}$}\\
     \hline
     critic learning rate & \multicolumn{6}{c}{$10^{-3}$}\\
\end{tabular}
\end{center}
\end{table*}

\vspace{50pt}
\hspace{100mm}
\begin{table}[htbp]
\centering
\caption{Environment details.}
\label{details for environments}
\begin{tabular}{l|c|c|c|c}
    \textbf{environment}  &  {$S \in$} & {$G \in$}  & {$A \in$} & Max steps per episode   \\
    \hline
    Wall-maze  & $\mathbb{R}^{2}$ &  $\mathbb{R}^{2}$ & $[-0.95,0.95]^{2}$ & 100    \\
    \hline
    U-maze  &  $\mathbb{R}^{6}$ &  $\mathbb{R}^{2}$ &  $[-1,1]*[-0.25,0.25]$ & 500   \\
    \hline
    Point-push  &  $\mathbb{R}^{11}$ &  $\mathbb{R}^{2}$ &  $[-1,1]*[-0.25,0.25]$ & 500\\
    \hline
    window-open  &  $\mathbb{R}^{39}$ &  $\mathbb{R}^{3}$ &  $[-1,1]^{4}$ & 500\\
    \hline
    soccer  &  $\mathbb{R}^{39}$ &  $\mathbb{R}^{3}$ &  $[-1,1]^{4}$ & 500\\
    \hline
    button-press  &  $\mathbb{R}^{39}$ &  $\mathbb{R}^{3}$ &  $[-1,1]^{4}$ & 500\\
\end{tabular}
\end{table}

\vspace{80pt}
\begin{table}[h]
\begin{center}
\centering
\caption{Implementation details for DOIE.}
\vspace{2mm}
\label{DOIE details}
\begin{tabular}{c|c|c|c|c|c|c}
    \textbf{Hyperparameter}  &  {wall-maze} & {U-maze} & {Point-push} & {window-open} & {soccer} & {button-press} \\
    \hline
     batch size & \multicolumn{6}{c}{256}      \\
    \hline
     number of updates per episode &  \multicolumn{6}{c}{100}    \\
     \hline 
     replay buffer size & \multicolumn{6}{c}{$5*10^{5}$}\\
     \hline 
     actor learning rate & \multicolumn{6}{c}{$10^{-4}$}\\
     \hline
     critic learning rate & \multicolumn{6}{c}{$5*10^{-3}$}\\
     \hline
      discount factor $\gamma$ & \multicolumn{6}{c}{0.99}  \\
      \hline
      action scaling & \multicolumn{6}{c}{0.01}\\
      \hline
      environment scaling & \multicolumn{6}{c}{0.1 for each dimension}\\
      \hline
      knownness mapping type & \multicolumn{6}{c}{polynomial} \\
    
\end{tabular}
\end{center}
\end{table}

\vspace{50pt}
\begin{table}[h]
\centering
\caption{Implementation details for SAC, TD3, and DDPG.}
\label{SAC details}
\begin{tabular}{c|c|c|c|c|c|c}
    \textbf{Hyperparameter}  &  {wall-maze} & {U-maze} & {Point-push} & {window-open} & {soccer} & {button-press} \\
    \hline
     batch size & \multicolumn{3}{c|}{128} & \multicolumn{3}{c}{512}      \\
    \hline
     update frequency per step &  \multicolumn{3}{c|}{12} & \multicolumn{3}{c}{2}    \\
     \hline 
     action noise &  $\sim N(0,0.2)$ & \multicolumn{2}{c|}{$\sim N(0,(0.3, 0.05))$} & \multicolumn{3}{c}{$\sim N(0,(0.15))$}  \\
     \hline
     warmup period  & \multicolumn{6}{c}{$2*10^{5}$ steps}\\
     \hline 
     replay buffer size & \multicolumn{6}{c}{$10^{6}$}\\
     \hline
     learning rate & \multicolumn{6}{c}{$3*10^{-4}$}\\
     \hline
     soft target updates $\tau$ &   \multicolumn{6}{c} {$5*10^{-3}$}   \\
      \hline
      discount factor $\gamma$ & \multicolumn{6}{c}{0.99}  \\
\end{tabular}
\end{table}

\vspace{10pt}

\subsection{ETGL-DDPG Algorithm}

In this section, we introduce ETGL-DDPG, as detailed in Algorithm \ref{alg:ETGL algorithm}, which is organized into three primary functions: \texttt{train}, \texttt{run\_episode}, and \texttt{update}. The \texttt{train} function is called once at the start of the training process. For each training episode, the \texttt{run\_episode} function is invoked to perform a training episode within the environment, followed by the \texttt{update} function to adjust the networks based on the experience gained from the episode.

\subsection{Analyzing the Impact of $N$ on $\epsilon t$-greedy and $\epsilon z$-greedy Exploration} \label{sec: impact of N}

The tree budget $N$ upper bounds the option length of $\epsilon t$-greedy due to the fact that the longest path between nodes in the tree is shorter or equal to the number of nodes in the tree. This is analogous to the role of $N$ in $\epsilon z$-greedy, where a uniform distribution $z(n)=\mathbbm{1}_{n \leq N}/N$ is used. To evaluate both methods, we assess environment coverage under varying budget sizes, calculating the coverage after 1 million training frames. Table \ref{table: budget} shows the results: $\epsilon t$-greedy consistently achieves greater coverage than $\epsilon z$-greedy across all environments and budget sizes. Additionally, $\epsilon t$-greedy demonstrates improved coverage as the budget increases. In contrast, increasing the budget for $\epsilon z$-greedy does not consistently improve coverage and can even decrease it in some cases. This highlights the advantages of directed exploration over undirected methods, particularly in complex environments with numerous obstacles, such as Wall-maze.

\begin{table*}[t] \small
\centering
\caption{Analysis of the impact of budget N on the environment coverage.}
\label{table: budget}
\begin{tabular}{c|c|c|c|c|c|c}
\toprule
\headercell{\\budget $N$}  &\multicolumn{3}{c|}  {$\epsilon z$-greedy} &\multicolumn{3}{c} {$\epsilon t$-greedy}\\
\cmidrule(l){2-7}
\cmidrule(l){2-7}
& Wall-maze &  U-maze & Point-push  & Wall-maze & U-maze & Point-push     \\ 
\midrule
  5  & 0.36           &  0.55          & 0.36           & 0.76       & 0.94          & 0.40           \\
  10 & \textbf{0.38}  &  0.91          & 0.38           & 0.97       & 0.91          & 0.41            \\
  15 & 0.34           &  0.85          & 0.39           & 0.65       & 0.94          & 0.42            \\
  20 & 0.30           &  0.84          & 0.40           & 0.83       & 0.94          & 0.48            \\
  25 & 0.28           &  \textbf{0.86} & 0.40           & \textbf{1} & 0.95          & 0.47            \\
  30 & 0.27           &  0.83          & 0.39           & 1          & \textbf{0.97} & 0.51            \\
  35 & 0.25           &  0.82          & 0.40           & 1          & 0.95          & 0.53           \\
  40 & 0.24           &  0.82          & 0.40           & 1          & 0.97          & 0.55            \\
  45 & 0.22           &  0.85          & \textbf{0.41}  & 1          & 0.96          & 0.64           \\
  50 & 0.22           &  0.79          & 0.40           & 1          & 0.97          & \textbf{0.73}  \\ \bottomrule
\end{tabular}
\end{table*}

\subsection{Terminal States Distribution} \label{sec: terminal dist}
We analyze the order in which the agent visits different parts of the environment by examining the distribution of the last states in the episodes. To make it more visually appealing and easy to interpret, we only sample some of the episodes. The results for Wall-maze, U-maze, and Point-push can be found in Figures \ref{fig: loc wall-maze},  \ref{fig: loc u-maze}, and \ref{fig: loc push}, respectively. In Wall-maze, only $\epsilon t$-greedy and DOIE can effectively navigate to different regions of the environment and ultimately reach the goal area. Other methods often get trapped in one of the local optima and are unable to reach the goal. The reason some methods, such as TD3, have fewer points is that the agent spends a lot of time revisiting congested areas instead of exploring new ones. In U-maze, most methods can explore the majority of the environment. However, during the final stages of training, methods such as DDPG, SAC, and DDPG + intrinsic motivation have lower success rates and may end up in locations other than the goal areas. In Point-push, $\epsilon t$-greedy, $\epsilon z$-greedy, and DOIE first visit the lower section of the environment in the early stages. After that, they push aside the movable box and proceed to the upper section to visit the goal area. For the other methods, the pattern is almost the same, with occasional visits to the lower section. To further illustrate how the policy evolves during training and how closely it approximates the optimal policy, we present the agent’s trajectories across episodes. For clarity, we sample a few representative episodes from different stages of training. Figure~\ref{fig: policy quality} presents the results for U-maze and Point-push. In both environments, the agent initially becomes trapped in suboptimal regions. As training progresses, the policy gradually improves, and by the end, the agent consistently follows a trajectory that closely approximates the shortest path to the goal.

\begin{figure*}[h]
    \centering
        \begin{subfigure}[b]{160pt}
        \centering
        \includegraphics[width=\textwidth,height=4.5cm]{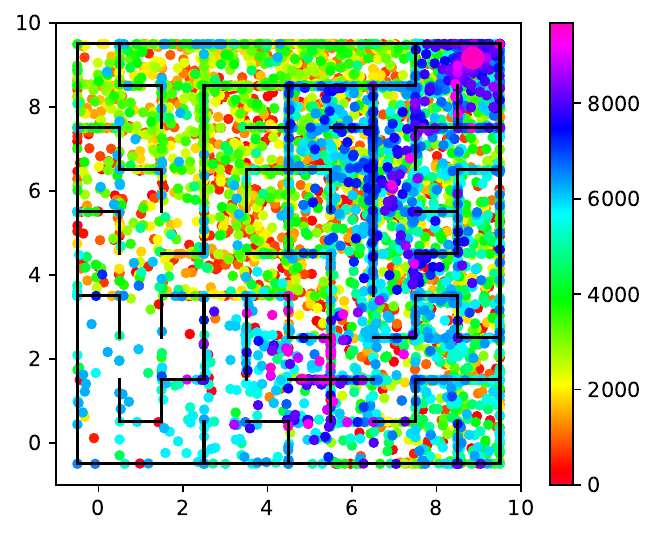}
        \caption{DDPG + $\epsilon t$-greedy}
        \end{subfigure}
        \begin{subfigure}[b]{160pt}
        \centering
        \includegraphics[width=\textwidth,height=4.5cm]{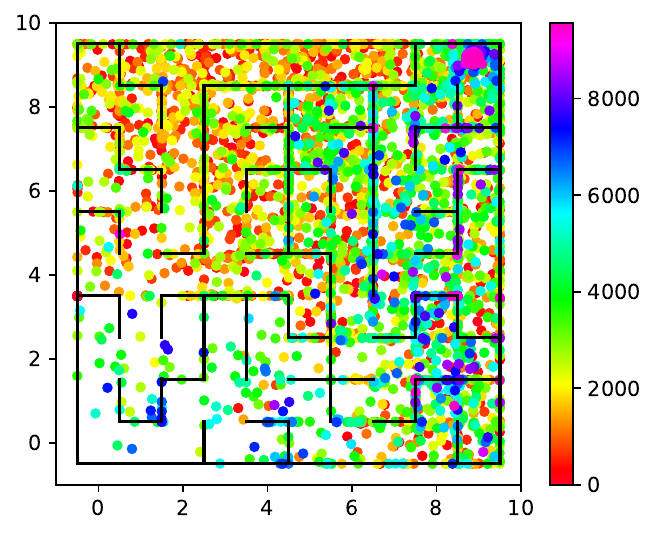}
        \caption{DOIE}
        \end{subfigure}
        \begin{subfigure}[b]{160pt}
        \centering
        \includegraphics[width=\textwidth,height=4.5cm]{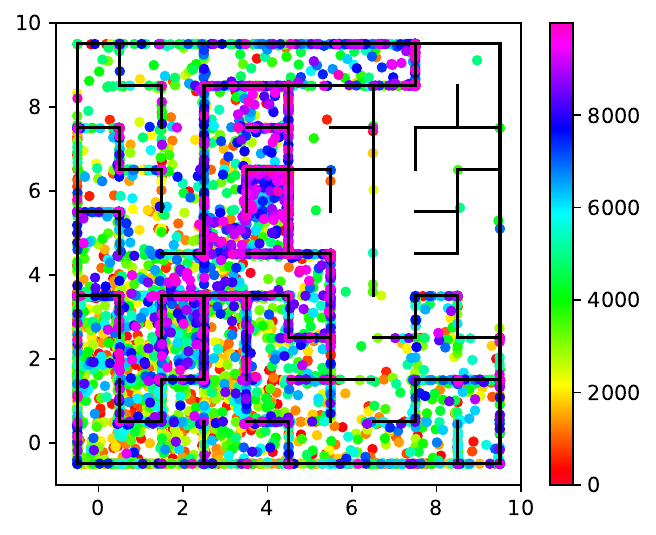}
        \caption{DDPG + $\epsilon z$-greedy}
        \end{subfigure}
        \begin{subfigure}[b]{160pt}
        \centering
        \includegraphics[width=\textwidth,height=4.5cm]{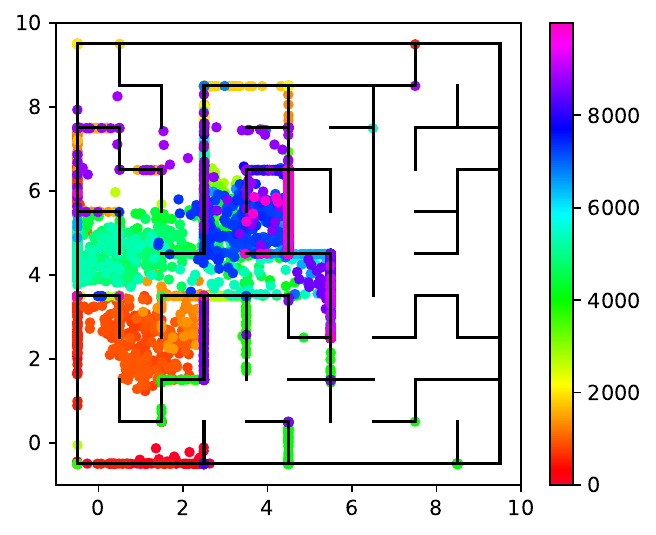}
        \caption{SAC}
        \end{subfigure}
        \begin{subfigure}[b]{160pt}
        \centering
        \includegraphics[width=\textwidth,height=4.5cm]{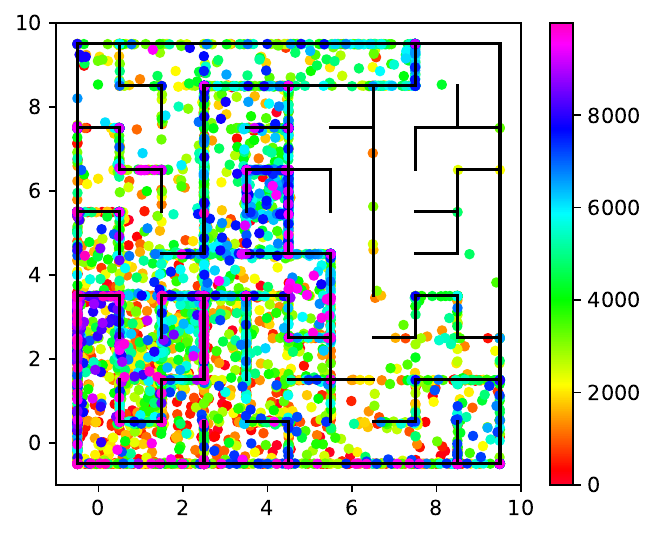}
        \caption{DDPG + intrinsic motivation}
        \end{subfigure}
        \begin{subfigure}[b]{160pt}
        \centering
        \includegraphics[width=\textwidth,height=4.5cm]{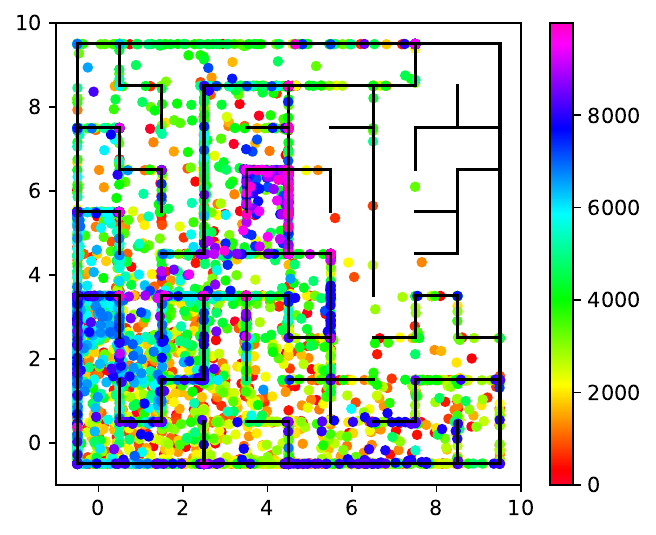}
        \caption{TD3}
        \end{subfigure}
        \begin{subfigure}[b]{160pt}
        \centering
        \includegraphics[width=\textwidth,height=4.5cm]{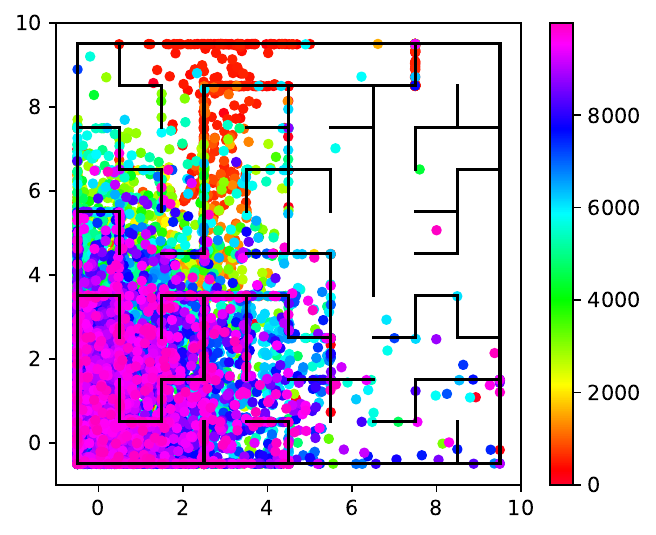}
        \caption{DDPG}
        \end{subfigure}
    \caption{The agent's location at the end of episodes throughout the training in Wall-maze.}
    \label{fig: loc wall-maze}
\end{figure*}

\begin{figure*}[h]
    \centering
        \begin{subfigure}[b]{160pt}
        \centering
        \includegraphics[width=\textwidth,height=4.5cm]{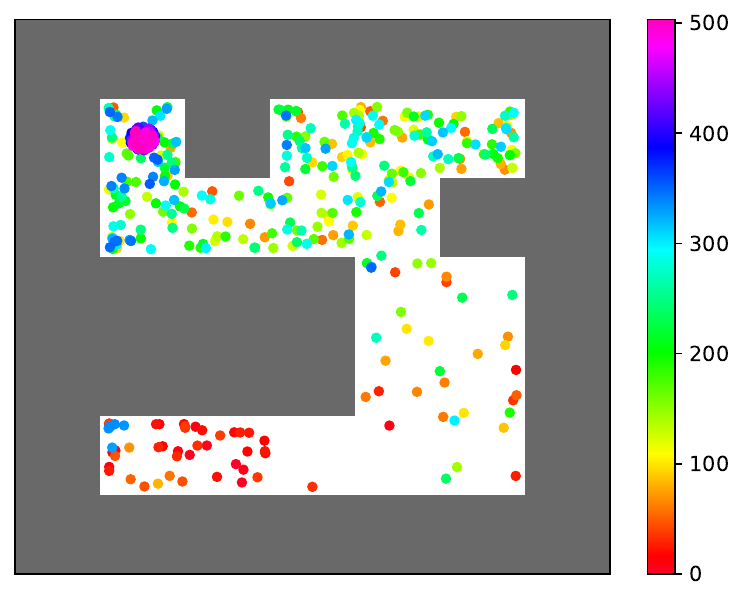}
        \caption{DDPG + $\epsilon t$-greedy}
        \end{subfigure}
        \begin{subfigure}[b]{160pt}
        \centering
        \includegraphics[width=\textwidth,height=4.5cm]{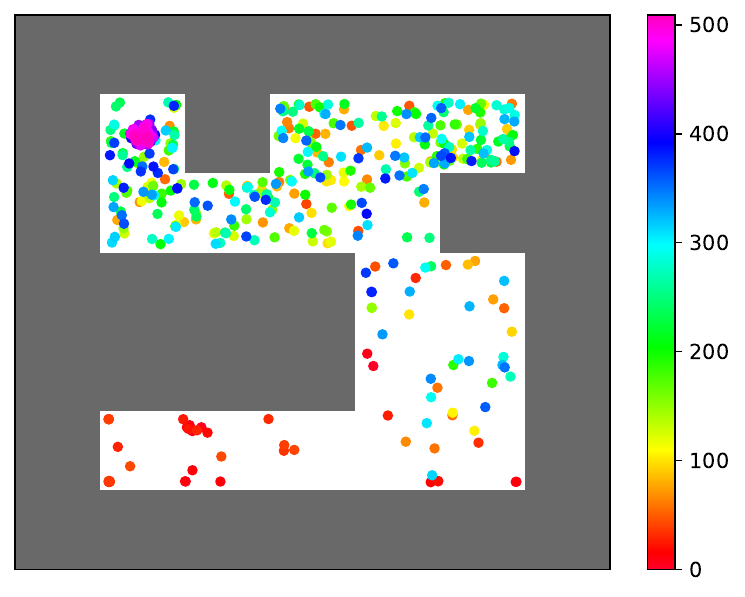}
        \caption{DOIE}
        \end{subfigure}
        \begin{subfigure}[b]{160pt}
        \centering
        \includegraphics[width=\textwidth,height=4.5cm]{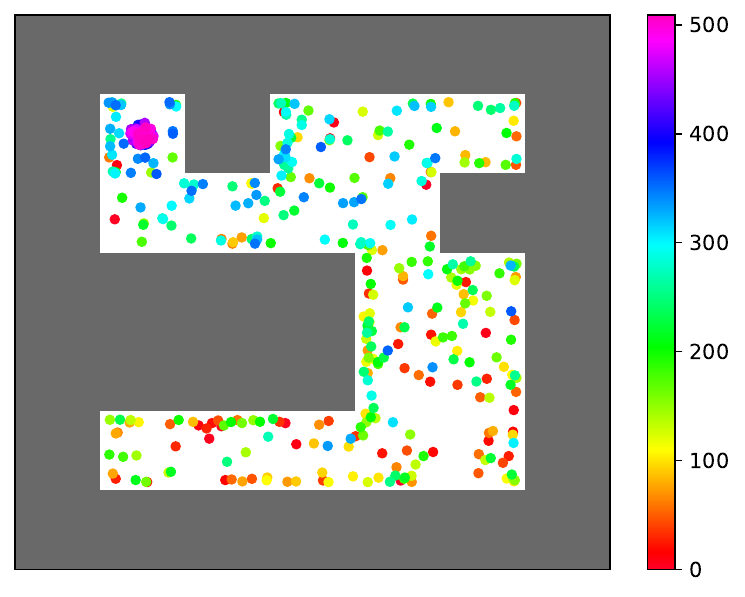}
        \caption{DDPG + $\epsilon z$-greedy}
        \end{subfigure}
        \begin{subfigure}[b]{160pt}
        \centering
        \includegraphics[width=\textwidth,height=4.5cm]{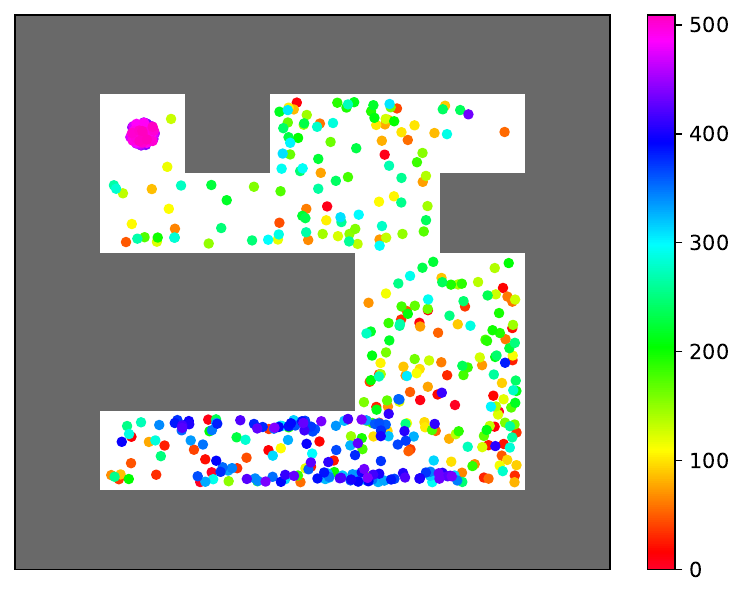}
        \caption{SAC}
        \end{subfigure}
        \begin{subfigure}[b]{160pt}
        \centering
        \includegraphics[width=\textwidth,height=4.5cm]{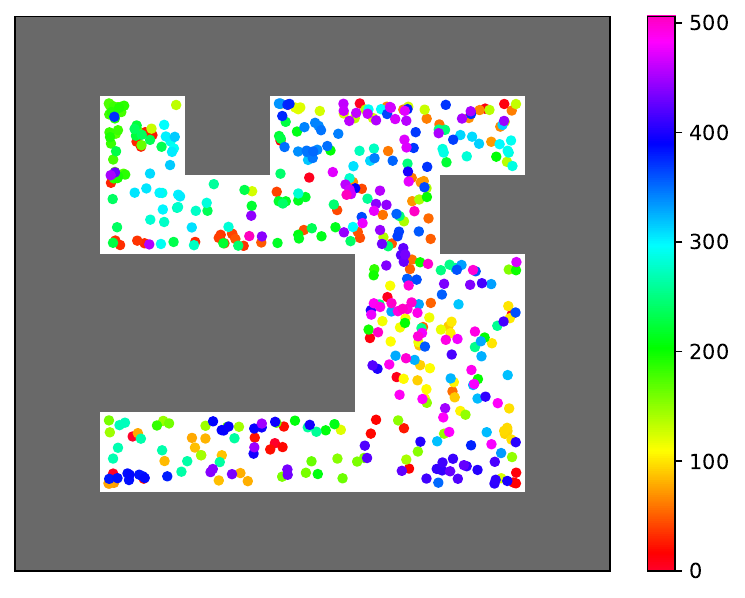}
        \caption{DDPG + intrinsic motivation}
        \end{subfigure}
        \begin{subfigure}[b]{160pt}
        \centering
        \includegraphics[width=\textwidth,height=4.5cm]{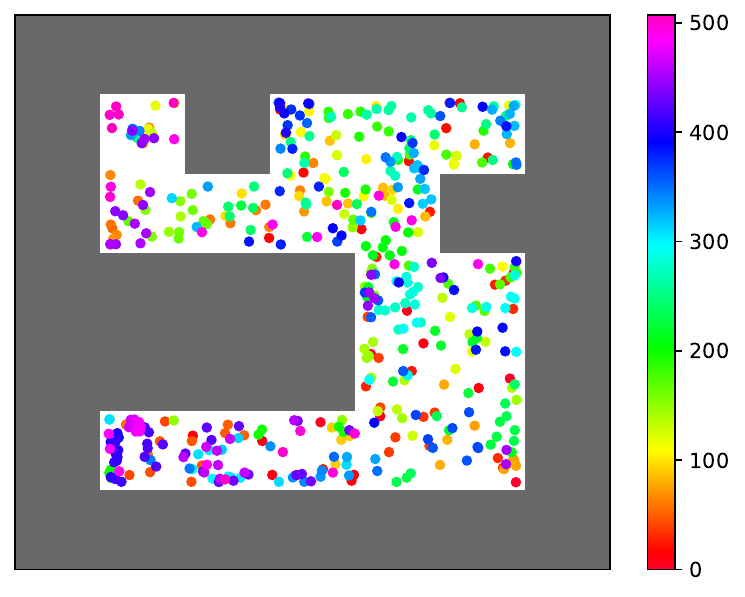}
        \caption{TD3}
        \end{subfigure}
        \begin{subfigure}[b]{160pt}
        \centering
        \includegraphics[width=\textwidth,height=4.5cm]{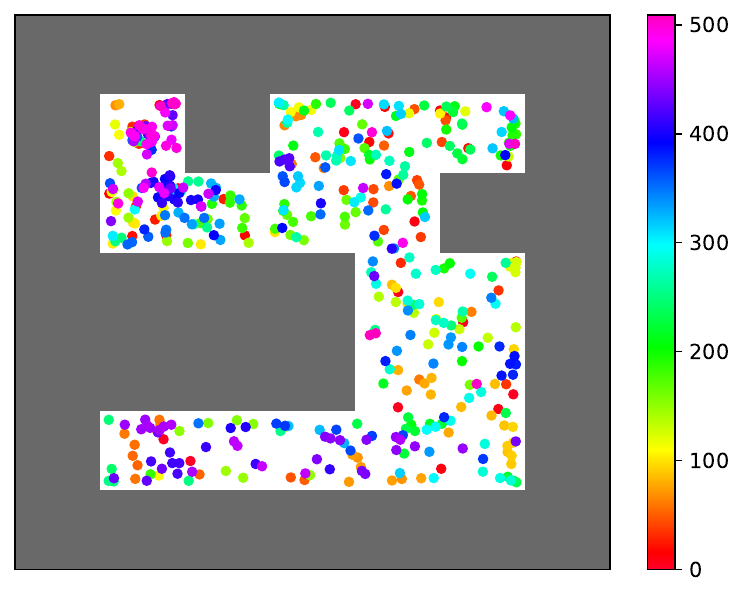}
        \caption{DDPG}
        \end{subfigure}
    \caption{The agent's location at the end of episodes throughout the training in U-maze.}
    \label{fig: loc u-maze}
\end{figure*}

\begin{figure*}[h]
    \centering
        \begin{subfigure}[b]{160pt}
        \centering
        \includegraphics[width=\textwidth,height=4.5cm]{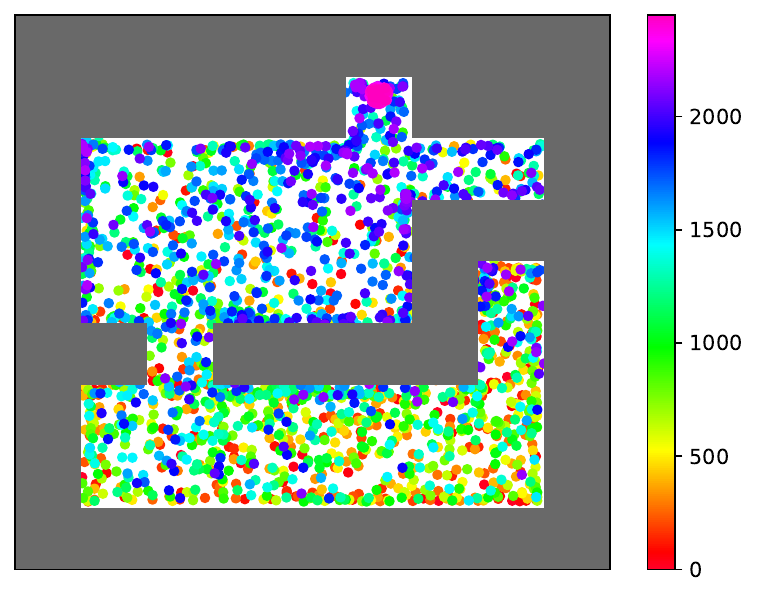}
        \caption{DDPG + $\epsilon t$-greedy}
        \end{subfigure}
        \begin{subfigure}[b]{160pt}
        \centering
        \includegraphics[width=\textwidth,height=4.5cm]{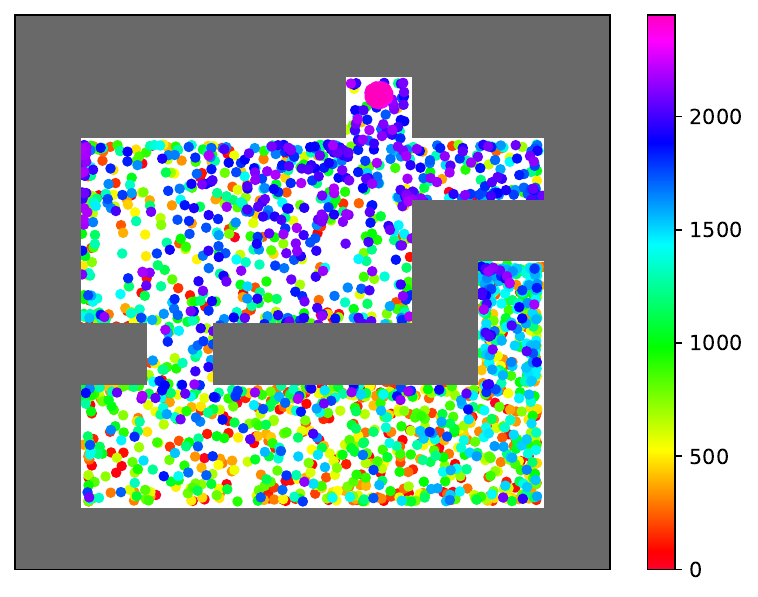}
        \caption{DOIE}
        \end{subfigure}
        \begin{subfigure}[b]{160pt}
        \centering
        \includegraphics[width=\textwidth,height=4.5cm]{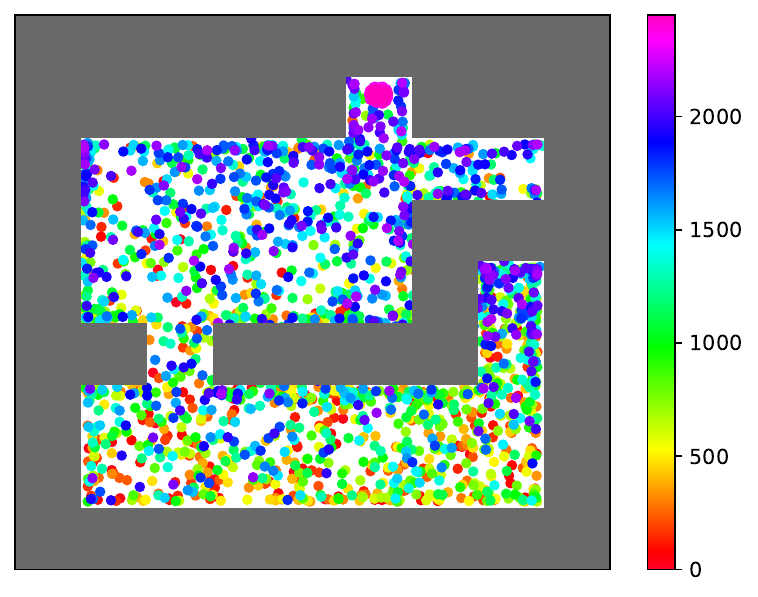}
        \caption{DDPG + $\epsilon z$-greedy}
        \end{subfigure}
        \begin{subfigure}[b]{160pt}
        \centering
        \includegraphics[width=\textwidth,height=4.5cm]{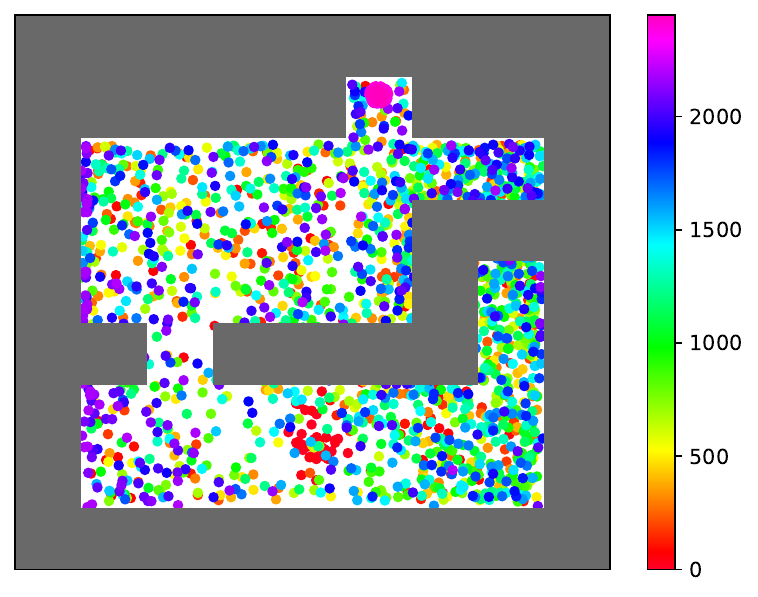}
        \caption{SAC}
        \end{subfigure}
        \begin{subfigure}[b]{160pt}
        \centering
        \includegraphics[width=\textwidth,height=4.5cm]{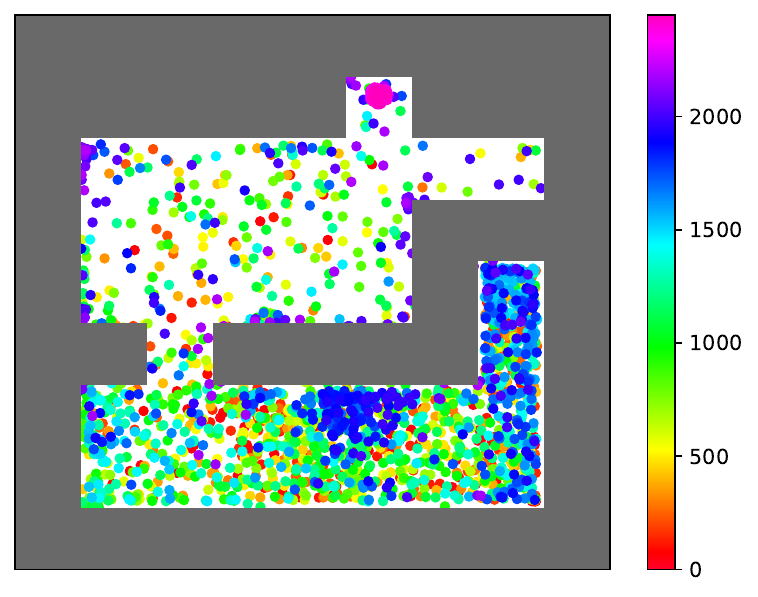}
        \caption{DDPG + intrinsic motivation}
        \end{subfigure}
        \begin{subfigure}[b]{160pt}
        \centering
        \includegraphics[width=\textwidth,height=4.5cm]{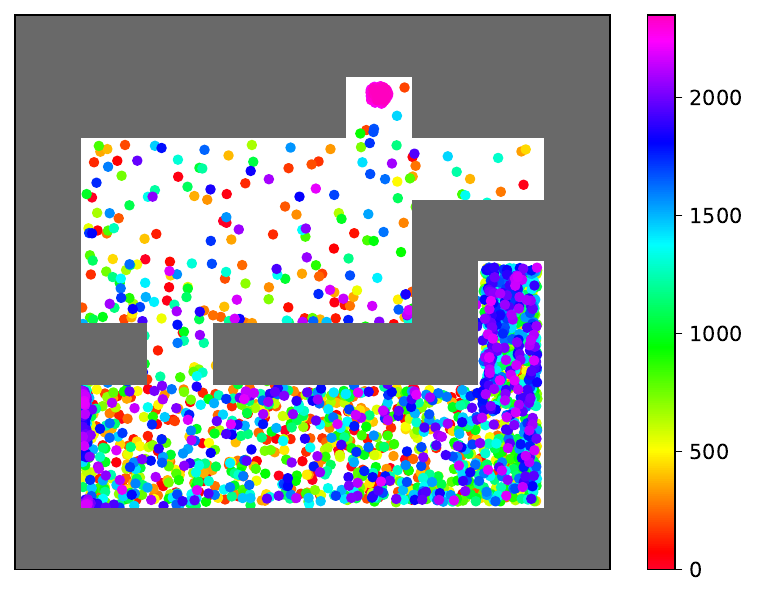}
        \caption{TD3}
        \end{subfigure}
        \begin{subfigure}[b]{160pt}
        \centering
        \includegraphics[width=\textwidth,height=4.5cm]{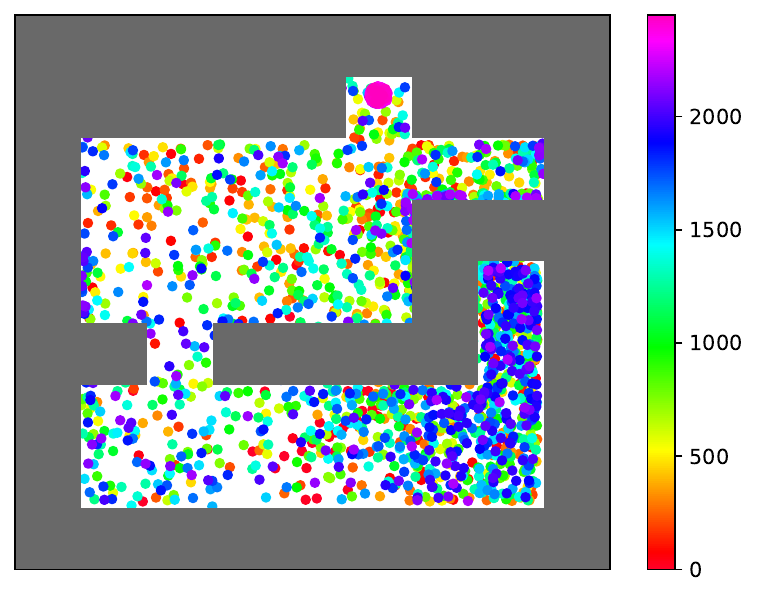}
        \caption{DDPG}
        \end{subfigure}
    \caption{The agent's location at the end of episodes throughout the training in Point-push.}
    \label{fig: loc push}
\end{figure*}

\begin{figure*}[t]
    \centering
        \begin{subfigure}[b]{200pt}
        \centering
        \includegraphics[width=\textwidth]{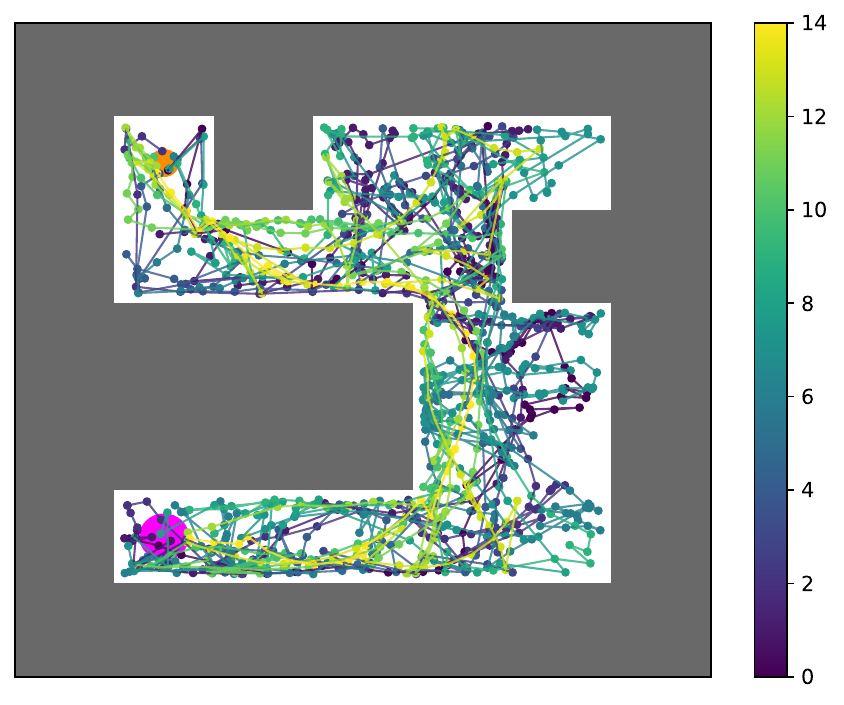}
        \caption{U-maze}
        \end{subfigure}
        \begin{subfigure}[b]{200pt}
        \centering
        \includegraphics[width=\textwidth]{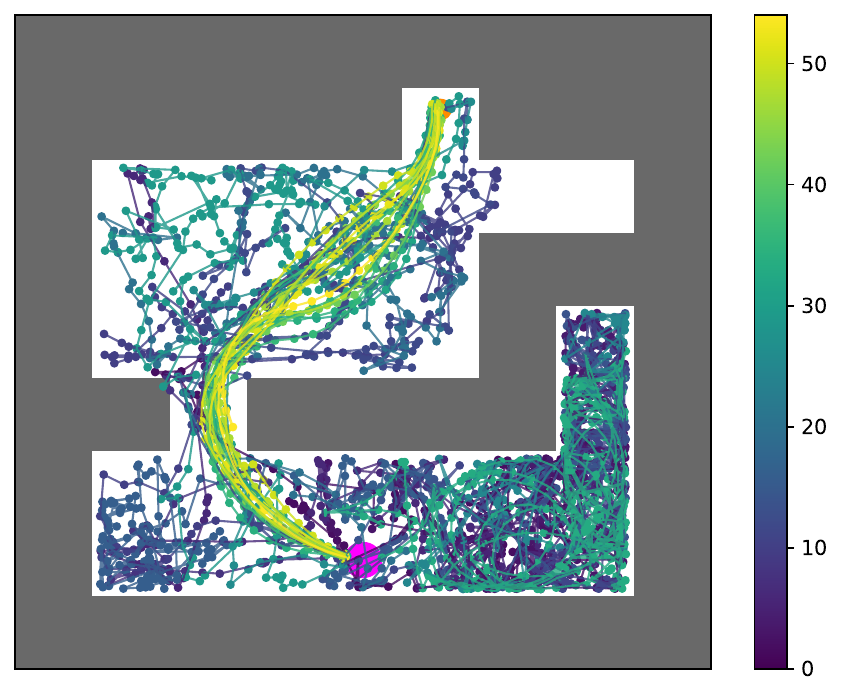}
        \caption{Point-push}
        \end{subfigure}
    \caption{The agent's policy during training, illustrated by the paths taken in each episode. Episodes are sampled from different stages of training to show how the policy evolves over time.}
    \label{fig: policy quality}
\end{figure*}

\end{document}